\documentclass[]{bytedance_seed}

\usepackage[toc,page,header]{appendix}

\usepackage[nohints]{minitoc}

\usepackage{amsthm}
\usepackage{needspace}
\usepackage{amssymb}
\usepackage{algorithm}
\usepackage[noend]{algcompatible} %
\usepackage{algorithm} %
\usepackage{xcolor}
\usepackage{fix-cm}
\newtheorem{theorem}{Theorem}[section]

\theoremstyle{definition} 

\newtheorem{assumption}[theorem]{Assumption}

\crefname{figure}{Fig.}{Figs.}
\crefname{table}{Tab.}{Tabs.}
\crefname{equation}{Eq.}{Eqs.}
\crefname{section}{Sec.}{Secs.}
\crefname{theorem}{Theorem}{Theorem}
\crefname{algorithm}{Algorithm}{Algorithm}
\newcounter{nexusproofassumption}
\crefname{nexusproofassumption}{Assumption}{Assumptions}
\Crefname{nexusproofassumption}{Assumption}{Assumptions}
\newcounter{nexusclosenessproofstep}
\crefname{nexusclosenessproofstep}{Step}{Steps}
\Crefname{nexusclosenessproofstep}{Step}{Steps}
\newcounter{nexusupdateproofstep}
\crefname{nexusupdateproofstep}{Step}{Steps}
\Crefname{nexusupdateproofstep}{Step}{Steps}

\crefname{appendix}{Appendix}{Appendices}
\Crefname{appendix}{Appendix}{Appendices}

\newif\ifcomment
\commenttrue

\title{Nexus: Same Pretraining Loss, Better Downstream Generalization via Common Minima}

\author[1,2,\ddagger]{Huanran Chen}
\author[1]{Huaqing Zhang}
\author[2,\dagger]{Xiao Li}
\author[1, *]{Yinpeng Dong}
\author[2]{Ke Shen}
\author[1, *]{Jun Zhu}

\affiliation[1]{Tsinghua University}
\affiliation[2]{ByteDance Seed}

\contribution[\ddagger]{Work done at ByteDance Seed}
\contribution[\dagger]{Project Lead}
\contribution[*]{Corresponding authors}

\abstract{

The foundational capabilities of large language models are acquired during pretraining on internet-scale, highly heterogeneous data mixtures. In this work, we investigate an interesting geometric question regarding the converged state of pretraining: Does the model converge to a common minimizer across all data sources (e.g., \cref{fig:cwa_illustration:close}), or merely a minimizer of the summed loss (e.g., \cref{fig:cwa_illustration:distant})? We hypothesize that the geometric "closeness" of task-specific minima is intrinsically linked to downstream generalization. We reveal that standard optimizers (e.g., AdamW) often converge to points where task-specific minima are distant from each other. To address this, we propose the Nexus optimizer, which encourages the closeness of these minima by maximizing gradient similarity during optimization. Experiments across models ranging from 130M to 3B parameters, various data mixtures and hyperparameter schedules, show that Nexus \textit{significantly boosts downstream performance}, despite \textit{achieving the same pretraining loss} (see \cref{fig:demo:benchmark}). Notably, on the 3B model, Nexus reduces the out-of-distribution loss by 0.012 and yields up to a 15.0\% accuracy improvement on complex reasoning tasks (e.g., GSM8k). This finding challenges the reliance on pretraining loss as the sole proxy for model evaluation and demonstrates the importance of implicit biases in unlocking downstream generalization.

}

\date{\today}

\checkdata[Email]{Yingpeng Dong, Jun Zhu at \email{\{dongyinpeng, dcszj\}@tsinghua.edu.cn}; \\
Huanran Chen, Huaqing Zhang at \email{\{chenhr25, zhanghq22\}@mails.tsinghua.edu.cn}; \\
Xiao Li, Ke Shen at \email{\{lixiao.20, shenke\}@bytedance.com}.}

\begin{document}
\maketitle

\begin{figure}[t]
\vspace{-1ex}
    \centering
    \begin{subfigure}[b]{0.31\linewidth}
        \centering
        \includegraphics[width=\linewidth]{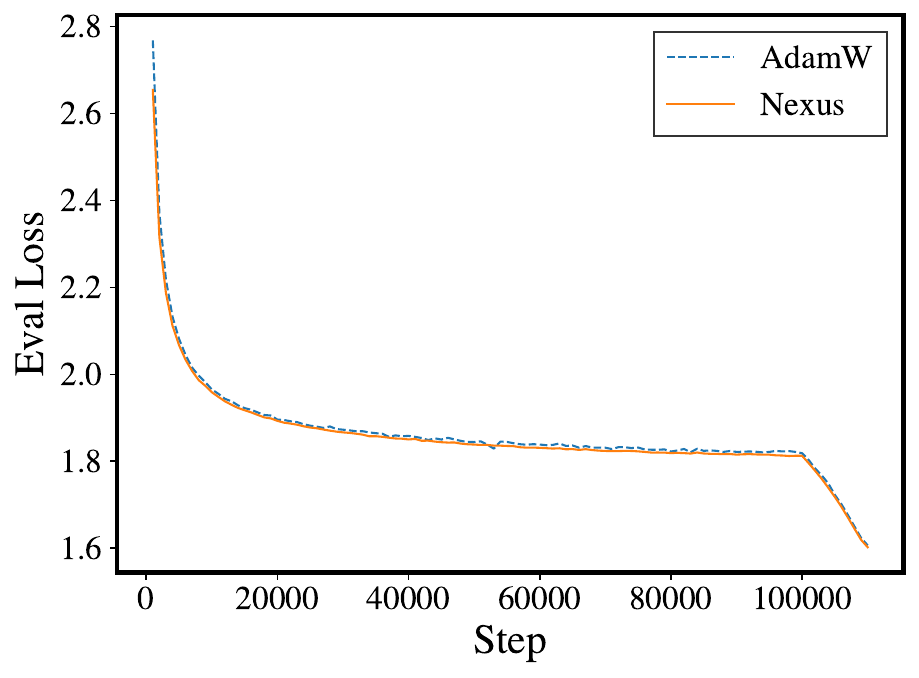}\vspace{-1ex}
        \caption{Pretraining Loss}
    \end{subfigure}
    \begin{subfigure}[b]{0.31\linewidth}
        \centering
        \includegraphics[width=\linewidth]{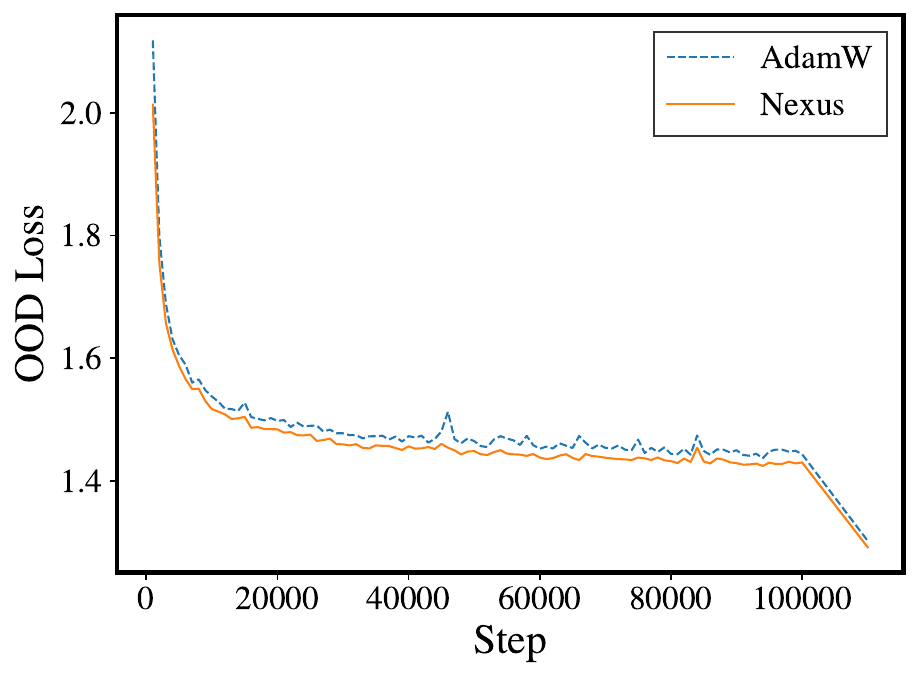}\vspace{-1ex}
        \caption{OOD Loss}
    \end{subfigure}
    \begin{subfigure}[b]{0.31\linewidth}
        \centering
        \includegraphics[width=\linewidth]{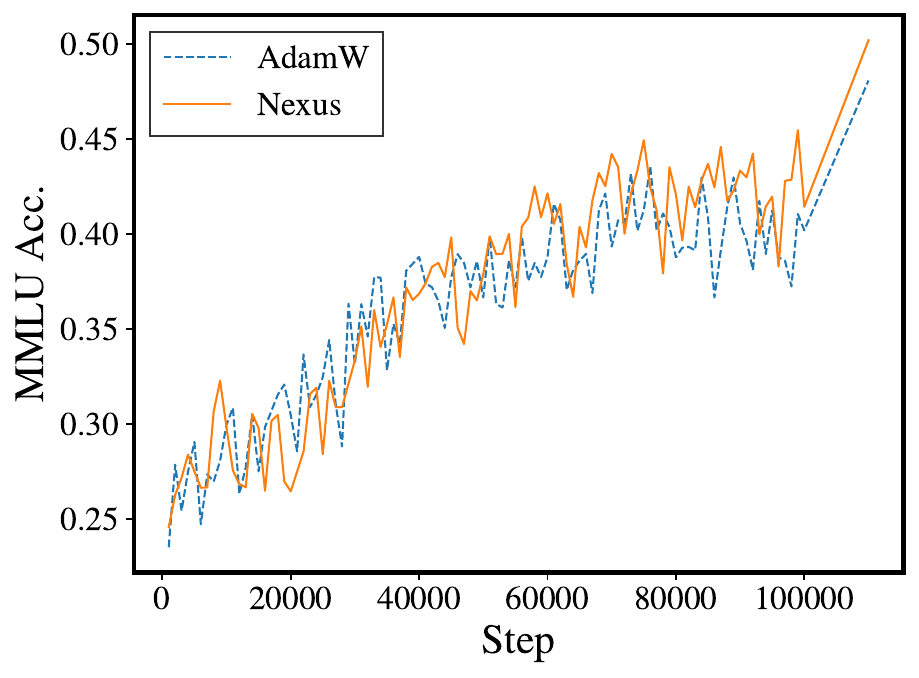}\vspace{-1ex}
        \caption{MMLU Acc.}
    \end{subfigure}\vspace{1ex}
    \begin{subfigure}[b]{0.31\linewidth}
        \centering
        \includegraphics[width=\linewidth]{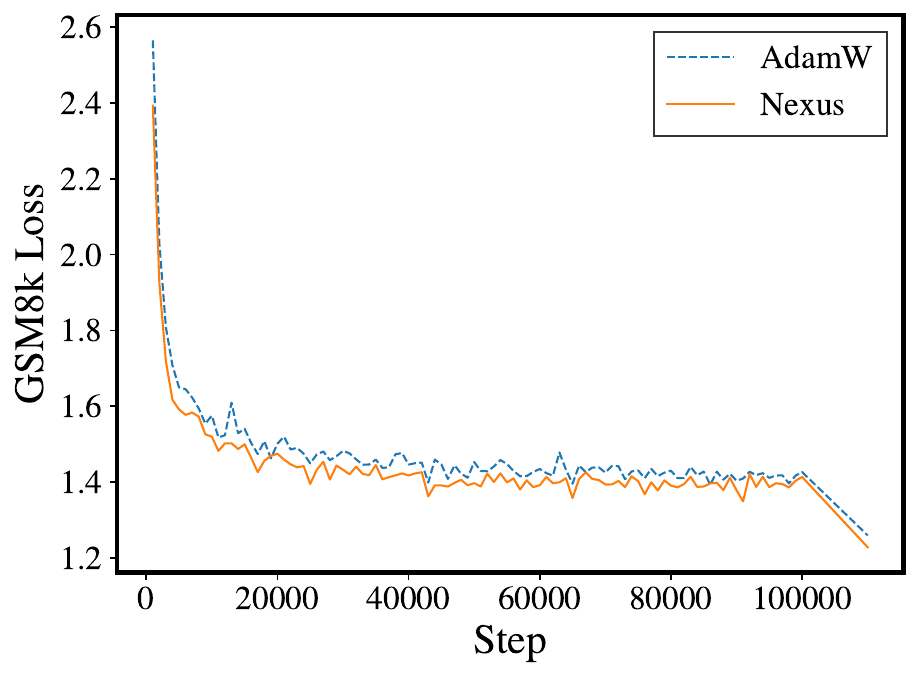}\vspace{-1ex}
        \caption{GSM8k Loss}
    \end{subfigure}
    \begin{subfigure}[b]{0.31\linewidth}
        \centering
        \includegraphics[width=\linewidth]{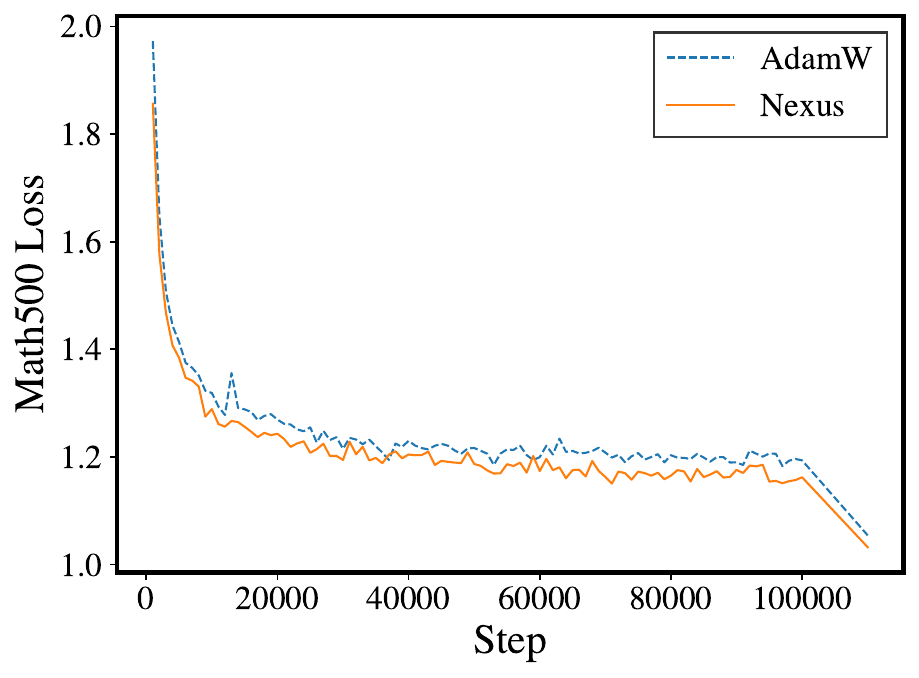}\vspace{-1ex}
        \caption{Math500 Loss}
    \end{subfigure}
    \begin{subfigure}[b]{0.31\linewidth}
        \centering
        \includegraphics[width=\linewidth]{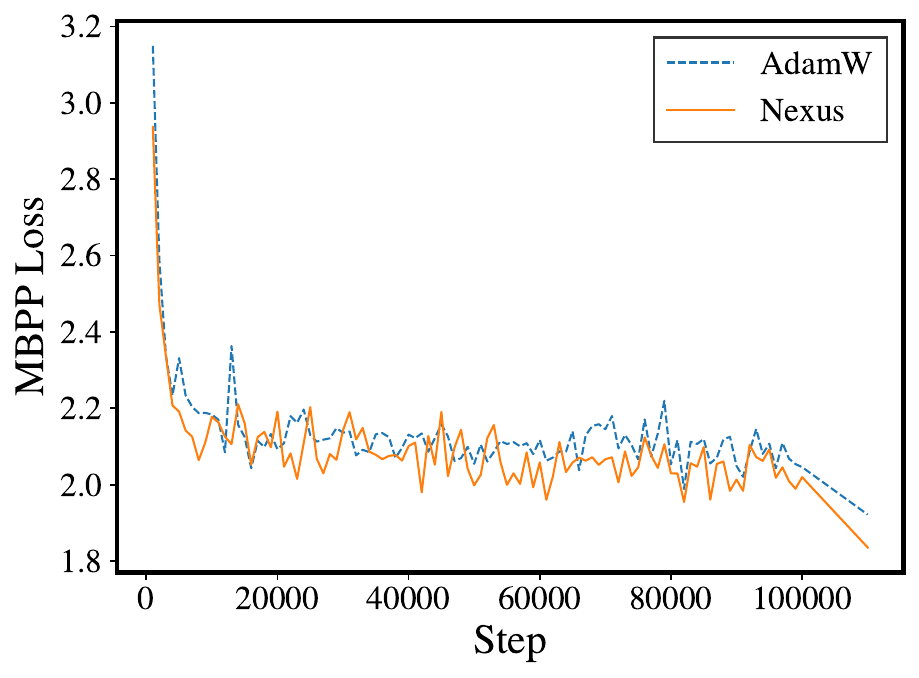}\vspace{-1ex}
        \caption{MBPP Loss}
    \end{subfigure}\vspace{1ex}
    \begin{subfigure}[b]{0.31\linewidth}
        \centering
        \includegraphics[width=\linewidth]{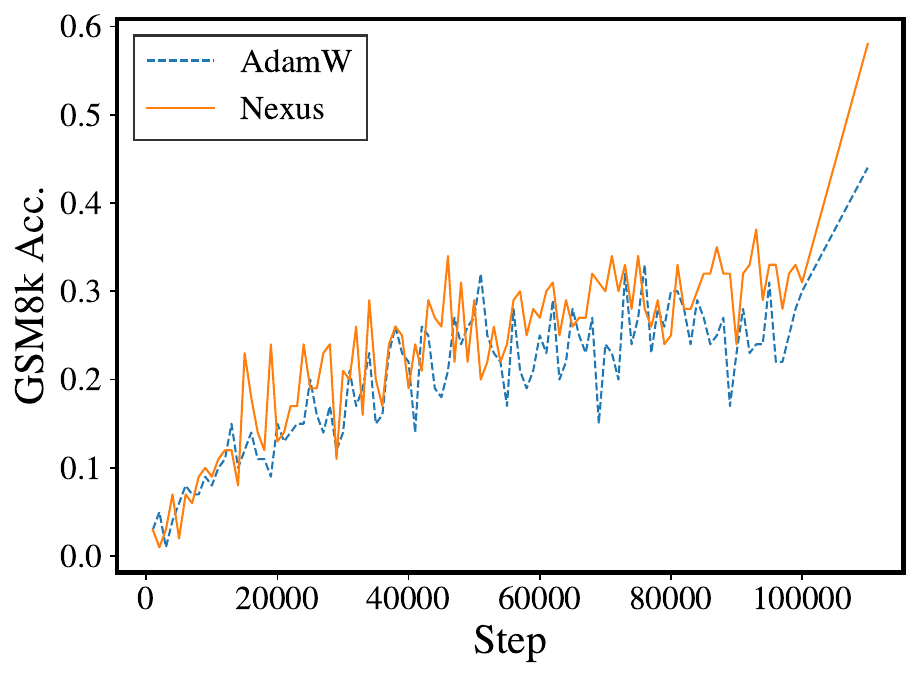}\vspace{-1ex}
        \caption{GSM8k Acc.}
    \end{subfigure}
    \begin{subfigure}[b]{0.31\linewidth}
        \centering
        \includegraphics[width=\linewidth]{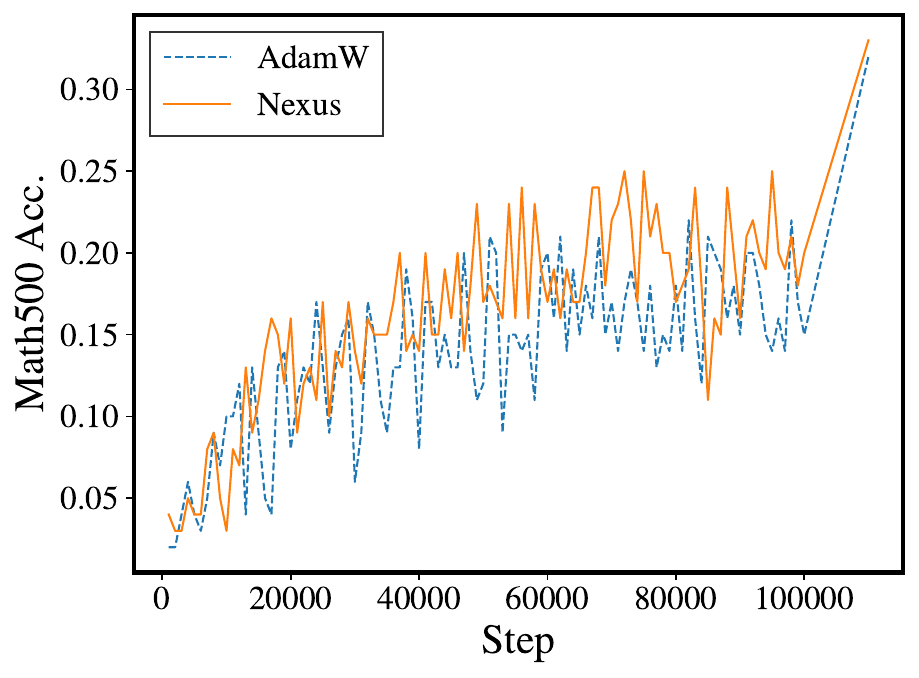}\vspace{-1ex}
        \caption{Math500 Acc.}
    \end{subfigure}
    \begin{subfigure}[b]{0.31\linewidth}
        \centering
        \includegraphics[width=\linewidth]{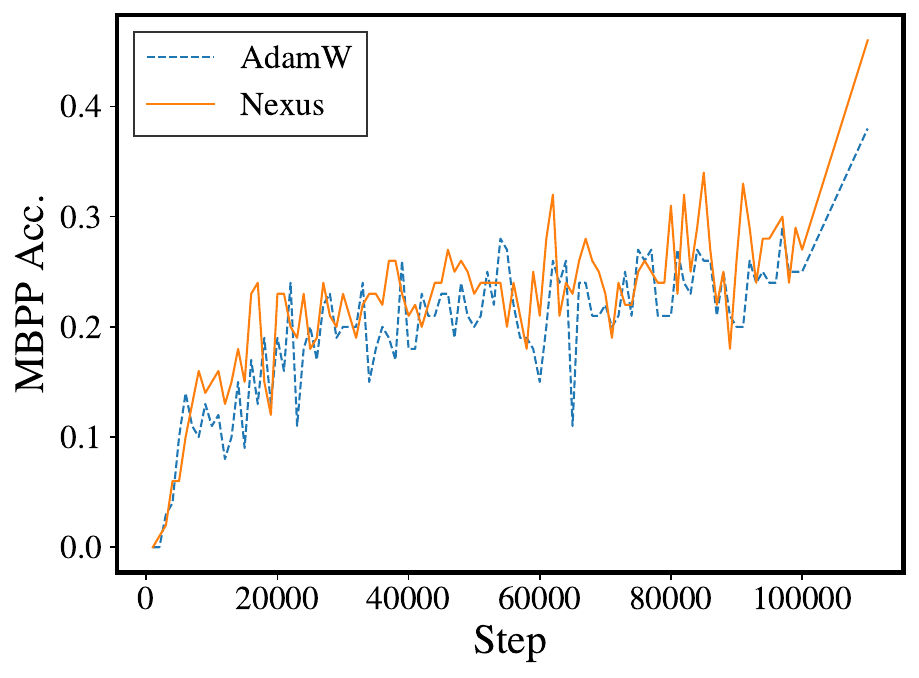}\vspace{-1ex}
        \caption{MBPP Acc.}
    \end{subfigure}
    \caption{Illustration of ``same pretraining loss, better downstream task''. The training loss of baseline and our Nexus are nearly the same. However, Nexus achieves better downstream generalization.}
    \label{fig:demo:benchmark}
    \vspace{-4ex}
\end{figure}

\section{Introduction}

Pretraining is the cornerstone of Large Language Models (LLMs). Accounting for over 95\% of the total computational budget and data, it serves as the indispensable engine for their capabilities~\cite{liu2024deepseekv3, yang2024qwen25,dubey2024llama,olmo20252olmo2furious}. During pretraining, LLMs acquire foundational knowledge from unprecedentedly large and diverse data sources, encompassing a vast array of domains such as general language, mathematics, code, and complex reasoning~\cite{openai2023gpt,anthropic2024claude,team2023gemini}. To learn from such a heterogeneous corpus of $K$ distinct sources, the standard practice is to average the loss of each data source $\mathcal{L}_k(\bm{\theta})$ and minimize the averaged loss $\mathcal{L}_{\text{train}}(\bm{\theta}) = \frac{1}{K} \sum_{k=1}^K \mathcal{L}_k(\bm{\theta})$.

In this work, we investigate an interesting geometric question: \textit{Does the model converge to a common minimizer across all data sources $\mathcal{L}_k$, or does it merely find a minimizer of the summed loss $\mathcal{L}_{\text{train}}$?} To illustrate this, consider a simplified setting composed of two data sources ($\mathcal{L}_1$ and $\mathcal{L}_2$), yielding a training loss of $\mathcal{L}_{\text{train}}(\bm{\theta}) = \frac{1}{2}(\mathcal{L}_1(\bm{\theta}) + \mathcal{L}_2(\bm{\theta}))$. As depicted in \cref{fig:cwa_illustration}, there exist two distinct types of minimizers that achieve the exact same training loss $\mathcal{L}_{\text{train}}$. The first corresponds to the \textbf{Average-type Minimizer} (\cref{fig:cwa_illustration:distant}), where the converged parameter $\bm{\theta}_{\text{train}}^*$ successfully minimizes the total training loss $\mathcal{L}_{\text{train}}$ yet may remain geometrically distant from the minimizers of individual tasks $\mathcal{L}_k$. The second approaches the \textbf{Intersection-type Minimizer} (\cref{fig:cwa_illustration:close}), where $\bm{\theta}_{\text{train}}^*$ is not only a minimizer of $\mathcal{L}_{\text{train}}$, but is also geometrically close to the minimizer of each individual task $\mathcal{L}_k$.

We hypothesize that this geometric ``\textbf{closeness}''---the distance between task-specific minima---is strongly correlated with \textbf{downstream generalization}. Even when achieving the exact \textit{same pretraining loss}, these two types of minimizers yield drastically \textit{different downstream} losses $\mathcal{L}_{\mathcal{T}}$ (see the blue curve $\mathcal{L}_{\mathcal{T}}(\bm{\theta})$ in \cref{fig:cwa_illustration}). Intuitively, if the training losses $\mathcal{L}_k$ and downstream task $\mathcal{L}_{\mathcal{T}}$ are quadratic and i.i.d. distributed, the Intersection-type minimizer (\cref{fig:cwa_illustration:close}) will strictly outperform the Sum-type minimizer (\cref{fig:cwa_illustration:distant}) on the downstream task $\mathcal{L}_{\mathcal{T}}$, given the same pretraining loss (see \cref{theorem:closeness_generalization:simple}). Therefore, we posit that this intuition may generalize beyond quadratics to LLM pretraining, and steering the optimization toward the Intersection-type minimizer would achieve the ``same pretraining loss, better downstream task''.

\begin{figure}[t]
    \centering
    \begin{subfigure}[b]{0.45\linewidth}
        \centering
        \includegraphics[width=\linewidth]{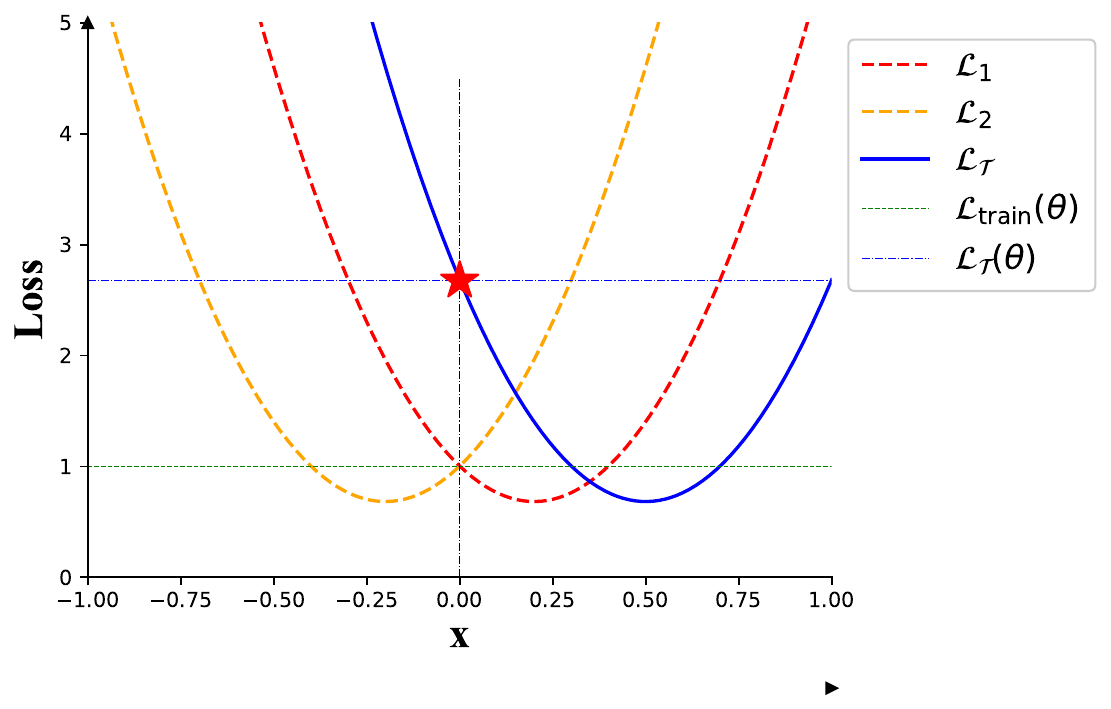}
        \caption{Distant (Average-type Minimizer)}
        \label{fig:cwa_illustration:distant}
    \end{subfigure}
    \hfill
    \begin{subfigure}[b]{0.45\linewidth}
        \centering
        \includegraphics[width=\linewidth]{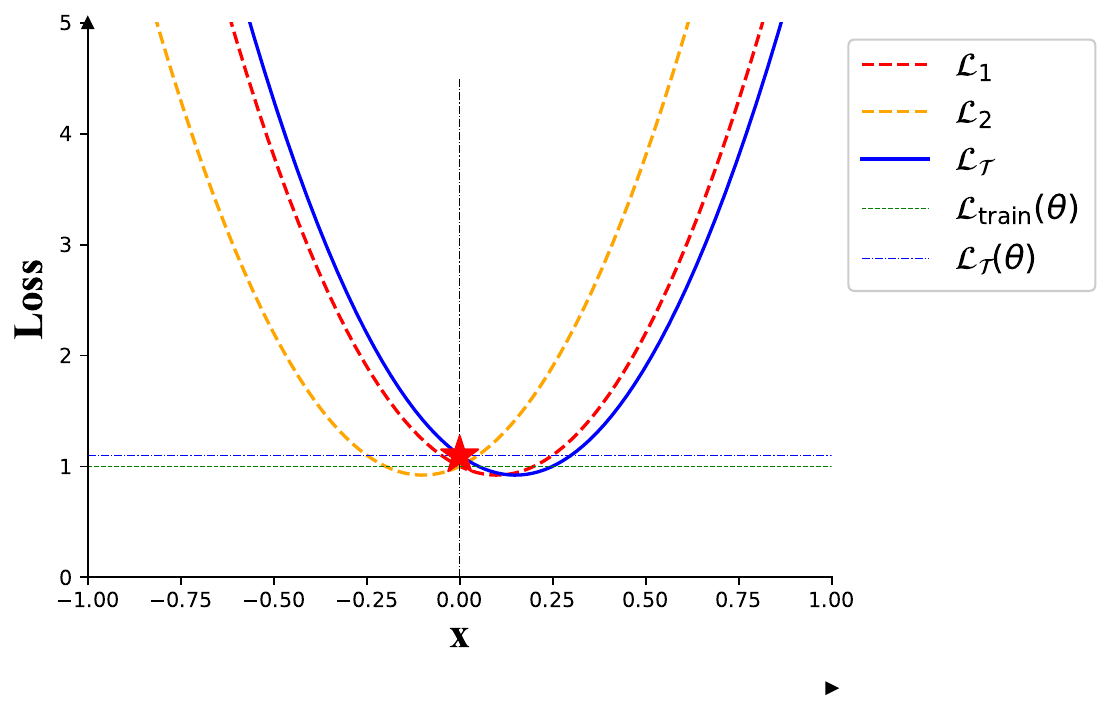}
        \caption{Close (Intersection-type Minimizer)}
        \label{fig:cwa_illustration:close}
    \end{subfigure}
    
    \caption{Illustration of two types of minimizer. (a) Distant: Minimizers of each source are distant from each other. (b) Close: Minimizers are geometrically close to each other. Although both configurations achieve the \textit{same total training loss}, they perform \textit{significantly differently on a new downstream task} $\mathcal{L}_{\mathcal{T}}$.}
    \label{fig:cwa_illustration}
\end{figure}

However, directly optimizing for this geometric ``closeness'' is computationally intractable, as it requires knowing the exact minimizer of each $\mathcal{L}_k$ at every training step. To overcome this, we prove that the gradient similarity between tasks, $\text{CosSim}(\nabla \mathcal{L}_i, \nabla \mathcal{L}_j) \triangleq \frac{\nabla \mathcal{L}_i^T \nabla \mathcal{L}_j}{\|\nabla \mathcal{L}_i\| \|\nabla \mathcal{L}_j\|} $, upper bounds the geometric closeness. The rationale is straightforward: if the gradient directions of each loss $\nabla \mathcal{L}_k$ are always exactly the same throughout optimization, their respective minimizers $\bm{\theta}_k^*$ must be exactly the same. Based on this insight, we propose the Nexus algorithm, which approximates the gradient of gradient similarity $\nabla \text{CosSim}(\nabla \mathcal{L}_i, \nabla \mathcal{L}_j)$. Combining Nexus with pretraining optimizer~\cite{wen2025fantastic,kingma2014adam,jordan2024muon} effectively maximizes $\text{CosSim}(\nabla \mathcal{L}_i, \nabla \mathcal{L}_j)$. In \cref{sec:exp_validate_theory}, we show that both gradient similarity and geometric closeness generalize to downstream tasks, thus leading to lower downstream loss and better downstream performance, even when achieving the same pretraining loss.

We empirically validate Nexus across various settings, including model scales ranging from 130M to 3B parameters~\cite{yang2024qwen25,wen2025fantastic,touvron2023llama}, diverse pretraining data and mixtures~\cite{seed2025seed-oss,basant2025nvidia_nemotron}, learning rate schedules~\cite{wen2025understanding,loshchilov2017sgdr,hu2024minicpm}, and training compute allocations~\cite{kaplan2020scaling}. Experimental results demonstrate that, across nearly all settings, Nexus consistently enhances downstream capabilities by reducing downstream losses while achieving nearly the same pretraining loss. For instance, on the 3B model, Nexus improves GSM8K accuracy by 15\%, MATH500 by 8\%, and HumanEval by 4\%. These consistent and substantial downstream gains demonstrate the importance of implicit biases in unlocking downstream generalization~\cite{liu2023same}, particularly as the current pretraining paradigm transitions from being compute-bound to data-bound~\cite{springer2025overtrained,kim2025pre,prabhudesai2025diffusion,ni2025diffusion}.

\section{Closeness: A Second-Order Property Related to Generalization}

\subsection{Problem Formulation}

Formally, let the pretraining corpus be the union of $K$ distinct data sources, denoted as $\mathcal{D}_{\text{train}} = \cup_{k=1}^{K} \mathcal{D}_k$. Let $\alpha_k$ represent the sampling probability (data mixing ratio) for the $k$-th source. We define the \textit{weighted} empirical loss function for the $k$-th source as:
\begin{equation}
    \mathcal{L}_k(\bm{\theta}) = - \frac{\alpha_k }{|\mathcal{D}_k|} \sum_{j=1}^{|\mathcal{D}_k|} \log p(x_j|\bm{\theta}).
\end{equation}
Consequently, the total pretraining objective is simply the average of these weighted losses:
\begin{equation}
    \mathcal{L}_{\text{train}}(\bm{\theta}) = \frac{1}{K}\sum_{k=1}^K \mathcal{L}_k(\bm{\theta}).
\end{equation}

\subsection{Flatness and Closeness are both Second Order Generalization Biases}
\label{sec:flatness_closeness}

Our primary interest lies in how well our pretraining minimizer $\bm{\theta}_{\text{train}}^* \in \arg\min_{\bm{\theta}} \mathcal{L}_{\text{train}}(\bm{\theta})$ performs on the downstream task $\mathcal{T}$, i.e., the downstream loss $\mathcal{L}_{\mathcal{T}}(\bm{\theta}_{\text{train}}^*)$.

Let $\mathcal{S}_{\mathcal{T}} = \{ \bm{\theta}\mid \exists \epsilon > 0, \forall \bm{\theta}' \in B_\epsilon(\bm{\theta}), \mathcal{L}_{\mathcal{T}}(\bm{\theta}) \leq \mathcal{L}_{\mathcal{T}}(\bm{\theta}') \}$ be the set of local minimizers for the downstream task. We define $\bm{\theta}^*_{\mathcal{T}}$ as the closest minimizer of downstream loss:
\begin{equation}
    \bm{\theta}^*_{\mathcal{T}} = \arg\min_{\bm{\theta} \in \mathcal{S}_{\mathcal{T}}} \| \bm{\theta} - \bm{\theta}_{\text{train}}^* \|_2.
\end{equation}
By applying a second-order Taylor expansion with the Lagrange remainder around the optimal point $\bm{\theta}^*_{\mathcal{T}}$, we can characterize the downstream loss at the pretraining minimizer $\bm{\theta}_{\text{train}}^*$:
\begin{equation}
    \begin{aligned}
        \mathcal{L}_{\mathcal{T}}(\bm{\theta}_{\text{train}}^*) &= \mathcal{L}_{\mathcal{T}}(\bm{\theta}^*_{\mathcal{T}}) + (\bm{\theta}_{\text{train}}^* - \bm{\theta}^*_{\mathcal{T}})^\top \nabla \mathcal{L}_{\mathcal{T}}(\bm{\theta}^*_{\mathcal{T}}) + \frac{1}{2} (\bm{\theta}_{\text{train}}^* - \bm{\theta}^*_{\mathcal{T}})^\top \nabla^2 \mathcal{L}_{\mathcal{T}}(\bm{\xi}) (\bm{\theta}_{\text{train}}^* - \bm{\theta}^*_{\mathcal{T}})  \\
        &= \mathcal{L}_{\mathcal{T}}(\bm{\theta}^*_{\mathcal{T}}) + \frac{1}{2} \underbrace{\| \bm{\theta}_{\text{train}}^* - \bm{\theta}^*_{\mathcal{T}} \|^2_2}_{\text{Closeness}} \cdot \underbrace{ \bm{u}^\top \nabla^2 \mathcal{L}_{\mathcal{T}}(\bm{\xi}) \bm{u}}_{\text{Directional Sharpness}},
    \end{aligned}
\label{eq:generalization:flatness_closeness_expansion}
\end{equation}
where $\bm{\xi} \in [\bm{\theta}^*_{\mathcal{T}}, \bm{\theta}_{\text{train}}^*]$ denotes an intermediate point on the line segment connecting $\bm{\theta}^*_{\mathcal{T}}$ and $\bm{\theta}_{\text{train}}^*$, and $\bm{u} = \frac{\bm{\theta}_{\text{train}}^* - \bm{\theta}^*_{\mathcal{T}}}{\|\bm{\theta}_{\text{train}}^* - \bm{\theta}^*_{\mathcal{T}}\|_2}$ represents the unit directional vector. 
Note that the first-order term vanishes because $\bm{\theta}^*_{\mathcal{T}}$ is a local minimizer (i.e., $\nabla \mathcal{L}_{\mathcal{T}}(\bm{\theta}^*_{\mathcal{T}}) = \mathbf{0}$). The remaining term is controlled by two scalars: the directional \textit{Flatness} $\bm{u}^\top \nabla^2 \mathcal{L}_{\mathcal{T}}(\bm{\xi}) \bm{u} \leq \|\nabla^2 \mathcal{L}_{\mathcal{T}}(\bm{\xi})\|_2 \leq \|\nabla^2 \mathcal{L}_{\mathcal{T}}(\bm{\xi})\|_F$, and crucially, the \textit{Closeness} $\| \bm{\theta}_{\text{train}}^* - \bm{\theta}^*_{\mathcal{T}} \|^2_2$.

Therefore, the flatter the local loss landscape of $\mathcal{L}_{\mathcal{T}}$ and the closer the converged parameter $\bm{\theta}_{\text{train}}^*$ is to the task minimizer $\bm{\theta}^*_{\mathcal{T}}$, the better the generalization. Together, flatness and closeness encapsulate all second-order information for downstream generalization. While flatness has been well-studied in prior literature \cite{SAM,srivastava2014dropout,chen2025understanding,kwon2021asam,zhang2024duality}, in this work, we focus solely on our new implicit bias: closeness.

\subsection{Closeness Improves Out-of-Distribution Generalization}

\cref{eq:generalization:flatness_closeness_expansion} reveals that the closeness between the trained parameters and the downstream task minimizers directly correlates with downstream generalization. In other words, if one could minimize $\| \bm{\theta}_{\text{train}}^* - \bm{\theta}^*_{\mathcal{T}}\|^2_2$ without compromising the intrinsic loss $\mathcal{L}_{\mathcal{T}}(\bm{\theta}^*_{\mathcal{T}})$ and the flatness, one would directly boost downstream generalization. 

However, in practice, minimizing closeness typically comes at a cost: either (1) an increase in intrinsic loss $\mathcal{L}_{\mathcal{T}}(\bm{\theta}^*_{\mathcal{T}})$ or (2) an increase in sharpness (see \cref{sec:exp_validate_theory}). This trade-off is expected; if one were to minimize the closeness even among training tasks (i.e., $\frac{1}{K}\sum_{k=1}^K\|\bm{\theta}_{\text{train}}^* - \bm{\theta}^*_{k}\|^2_2$) without penalty, it would imply achieving significantly smaller training error and faster optimization rates. This contradicts the prevailing assumption and empirical observations regarding the inherent hardness of discovering significantly faster optimizers~\cite{wen2025fantastic,semenov2025benchmarking}.

In this paper, we specifically focus on the "same training loss" regime. We demonstrate that a "close" minimizer (\cref{fig:cwa_illustration:close}) yields significantly better out-of-distribution generalization compared to a "distant" minimizer (\cref{fig:cwa_illustration:distant}), even at the \textit{same pretraining loss} $\mathcal{L}_{\text{train}}(\bm{\theta}) = \frac{1}{K} \sum_{k=1}^K \mathcal{L}_{k}(\bm{\theta})$. We analyze the specific scenario where improved closeness is achieved solely at the cost of increasing the intrinsic task loss $\mathcal{L}_{k}(\bm{\theta}^*_{k})$. This assumption decouples our analysis from the flatness bias (thereby eliminating flatness as a confounding factor) and aligns with the actual behavior observed in our experiments (see \cref{sec:exp_validate_theory}).

The core intuition is illustrated in \cref{fig:cwa_illustration}: as long as the loss landscape is quadratic-like along the directions of interest (i.e., locally and directionally strongly convex), and the pretraining and downstream tasks share a common task distribution, improved closeness will inherently lead to a lower generalization gap. We begin with a simplified analysis assuming strictly quadratic loss functions to mathematically substantiate this intuition.

\begin{theorem}[Generalization of Closeness in the Quadratic Case, Proof in \cref{appendix:proof:theorem:closeness_generalization:simple}]
\label{theorem:closeness_generalization:simple}
To model the non-convex landscape, assume the parameter space $\mathbb{R}^d$ is partitioned into a set of disjoint basins of attraction $\{\mathcal{B}\}$. Within any specific basin $\mathcal{B}$, assume that any task $\mathcal{L}$ sampled from a distribution $\mathcal{P}$ is locally a quadratic function: $\mathcal{L}(\bm{\theta}) = \frac{a}{2} \|\bm{\theta}-\bm{\theta}_{\mathcal{L}}^*\|_2^2 +c_{\mathcal{B}}$, where $a>0$ is shared across tasks and basins, $\bm{\theta}_{\mathcal{L}}^*$ is the task-specific minimizer in basin $\mathcal{B}$, and $c_{\mathcal{B}}$ is the intrinsic loss (depth) of the basin. We assume the quadratic representation holds at the empirical minimizer within each basin. Let $\bm{\mu}_{\mathcal{B}} =\mathbb{E}_{\mathcal{L}\sim\mathcal{P}} [\bm{\theta}_{\mathcal{L}}^*]$ and define the variance of task-specific minimizers as $\sigma_{\mathcal{B}}^2 = \mathbb{E}_{\mathcal{L}\sim\mathcal{P}} \left[ \|\bm{\theta}_{\mathcal{L}}^* -\bm{\mu}_{\mathcal{B}}\|_2^2 \right] <\infty$. Let the pretraining tasks $\{\mathcal{L}_k\}_{k=1}^K$ and the downstream task $\mathcal{L}_{\mathcal{T}}$ be i.i.d. samples from $\mathcal{P}$. Let $\Theta = \left\{ \bm{\theta}_{\text{train},\mathcal{B}}^* \,\middle|\, \mathbb{E}_{\{\mathcal{L}_k\}} \left[ \mathcal{L}_{\text{train}} (\bm{\theta}_{\text{train},\mathcal{B}}^*) \right] = C_{\text{train}} \right\}$ be the set of converged minimizers across different basins that achieve the same expected pretraining loss $C_{\text{train}}$.

For any fixed basin $\mathcal{B}$ whose converged minimizer belongs to $\Theta$, the expected downstream loss on an unseen task $\mathcal{T}\sim\mathcal{P}$ satisfies
\begin{equation}
\label{eq:closeness_generalization_simple}
    \mathbb{E}_{\{\mathcal{L}_k\},\mathcal{T}} \left[ \mathcal{L}_{\mathcal{T}} (\bm{\theta}_{\text{train},\mathcal{B}}^*) \right] = C_{\text{train}} +\frac{a}{K}\sigma_{\mathcal{B}}^2.
\end{equation}
\end{theorem}

Consequently, as long as downstream tasks and pretraining tasks follow the same distribution, by trading intrinsic loss $c_{\mathcal{B}}$ for improved closeness (i.e., smaller $\sigma_{\mathcal{B}}^2$), one obtains better out-of-distribution generalization due to the reduction in task variance.

We can also extend \cref{theorem:closeness_generalization:simple} beyond purely quadratic loss functions to the broader class of general loss landscapes exhibiting local and directional strong convexity, as demonstrated in the following theorem.

\begin{theorem}[Generalization of Closeness beyond Quadratics, Proof in \cref{appendix:proof:theorem:closeness_generalization:general}]
\label{theorem:closeness_generalization:general}
Let $\{\mathcal{L}_k\}_{k=1}^K$ and $\mathcal{L}_{\mathcal{T}}$ be i.i.d. samples from a task distribution $\mathcal{P}$. Let $\bm{\theta}^*$ be a fixed local minimizer of the population loss $\mathbb{E}_{\mathcal{L}\sim\mathcal{P}}[\mathcal{L}(\bm{\theta})]$, and let
$\bm{\theta}_{\mathcal{L}}^*
=\arg\min_{\bm{\theta}\in\mathcal{S}_{\mathcal{L}}}
\|\bm{\theta}-\bm{\theta}^*\|_2$
be the corresponding task-specific local minimizer, assumed to exist with a measurable selection.
Let $\bm{\theta}_{\mathrm{train}}^*$ be a stationary local minimizer of the empirical pretraining loss.

Assume that the losses are twice continuously differentiable along the relevant segments, and satisfy the following local curvature conditions, uniformly over tasks and almost surely over training samples:

\begin{itemize}
    \item Along $[\bm{\theta}_{\mathcal{L}}^*,\bm{\theta}^*]$, we have
    $\|\nabla^2\mathcal{L}(\bm{\xi})\|_2
    \leq\lambda_{\max}$
    and
    $\bm{u}^{\top}\nabla^2\mathcal{L}(\bm{\xi})\bm{u}
    \geq\lambda_{\min}>0$
    for every unit vector
    \mbox{$\bm{u}\in\operatorname{span}\{\bm{\theta}^*-\bm{\theta}_{\mathcal{L}}^*\mid\mathcal{L}\sim\mathcal{P}\}$}.

    \item Along $[\bm{\theta}^*,\bm{\theta}_{\mathrm{train}}^*]$, the same spectral norm bound holds for every task, while the directional curvature lower bound holds along
    $\bm{\theta}_{\mathrm{train}}^*-\bm{\theta}^*$
    whenever this displacement is nonzero.
\end{itemize}

Let
$\bm{\mu}=\mathbb{E}[\bm{\theta}_{\mathcal{L}}^*]$,
$\sigma^2=\mathbb{E}[
\|\bm{\theta}_{\mathcal{L}}^*-\bm{\mu}\|_2^2]$,
and
$\kappa=\lambda_{\max}/\lambda_{\min}$.
Assume the required expectations are finite and differentiation commutes with expectation at $\bm{\theta}^*$.

Defining $C_{\mathrm{train}} = \mathbb{E}_{\{\mathcal{L}_k\}} [\mathcal{L}_{\mathrm{train}} (\bm{\theta}_{\mathrm{train}}^*)]$, the expected downstream generalization gap satisfies
\begin{equation}
\label{eq:closeness_generalization_general}
    \mathbb{E}_{\{\mathcal{L}_k\},\mathcal{T}} [\mathcal{L}_{\mathcal{T}} (\bm{\theta}_{\mathrm{train}}^*)] -C_{\mathrm{train}}
    \leq \frac{\lambda_{\max}\kappa(\kappa+1)}{2K} \left(1+\frac{(\kappa-1)^2}{4}\right)\sigma^2.
\end{equation}
\end{theorem}

Therefore, as long as the loss landscape exhibits quadratic-like behavior (i.e., local and directional strong convexity) along these typical directions $[\bm{\theta}_k^*, \bm{\theta}^*]$, explicitly optimizing for closeness $\sigma^2=\mathbb{E}[\|\bm{\theta}_{k}^* - \bm{\mu}\|_2^2]$ would be beneficial for a lower downstream loss.

\section{Nexus Optimizer: Enhancing Closeness via Second-Order Approximation}
\label{sec:cwa}

Both geometric intuition and our analysis of quadratic functions support the conclusion that a ``close'' minimizer (\cref{fig:cwa_illustration:close}) generalizes to out-of-distribution data significantly better than a ``distant'' minimizer (\cref{fig:cwa_illustration:distant}), even when achieving the same training loss. Consequently, we aim to explicitly optimize this closeness during LLM pretraining. In this section, we introduce a second-order gradient approximator named ``Nexus'', which effectively optimizes parameter closeness on the training tasks (i.e., $\frac{1}{K}\sum_{k=1}^K \|\bm{\theta}-\bm{\theta}_k^*\|_2^2$), and successfully generalizes to the closeness of unseen downstream tasks (i.e., $\|\bm{\theta}-\bm{\theta}_{\mathcal{T}}^*\|_2^2$).

\subsection{Gradient Similarity Upper Bounds Closeness}

Directly optimizing the closeness metric $\|\bm{\theta}-\bm{\theta}_k^*\|$ involves finding the specific minimizer $\bm{\theta}_k^*$ for each task, which is itself a minimization problem and computationally prohibitive. Fortunately, we observe that the gradient similarity between different source tasks, given by $\sum_{i \neq j} - \nabla \mathcal{L}_i(\bm{\theta})^\top \nabla \mathcal{L}_j(\bm{\theta})$, provides a tractable upper bound for closeness. Intuitively, if the gradients of distinct tasks $\mathcal{L}_k$ consistently align in direction, their respective minimizers must coincide. Theoretically, both the gradient dot product and cosine similarity serve as tight bounds for closeness:

\newpage
\begin{theorem}[Gradient Similarity Upper Bounds Closeness]
\label{thm:gradient_similarity_bound_closeness}
    Let $\bm{\theta}$ be a converged parameter satisfying $\nabla \mathcal{L}_{train}(\bm{\theta}) = \frac{1}{K}\sum_{k=1}^K \nabla \mathcal{L}_k(\bm{\theta}) = \mathbf{0}$. Let $\mathcal{S}_k = \{\bm{\vartheta} \mid \exists \epsilon > 0,\, \forall \bm{\vartheta}' \in B_\epsilon(\bm{\vartheta}), \mathcal{L}_k(\bm{\vartheta}) \leq \mathcal{L}_k(\bm{\vartheta}')\}$ be the set of local minimizers for task $k$, and let $\bm{\theta}_k^* = \arg\min_{\bm{\vartheta}\in\mathcal{S}_k} \|\bm{\vartheta}-\bm{\theta}\|_2$ be a closest local minimizer, assumed to exist. Assume that each $\mathcal{L}_k$ is twice continuously differentiable along $[\bm{\theta}_k^*,\bm{\theta}]$ and that $\nabla\mathcal{L}_k(\bm{\theta})\neq\mathbf{0}$ for all $k$, so that cosine similarity is well-defined. Define $\bm{u}_k = \frac{\bm{\theta}-\bm{\theta}_k^*} {\|\bm{\theta}-\bm{\theta}_k^*\|_2}$, $\lambda_{\min} = \min_k \inf_{\bm{\xi}\in[\bm{\theta}_k^*,\bm{\theta}]} \bm{u}_k^\top \nabla^2\mathcal{L}_k(\bm{\xi}) \bm{u}_k > 0$, and let $G=\sup_k\|\nabla\mathcal{L}_k(\bm{\theta})\|_2$. Then, the closeness between the minimizers is bounded by:
    \begin{equation}
    \frac{1}{K}\sum_{k=1}^K \|\bm{\theta}-\bm{\theta}_k^*\|_2^2  \leq \frac{1}{K\lambda_{\min}^2} \sum_{i\neq j} \left( -\nabla\mathcal{L}_i(\bm{\theta})^\top \nabla\mathcal{L}_j(\bm{\theta}) \right)
    \leq \frac{G^2}{K\lambda_{\min}^2} \sum_{i\neq j} \left( 1-\text{CosSim}( \nabla\mathcal{L}_i(\bm{\theta}), \nabla\mathcal{L}_j(\bm{\theta})) \right).
\end{equation}
\end{theorem}

In other words, optimizing the training trajectory towards a regime where $\text{CosSim}(\nabla \mathcal{L}_i(\bm{\theta}), \nabla \mathcal{L}_j(\bm{\theta}))$ remains consistently high guarantees high closeness (i.e., a small distance $\|\bm{\theta} - \bm{\theta}_k^*\|_2$).
This "gradient similarity upper bound" also provides a more intuitive understanding of why closeness improves downstream generalization. Suppose that the high gradient similarity achieved among training tasks (i.e., high $\text{Sim}(\nabla \mathcal{L}_i(\bm{\theta}), \nabla \mathcal{L}_j(\bm{\theta}))$) successfully generalizes to the similarity between the training objective and the downstream task (i.e., high $\text{Sim}(\nabla \mathcal{L}_{\text{train}}(\bm{\theta}), \nabla \mathcal{L}_{\mathcal{T}}(\bm{\theta}))$). This similarity directly represents the reduction in downstream loss after a single Gradient Descent (GD) step on the training set (in the first-order sense):
\begin{equation}
    \underbrace{\mathcal{L}_{\mathcal{T}}(\bm{\theta}) - \mathcal{L}_{\mathcal{T}}(\bm{\theta}- \gamma \nabla \mathcal{L}_{\text{train}}(\bm{\theta})) }_{\text{decrease of downstream loss after one GD step on training set}} = \gamma \nabla \mathcal{L}_{\text{train}}(\bm{\theta})^\top \nabla \mathcal{L}_{\mathcal{T}}(\bm{\theta}) + O(\gamma^2).
\label{eq:cwa:gradient_similarity_implies_b_decrease_when_optimizing_on_a}
\end{equation}

Therefore, we view gradient similarity $\text{CosSim}(\nabla \mathcal{L}_i(\bm{\theta}), \nabla \mathcal{L}_j(\bm{\theta}))$ as a strong proxy for parameter closeness: it not only provides a tight upper bound on parameter distance (thereby enforcing closeness), but also leads to the same beneficial effects on downstream generalization. Given this strong connection, in the remainder of this paper, we use the term "closeness" to refer to both parameter closeness and gradient closeness.

\subsection{Optimizing Gradient Similarity via Nexus}
\label{sec:cwa:optimizing_using_cwa}

\begin{algorithm}[t]
\caption{Standard Nexus Algorithm}
\label{alg:cwa_standard}
\begin{algorithmic}[1]
    \REQUIRE Initial params $\bm{\theta}_0$, losses $\{\mathcal{L}_i\}_{i=1}^K$, total iterations $T$.
    \REQUIRE Optimizers: $\text{Opt}_{\text{inner}}$ (Normalized SGD), $\text{Opt}_{\text{outer}}$ (e.g., AdamW). Inner learning rate $\gamma$.
    \FOR{$t=1$ {\bfseries to} $T$}
        \STATE $\bm{\theta}_{t, 0} \leftarrow \bm{\theta}_{t-1}$ \COMMENT{Initialize inner loop}
        \FOR{$m=1$ {\bfseries to} $K$}
            \STATE Sample task index $s_m \sim \text{Uniform}(\{1, \dots, K\})$
            \STATE $\bm{g} \leftarrow \nabla \mathcal{L}_{s_m}(\bm{\theta}_{t, m-1})$
            \STATE $\bm{\theta}_{t, m} \leftarrow \bm{\theta}_{t, m-1} - \gamma \cdot \frac{\bm{g}}{\|\bm{g}\|_2}$ \COMMENT{Update inner trajectory}
        \ENDFOR
        \STATE $\hat{\bm{g}}_t \leftarrow \bm{\theta}_{t, 0} - \bm{\theta}_{t, K}$ \COMMENT{Compute Nexus pseudo-gradient}
        \STATE $\bm{\theta}_{t} \leftarrow \text{Opt}_{\text{outer}}(\bm{\theta}_{t-1}, \hat{\bm{g}}_t)$ \COMMENT{Outer-update}
    \ENDFOR
    \textbf{Return}:  $\bm{\theta}_T$
\end{algorithmic}
\end{algorithm}

Therefore, to encourage parameter closeness, it suffices to maximize the gradient similarity. However, directly optimizing this objective is computationally intractable because the gradient of the cosine similarity involves the Jacobian-vector product:
\begin{equation}
    \nabla_{\bm{\theta}} \text{CosSim}(\nabla_{\bm{\theta}} \mathcal{L}_i(\bm{\theta}), \nabla_{\bm{\theta}} \mathcal{L}_j(\bm{\theta})) = \left( \nabla_{\bm{\theta}} \frac{\nabla_{\bm{\theta}} \mathcal{L}_i(\bm{\theta})}{\|\nabla_{\bm{\theta}} \mathcal{L}_i(\bm{\theta})\|_2} \right)^\top \frac{\nabla_{\bm{\theta}} \mathcal{L}_j(\bm{\theta})}{\|\nabla_{\bm{\theta}} \mathcal{L}_j(\bm{\theta})\|_2} 
    + \left( \nabla_{\bm{\theta}} \frac{\nabla_{\bm{\theta}} \mathcal{L}_j(\bm{\theta})}{\|\nabla_{\bm{\theta}} \mathcal{L}_j(\bm{\theta})\|_2} \right)^\top \frac{\nabla_{\bm{\theta}} \mathcal{L}_i(\bm{\theta})}{\|\nabla_{\bm{\theta}} \mathcal{L}_i(\bm{\theta})\|_2}.
    \label{eq:gradient_of_gradient_similarity}
\end{equation}
To address this challenge, we propose the \textbf{Nexus optimizer}, which approximates the gradient in \cref{eq:gradient_of_gradient_similarity} through a dual-loop mechanism. The complete procedure is outlined in \cref{alg:cwa_standard}. Conceptually, one should view each step in the outer loop as a standard parameter update, while the $K$ steps in the inner loop serve as a \textit{gradient approximator} for \cref{eq:gradient_of_gradient_similarity}. Specifically, for each outer iteration, we perform $K$ normalized SGD steps (inner loop) to accumulate the approximated gradient $\hat{\bm{g}}_t$. This $\hat{\bm{g}}_t$ is then passed to the outer optimizer (e.g., AdamW \citep{kingma2014adam,loshchilov2017decoupled}, Muon \citep{jordan2024muon}) to perform the actual update. The following theorem demonstrates that Nexus algorithm effectively maximizes gradient similarity.

\begin{figure}[t]
    \centering
    \includegraphics[width=0.95\linewidth]{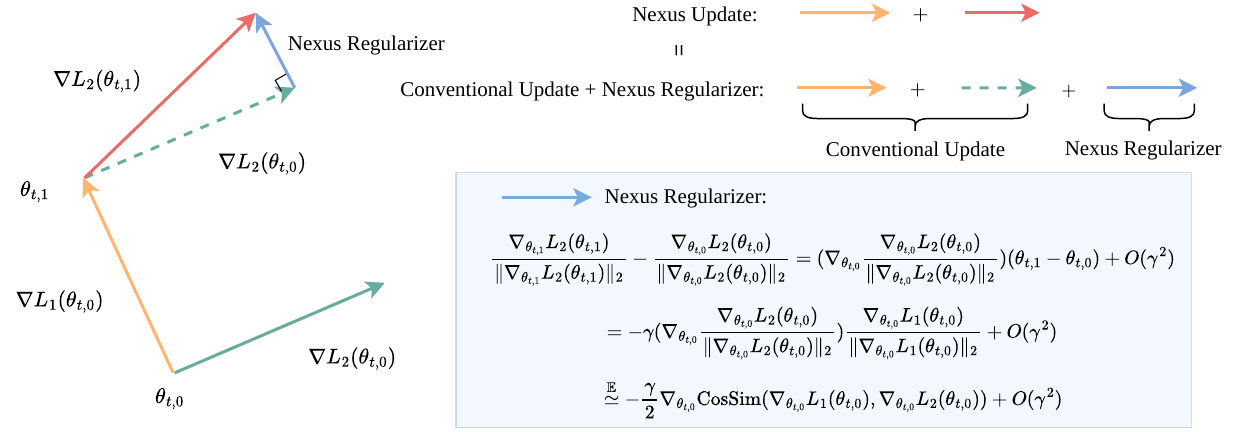}
    \caption{Intuitive illustration of Nexus Algorithm. $\overset{\mathbb{E}}{=}$ denotes equality in expectation over the task permutations.}
    \label{fig:algorithm:cwa}
\end{figure}

\begin{theorem}[Nexus Maximizes Gradient Similarity]
\label{thm:cwa_optimize_gradient_similarity}
Assume there exist constants $G_{\min}, L, \rho > 0$ such that for any $t \in [1, T]$ and $m \in [1, K]$:
\begin{equation}
    \|\nabla \mathcal{L}_i(\bm{\theta}_{t, m})\|_2 \geq G_{\min}; \quad \|\nabla^2 \mathcal{L}_i(\bm{\theta})\|_2 \leq L; \quad \|\nabla^2 \mathcal{L}_i(\bm{x}) - \nabla^2 \mathcal{L}_i(\bm{y})\|_2 \le \rho \|\bm{x} - \bm{y}\|_2.
\end{equation}
    Then, the sequence $\{\bm{\theta}_t\}$ generated by \cref{alg:cwa_standard} minimizes the following second-order objective:
        \begin{equation}
        \mathcal{J}_{\text{2nd}}(\bm{\theta}) = \gamma \sum_{i=1}^K \mathcal{L}_i(\bm{\theta}) - \gamma^2 \frac{K-1}{4K} \sum_{i \neq j} \text{CosSim}\Big(\nabla \mathcal{L}_i(\bm{\theta}), \nabla \mathcal{L}_j(\bm{\theta})\Big).
    \label{eq:cwa_objective}
    \end{equation}
This holds because the expected update direction satisfies:
    \begin{equation}
        \mathbb{E}[\hat{\bm{g}}_t] = \gamma \sum_{i=1}^K \frac{\nabla \mathcal{L}_i(\bm{\theta}_t)}{\|\nabla \mathcal{L}_i(\bm{\theta}_t)\|_2} - \gamma^2 \frac{K-1}{4K} \Big( \nabla_{\bm{\theta}} \sum_{i \neq j} \text{CosSim}(\nabla \mathcal{L}_i, \nabla \mathcal{L}_j) + \bm{\mathcal{E}}_{\text{sym},i,j} \Big) + \bm{\mathcal{E}}_{\text{2nd}},
    \end{equation}
    where $\bm{\mathcal{E}}_{\text{sym},i,j} = \left( \nabla \frac{\nabla \mathcal{L}_i}{\|\nabla \mathcal{L}_i\|_2} - \left( \nabla \frac{\nabla \mathcal{L}_i}{\|\nabla \mathcal{L}_i\|_2} \right)^\top \right) \frac{\nabla \mathcal{L}_j}{\|\nabla \mathcal{L}_j\|_2} \perp \nabla \mathcal{L}_j$ is the Jacobian non-symmetric error, which becomes zero when the gradients align with the top eigenvectors of their respective Hessians~\cite{gur2018gradient,song2024does,cohen2021gradient,damian2022self}. The Taylor approximation error is bounded by $\|\bm{\mathcal{E}}_{\text{2nd}}\|_2 \le \frac{1}{6} \left( \frac{4L^2 + \rho G_{\min}}{G_{\min}^2} \right) K^3 \gamma^3 = \mathcal{O}(\gamma^3).$
\end{theorem}

\textbf{Intuitive Understanding of the Inner Loop.} To intuitively understand why Nexus's inner loop optimizes \cref{eq:gradient_of_gradient_similarity}, consider a simplified scenario with two loss functions, $\mathcal{L}_1$ and $\mathcal{L}_2$, as illustrated in \cref{fig:algorithm:cwa}. At the current parameter state $\bm{\theta}_{t,0}$, a conventional optimizer (e.g., AdamW, Muon) would simply aggregate the gradients as $\nabla \mathcal{L}_1(\bm{\theta}_{t,0}) + \nabla \mathcal{L}_2(\bm{\theta}_{t,0})$ for the update. In contrast, the Nexus inner loop operates sequentially: it first takes a step using $\nabla \mathcal{L}_1(\bm{\theta}_{t,0})$ to reach an intermediate point $\bm{\theta}_{t,1}$, and subsequently evaluates the next gradient $\nabla \mathcal{L}_2(\bm{\theta}_{t,1})$ at this displaced location. As shown in the figure's equations, this sequential trajectory is mathematically equivalent to a conventional update plus a "Nexus regularizer." Crucially, this regularizer naturally yields a Jacobian-vector product $\left( \nabla_{\bm{\theta}_{t,0}} \frac{\nabla \mathcal{L}_2(\bm{\theta}_{t,0})}{\|\nabla \mathcal{L}_2(\bm{\theta}_{t,0})\|_2} \right) \frac{\nabla \mathcal{L}_1(\bm{\theta}_{t,0})}{\|\nabla \mathcal{L}_1(\bm{\theta}_{t,0})\|_2}$, which equals the gradient of the gradient similarity (in the first-order and expectation sense). Consequently, the pseudo-gradient $\hat{\bm{g}}_t$ produced by the inner loop effectively serves as the sum of the gradient of the pretraining loss and the gradient of the gradient similarity defined in \cref{eq:gradient_of_gradient_similarity}.

Thus, Nexus serves as an effective mechanism for maximizing gradient closeness. Strictly speaking, Nexus should be conceptualized as a \textit{gradient approximator} rather than a standalone optimizer, for two reasons: (1) The inner optimization step must be a simple, memory-free update like NSGD or its generalized transformations (see \cref{appendix:discussion:generalized_nexus}). Otherwise, as shown in \cref{sec:exp_validate_theory}, incorporating historical momentum would disrupt the accurate maximization of gradient similarity. (2) The outer optimization step simply consumes the approximated gradient and can utilize standard optimization algorithms (see \cref{appendix:more_exp:sgdm}). Consequently, Nexus acts as a modular plug-in that is broadly compatible with a wide range of outer base optimizers \mbox{(e.g., AdamW, SGDM).}

\subsection{Adapting Nexus to Practical Pretraining}

\begin{figure}[t]
    \centering
    \begin{minipage}{0.48\textwidth}
    \begin{algorithm}[H]
        \caption{Standard Pretraining}
        \label{alg:standard_pretrain}
        \begin{algorithmic}[1]
        \REQUIRE \texttt{model}, \texttt{loader}
        \REQUIRE \texttt{opt\_outer} (e.g., AdamW)
        \REQUIRE \texttt{accum\_steps}
        \FOR{\texttt{i, batch} {\bfseries in} \texttt{loader}}       \State \{Mini-batch Step\}
            \STATE $\mathcal{L} \leftarrow \texttt{model(batch)}$
            \STATE $\mathcal{L}.\texttt{backward()}$
            \IF{\texttt{i} \% \texttt{accum\_steps} $== 0$}          \State \{Accumulation Step\}
                \STATE \texttt{opt\_outer.step()} 
                \STATE \texttt{opt\_outer.zero\_grad()}
            \ENDIF
        \ENDFOR
        \end{algorithmic}
    \end{algorithm}
    \end{minipage}
    \hfill
    \begin{minipage}{0.48\textwidth}
    \begin{algorithm}[H]
        \caption{Nexus (Engineering Adaptation)}
        \label{alg:cwa_adapted}
        \begin{algorithmic}[1]
        \REQUIRE \texttt{model}, \texttt{loader}, \texttt{opt\_outer}, \texttt{accum\_steps}
        \STATE \texttt{inner\_model} $\leftarrow$ \texttt{model.clone()}
        \STATE \texttt{opt\_inner} $\leftarrow$ \texttt{NSGD(inner\_model)}
        \FOR{\texttt{i, batch} {\bfseries in} \texttt{loader}}
            \STATE $\mathcal{L} \leftarrow \texttt{inner\_model(batch)}$
            \STATE $\mathcal{L}.\texttt{backward()}$
            \STATE \texttt{opt\_inner.step()} 
            \IF{\texttt{i} \% \texttt{accum\_steps} $== 0$}
                \STATE $\hat{\bm{g}} \leftarrow$ \texttt{inner\_model} $-$ \texttt{model}
                \STATE \texttt{opt\_outer.step(grad=}$-\hat{\bm{g}}$\texttt{)}
                \STATE \texttt{inner\_model} $\leftarrow$ \texttt{model.clone()} 
            \ENDIF
        \ENDFOR
        \end{algorithmic}
    \end{algorithm}
    \end{minipage}
    \vspace{1em}
    \caption{\textbf{Comparison of Standard Pretraining and Nexus Engineering Adaptation.} 
    \textbf{Left:} Standard training accumulates gradients over multiple mini-batches before performing a single optimizer update (at the micro-batch/accumulation step). 
    \textbf{Right:} We adapt Nexus from \cref{alg:cwa_standard} to pretraining by keeping an auxiliary \texttt{inner\_model}. It performs immediate updates on the \texttt{inner\_model} at every mini-batch step. At the accumulation boundary, the total displacement ($\texttt{inner\_model} - \texttt{model}$) serves as the pseudo-gradient $\hat{\bm{g}}$ for the outer optimizer, after which the inner model is re-synchronized.}
    \label{fig:code_comparison}
\end{figure}

\textbf{Motivation.} We establish that Nexus effectively maximizes gradient similarity with controllable higher-order errors in \cref{thm:cwa_optimize_gradient_similarity}. However, directly applying \cref{alg:cwa_standard} to pretraining is still difficult. This is because \cref{alg:cwa_standard} requires computing gradients for every data source to perform a single effective outer update. In pretraining, the number of data sources is typically large (e.g., $K > 50$), which would result in an effective batch size that differs significantly from standard settings~\cite{kaplan2020scaling, wen2025fantastic}. This prevents us from leveraging established hyperparameters, thereby increasing tuning costs and preventing the wide application of Nexus.

To address this, we propose an engineering adaptation to better adapt Nexus to practical pretraining. As shown in \cref{alg:standard_pretrain}, standard pretraining can be viewed as a gradient accumulation workflow: it computes gradients in every mini-batch and performs an optimizer update in every accumulation step.

\textbf{Adapted Nexus.} Leveraging this structure, we adapt Nexus as illustrated in \cref{alg:cwa_adapted}. Specifically, we introduce an auxiliary \texttt{inner\_model}. For each mini-batch, we perform an immediate Normalized SGD (NSGD) update on this inner model to approximate the hessian-gradient product. Upon completing the accumulation steps, we compute the displacement between the inner model and the frozen main model, using this displacement as the pseudo-gradient $\hat{\bm{g}}$ for the outer optimizer. Therefore, our adapted Nexus actually \textit{maximizes the cosine similarity between mini-batches within a single accumulation step}. Since the pretraining corpus is typically vast and the mixing ratio for each source is typically low, two consecutive mini-batches are highly likely to be sampled from different sources. Thus, this approach effectively achieves the \mbox{objective of \cref{alg:cwa_standard}.}

\textbf{Clarification on NSGD Normalization.} Note that we apply \textit{per-param} normalization rather than normalizing the globally flattened gradient vector. Mathematically, this adapted Nexus maximizes the sum of per-matrix gradient similarities $\sum_{l} \text{CosSim}(\nabla_{\bm{W}^{(l)}} \mathcal{L}_i, \nabla_{\bm{W}^{(l)}} \mathcal{L}_j)$ (see \cref{appendix:discussion:generalized_nexus}), rather than the global gradient similarity $\text{CosSim}(\nabla_{\bm{\theta}} \mathcal{L}_i, \nabla_{\bm{\theta}} \mathcal{L}_j)$. This design choice is primarily motivated by the theoretical insights of $\mu$P \cite{yang2021tensor,yang2021tuning,yang2023spectral} and recent practical wisdom \cite{liu2025muon} on controlling the RMS norm of each update, which potentially enables more stable optimization and better hyper-parameter transfer across different model scales.

\textbf{Remark.} It is worth noting that our adapted Nexus incurs \textbf{almost no extra computational cost}. The total number of forward and backward passes remains exactly the same as standard pretraining. The only computational overhead comes from the copy and update of the inner model, but this is negligible compared to the forward-backward pass (considering the classical $6NBS$ approximation~\cite{kaplan2020scaling}). The only memory overhead comes from the inner model, but this can be reduced to nearly zero through techniques like CPU offloading and asynchronous processing.

We employ \cref{alg:cwa_adapted} for all experiments, with the exception of specific ablation studies. Readers may proceed directly to \cref{sec:exp}. In \cref{appendix:discussion}, we also provide a theoretical analysis of Nexus's convergence speed and discuss its implications for standard Normalized SGD.

\section{Experiments}
\label{sec:exp}

In this section, we validate that Nexus achieves nearly the \textbf{same pretraining loss} while delivering \textbf{better downstream} performance through comprehensive experiments across various datasets, learning rate schedules, model scales and token scales.

\subsection{Experimental Settings}
\label{sec:exp:setting}

Our experimental setup largely follows the protocols established in \citet{wen2025fantastic} and \citet{olmo20252olmo2furious}.

\textbf{Pretraining Datasets.} We utilize an in-house pretraining dataset similar to \cite{seed2025seed-oss}. This corpus is: (1) strictly cleaned to ensure no data contamination regarding the evaluated benchmarks or distillation data; and (2) of higher quality and stability than typical open-source datasets, allowing us to observe smooth and clear optimization trends. We also conduct experiments on public datasets~\cite{basant2025nvidia_nemotron} in \cref{sec:exp:nvidia_nemotron}. However, these public datasets are not strictly decontaminated and contain training samples from our benchmarks. This leads to artificially inflated performance on certain tasks while underperforming on others. Consequently, we primarily rely on the strictly cleaned dataset for more stable analysis.

\textbf{Model Architecture.} Following \citet{wen2025fantastic}, we train Llama-architecture models of 130M, 300M, 520M, 1.2B, and 2.3B parameters (excluding embeddings). We primarily analyze the 520M (1B total parameters) and 2.3B (3B total parameters) models, hereafter referred to by their total parameter counts for brevity, except in the scaling law analysis (\cref{sec:discussion:scaling:num_params}) as required by \citet{kaplan2020scaling}.

\textbf{Hyperparameters.} \citet{wen2025fantastic} have already conducted extensive parameter searches using grid search, coordinate descent, and fine-grained tuning. To ensure fairness, we always apply exact the same hyper-parameters to both Nexus and its corresponding base optimizers. For the base optimizers, we adopt the optimal hyperparameters identified in \citet{wen2025fantastic}. We further verified these settings by sweeping the learning rate with a multiplier of $2$ (i.e., verifying $0.5\times$ and $2.0\times$), confirming that their configurations remain optimal for our dataset. Compared with its base optimizer, Nexus has only one additional hyperparameter, the inner learning rate $\gamma$. See \cref{sec:exp:inner_lr} for details.

\textbf{Benchmarks.} We evaluate on diverse benchmarks encompassing general knowledge (MMLU~\cite{hendrycks2020measuring_mmlu}), reasoning (GPQA, GPQA Diamond~\cite{rein2024gpqa}, BBH~\cite{suzgun2022challenging_bbh}), math (GSM8k~\cite{cobbe2021gsm8k}, MATH500~\cite{hendrycks2021measuring_math500}), and coding (HumanEval~\cite{chen2021codex_humaneval}, MBPP~\cite{austin2021program_mbpp}). Beyond discrete accuracies, we also track downstream task losses and out-of-distribution (OOD) loss. The OOD loss is evaluated on a strictly cleaned proprietary in-house corpus, which exhibits a strong correlation with downstream benchmark capabilities.

\textbf{Highlighting Strategy.} We use \textbf{bold} to highlight non-trivial performance gaps, defined as a loss difference $>0.01$ or a benchmark improvement $>2\%$, following \citet{wen2025fantastic}.

\subsection{Inner Learning Rate}
\label{sec:exp:inner_lr}

\textbf{Motivation.} The inner learning rate $\gamma$ is the only additional hyperparameter introduced by Nexus compared to its base optimizer. To determine its value, we sweep $\gamma \in \{0.01, 0.04, 0.10, 0.40\}$ using the 3B pretraining setting in \cref{sec:exp:main_result}, while keeping all other hyperparameters and the training budget unchanged.

\textbf{Results.} As shown in \cref{eq:cwa_objective}, the relative weights of the pretraining loss and the gradient similarity regularizer scale as $O(1:\gamma)$. Therefore, a larger $\gamma$ gives the Nexus regularizer more weight relative to the pretraining loss. As illustrated in \cref{fig:exp:inner_lr:within,fig:exp:inner_lr:downstream}, increasing $\gamma$ consistently improves gradient similarity both within the pretraining set and between the pretraining and OOD sets. Specifically, as $\gamma$ increases from 0.01 to 0.10, the average within-pretraining gradient similarity increases from 0.4540 to 0.5465, while the pretraining-to-OOD gradient similarity increases from 0.2252 to 0.3042 (\cref{tab:exp:inner_lr}). This improvement in gradient alignment is accompanied by better downstream generalization: the average downstream accuracy increases from 38.2\% to 40.2\%, compared to 37.1\% for the AdamW baseline. Meanwhile, Nexus maintains nearly the same pretraining validation loss as AdamW (1.602 vs. 1.606), as shown in \cref{fig:exp:inner_lr:loss}. However, further increasing $\gamma$ to 0.40, while improving both gradient similarity metrics, increases the pretraining validation loss to 1.622 and reduces the average downstream accuracy to 37.7\%.

\textbf{Conclusion.} Based on these observations, we choose $\gamma=0.10$, the largest tested inner learning rate that preserves the baseline's pretraining loss. This choice allows us to strengthen the Nexus regularizer as much as possible within the tested range, thereby encouraging greater gradient similarity and better downstream generalization, while maintaining the same pretraining loss as the baseline.

\begin{figure}[H]
    \centering
    \includegraphics[width=0.83\linewidth]{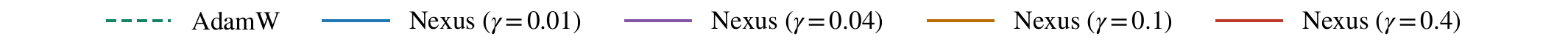}\par\smallskip
    \begin{subfigure}[b]{0.32\linewidth}
        \centering
        \includegraphics[width=\linewidth]{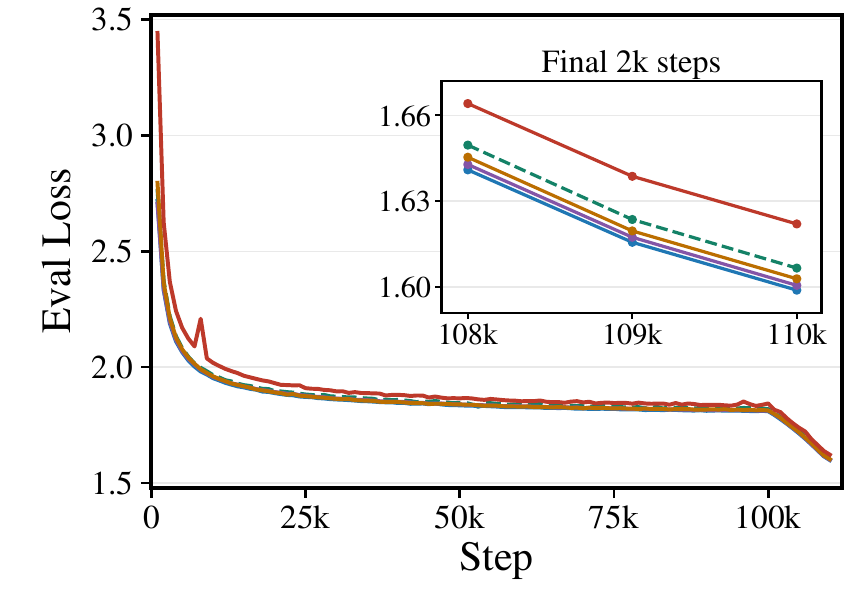}
        \caption{Pretraining Loss}
        \label{fig:exp:inner_lr:loss}
    \end{subfigure}\hfill
    \begin{subfigure}[b]{0.32\linewidth}
        \centering
        \includegraphics[width=\linewidth]{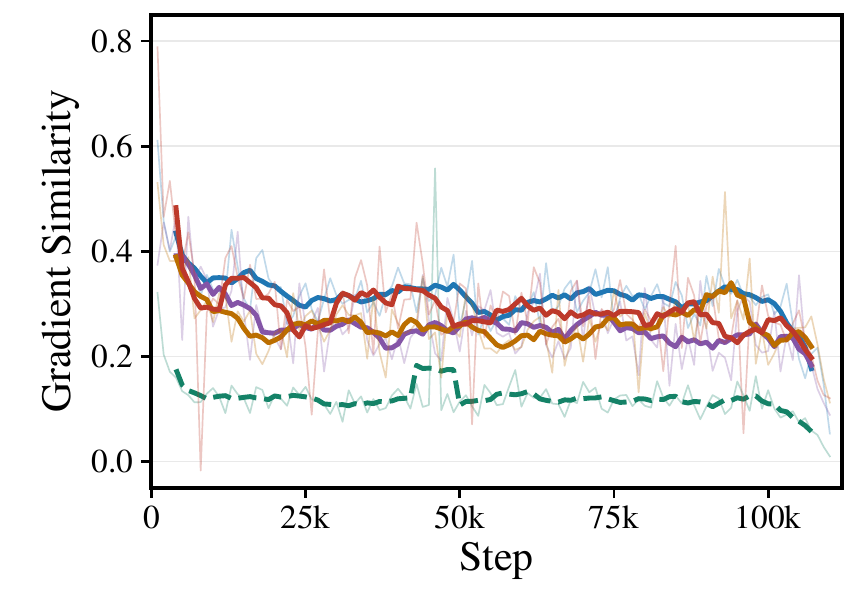}
        \caption{Within-Pretraining Similarity}
        \label{fig:exp:inner_lr:within}
    \end{subfigure}\hfill
    \begin{subfigure}[b]{0.32\linewidth}
        \centering
        \includegraphics[width=\linewidth]{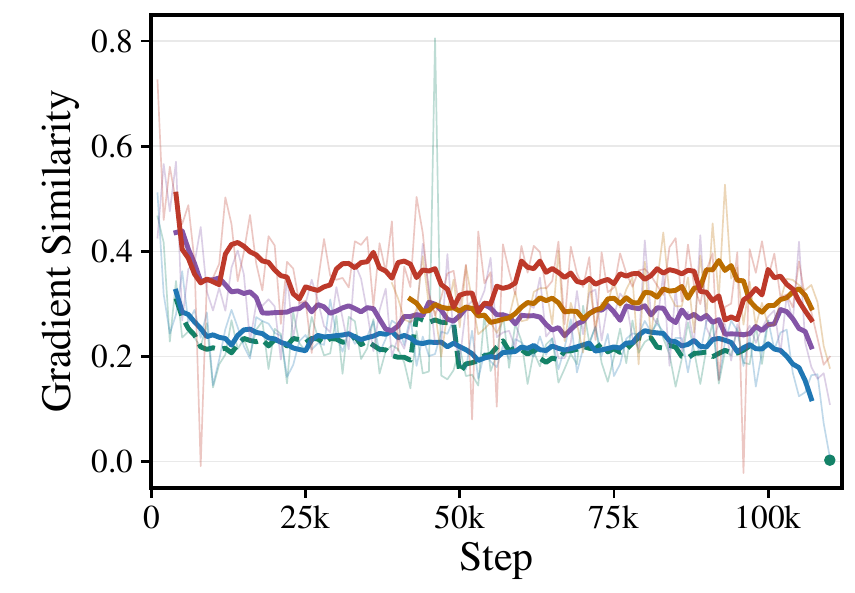}
        \caption{Train-to-Downstream Similarity}
        \label{fig:exp:inner_lr:downstream}
    \end{subfigure}
    \caption{\textbf{Inner learning rate sweep.}
    We use the same 3B pretraining setting as in \cref{sec:exp:main_result}.
    Since the ratio of the loss term to the regularizer is $O(1:\gamma)$, increasing the inner learning rate $\gamma$ strengthens the Nexus regularizer.
    As shown, when the inner LR increases, the gradient similarity within the pretraining set, $\operatorname{CosSim}(\nabla\mathcal{L}_i, \nabla\mathcal{L}_j)$, increases, and this gain carries over to train-to-downstream gradient similarity, $\operatorname{CosSim}(\nabla\mathcal{L}_{\mathrm{train}}, \nabla\mathcal{L}_{\mathcal{T}})$, thus improving downstream generalization.
    However, increasing the inner LR (i.e., the regularization coefficient) also leads to a large pretraining loss.
    Therefore, we choose the largest tested inner learning rate that preserves the baseline's pretraining loss.}
    \label{fig:exp:inner_lr}
\end{figure}

\begin{table}[H]
    \centering
    \caption{\textbf{Detailed numerical results for \cref{fig:exp:inner_lr}.}}
    \label{tab:exp:inner_lr}
    \setlength{\tabcolsep}{6.5pt}
    \renewcommand{\arraystretch}{1.10}
    \begin{tabular}{@{}lcccccc@{}}
        \toprule
        \multirow{2}{*}{\textbf{Optim.}} & \multirow{2}{*}{\textbf{Inner LR}} & \textbf{Pretraining}
            & \multicolumn{2}{c}{\textbf{Gradient Similarity} ($\uparrow$)}
            & \multicolumn{2}{c}{\textbf{Downstream}} \\
        \cmidrule(lr){3-3}\cmidrule(lr){4-5}\cmidrule(lr){6-7}
        & & \textbf{Val. Loss} ($\downarrow$) & \textbf{Within Train Set} & \textbf{Train to OOD}
            & \textbf{Avg. Loss} ($\downarrow$) & \textbf{Avg. Acc.} (\%, $\uparrow$) \\
        \midrule
        AdamW & \textemdash & 1.606 & 0.4499 & 0.2228 & 1.633 & 37.1 \\
        \midrule
        \multirow{4}{*}{Nexus} & 0.01 & 1.598 & 0.4540 & 0.2252 & 1.623 & 38.2 \\
        & 0.04 & 1.600 & 0.5416 & 0.2880 & \textbf{1.612} & \textbf{39.5} \\
        & 0.10 & 1.602 & 0.5465 & 0.3042 & \textbf{1.609} & \textbf{40.2} \\
        & 0.40 & 1.622 & 0.5656 & 0.3514 & 1.629 & 37.7 \\
        \bottomrule
    \end{tabular}
\end{table}

\FloatBarrier

\begin{table*}[t]
    \centering
    \caption{\textbf{Main Results.} Comparison of validation losses and downstream capabilities. Notably, Nexus consistently achieves nearly \textit{identical pretraining losses} compared to the base optimizers, yet demonstrates \textit{superior performance across downstream losses and benchmarks}. }
    \label{tab:exp:d803_main_results}
    \setlength{\tabcolsep}{0.35ex} 
    \resizebox{\textwidth}{!}{
        \begin{tabular}{lll cc c ccc cc cc c}
            \toprule
            \multirow{2}{*}{\textbf{Model}} & \multirow{2}{*}{\textbf{Optim.}} & \multirow{2}{*}{\textbf{Metric}} & 
            \multicolumn{2}{c}{\textbf{Loss Metrics} ($\downarrow$)} & 
            \multicolumn{1}{c}{\textbf{Gen.}} & 
            \multicolumn{3}{c}{\textbf{Reasoning}} & 
            \multicolumn{2}{c}{\textbf{Math}} & 
            \multicolumn{2}{c}{\textbf{Code}} &
            \multicolumn{1}{c}{\textbf{Avg.}} \\
            \cmidrule(lr){4-5} \cmidrule(lr){6-6} \cmidrule(lr){7-9} \cmidrule(lr){10-11} \cmidrule(lr){12-13} \cmidrule(lr){14-14}
             & & & Pretrain. & OOD & MMLU & GPQA & GPQA-D & BBH & GSM8k & MATH & HumanEval & MBPP & All \\
            \midrule
            
            \multirow{6}{*}{1B} 
              & \multirow{2}{*}{AdamW} & Acc. ($\uparrow$) & \multirow{2}{*}{1.826} & \multirow{2}{*}{1.433} & 32.1 & 25.0 & 21.8 & 29.6 & 18.0 & 13.0 & 19.0 & 17.0 & 21.9 \\
              & & Loss ($\downarrow$) & & & 2.363 & 2.221 & 2.124 & 1.640 & 1.429 & 1.204 & 1.270 & 2.035 & 1.786 \\
              \cmidrule{2-14}
              & \multirow{2}{*}{Nexus} & Acc. ($\uparrow$) & \multirow{2}{*}{1.826} & \multirow{2}{*}{1.428} & 33.5 & \textbf{30.4} & 21.8 & 29.3 & 20.0 & 13.0 & 19.0 & \textbf{22.0} & 23.6 \\
              & & Loss ($\downarrow$) & & & \textbf{2.316} & \textbf{2.201} & \textbf{2.102} & 1.638 & \textbf{1.396} & \textbf{1.176} & 1.261 & \textbf{1.977} & \textbf{1.758} \\
              \cmidrule{2-14}
              & \multirow{2}{*}{\textit{Improv.}} & Acc. ($\uparrow$) & - & - & +1.4 & \textbf{+5.4} & 0.0 & -0.3 & +2.0 & 0.0 & 0.0 & \textbf{+5.0} & +1.7 \\
              & & Loss ($\uparrow$) & 0.000 & +0.005 & \textbf{+0.047} & \textbf{+0.020} & \textbf{+0.022} & +0.002 & \textbf{+0.033} & \textbf{+0.028} & +0.009 & \textbf{+0.058} & \textbf{+0.027} \\
            \midrule
            
            \multirow{6}{*}{3B} 
              & \multirow{2}{*}{AdamW} & Acc. ($\uparrow$) & \multirow{2}{*}{1.606} & \multirow{2}{*}{1.302} & 47.8 & \textbf{32.8} & 22.6 & 36.6 & 44.0 & 32.0 & 43.0 & 38.0 & 37.1 \\
              & & Loss ($\downarrow$) & & & 2.265 & 2.005 & 1.910 & 1.534 & 1.259 & 1.054 & 1.116 & 1.922 & 1.633 \\
              \cmidrule{2-14}
              & \multirow{2}{*}{Nexus} & Acc. ($\uparrow$) & \multirow{2}{*}{1.602} & \multirow{2}{*}{\textbf{1.290}} & 48.9 & 29.6 & \textbf{23.4} & 36.6 & \textbf{59.0} & \textbf{40.0} & \textbf{47.0} & 38.0 & \textbf{40.3} \\
              & & Loss ($\downarrow$) & & & \textbf{2.179} & \textbf{1.981} & \textbf{1.881} & \textbf{1.504} & \textbf{1.227} & \textbf{1.026} & \textbf{1.086} & 1.921 & \textbf{1.601} \\
              \cmidrule{2-14}
              & \multirow{2}{*}{\textit{Improv.}} & Acc. ($\uparrow$) & - & - & +1.1 & -3.2 & +0.8 & 0.0 & \textbf{+15.0} & \textbf{+8.0} & \textbf{+4.0} & 0.0 & \textbf{+3.2} \\
              & & Loss ($\uparrow$) & +0.004 & \textbf{+0.012} & \textbf{+0.086} & \textbf{+0.024} & \textbf{+0.029} & \textbf{+0.030} & \textbf{+0.032} & \textbf{+0.028} & \textbf{+0.030} & +0.001 & \textbf{+0.032} \\
            
            \bottomrule
        \end{tabular}
    }
\end{table*}

\subsection{Main Experimental Results}
\label{sec:exp:main_result}

\textbf{Settings.} We train 1B models by 4$\times$ Chinchilla and 3B models for 2$\times$ Chinchilla tokens using two optimizer configurations: the standard AdamW baseline, AdamW equipped with our Nexus regularizer (Nexus).

\textbf{Nexus achieves Same Pretraining Loss, Better Downstream Task.} As detailed in \cref{tab:exp:d803_main_results}, Nexus strictly satisfies the ``same pretraining loss'' condition, showing an immaterial difference of 0.004 compared to the baseline. Despite this parity in pretraining loss, Nexus demonstrates substantial improvements across nearly all evaluated out-of-distribution and downstream metrics. Specifically, it reduces the OOD validation loss by 0.012 and yields significant accuracy gains on complex reasoning benchmarks, including a +15.0\% improvement on GSM8k, +8.0\% on MATH, and +4.0\% on HumanEval. These consistent gains across diverse domains validate our core hypothesis: optimizing closeness unlocks downstream generalization in the same \mbox{pretraining loss regime.}

\textbf{Comparison of Muon and Nexus.} Compared to the standard AdamW baseline, Muon reduces the pretraining loss by 0.029 and improves the average downstream accuracy by 2.3\%. In contrast, Nexus achieves a negligible 0.004 reduction in pretraining loss yet yields a 3.2\% improvement in average downstream accuracy, reaching a downstream performance level comparable to Muon (see \cref{tab:exp:d803:muon}). This observation indicates a fundamental divergence in their optimization pathways: while Muon's downstream improvements rely primarily on achieving a lower pretraining loss, the gains from Nexus stem from its implicit bias despite maintaining a nearly identical pretraining loss as AdamW.

\textbf{Output Analysis.} Compared to the AdamW baseline, Nexus improves accuracy by 15.0\% on GSM8k, 8.0\% on MATH, and 4.0\% on HumanEval. To investigate the source of these improvements, we analyze the model outputs on these benchmarks. We observe that the set of correctly answered questions by Nexus is almost a strict superset of those answered correctly by AdamW. Specifically, on GSM8k and HumanEval, Nexus retains a >95\% retention rate on the questions already solved by AdamW, while the 15.0\% net improvement stems entirely from exclusively solving previously failed questions. This additive behavior indicates that the performance gain provided by Nexus over the base optimizer is highly stable, expanding the capability boundaries without regressing on previously learned knowledge.

\subsection{Scaling Analysis on Model Size}
\label{sec:discussion:scaling:num_params}

\textbf{Motivation.} In the following two subsections, we investigate the scalability of Nexus across model size and training duration (tokens). Prevailing literature on implicit bias suggests that the role of implicit regularization becomes increasingly prominent with greater overparameterization and extended computational budgets, since sufficient expressive power and optimization steps grant the model the flexibility to satisfy the geometric implicit bias without compromising the minimization of the pretraining loss~\cite{wen2022does,belkin2019reconciling,power2022grokking,zhang2016understanding,neyshabur2014search,soudry2018implicit,lyu2019gradient}. Since Nexus operates via such implicit bias, we hypothesize that its downstream generalization benefits will also amplify at larger compute and model scales.

\begin{figure*}[t]
    \centering
    \begin{subfigure}[b]{0.32\linewidth}
        \centering
        \includegraphics[width=\linewidth]{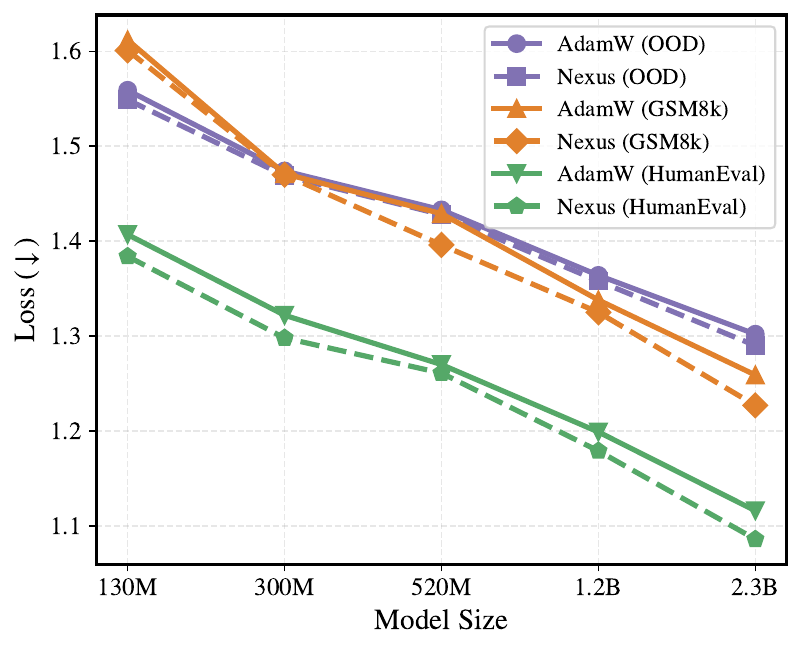}
        \caption{Downstream Loss}
        \label{fig:exp:scaling:loss}
    \end{subfigure}
    \hfill
    \begin{subfigure}[b]{0.32\linewidth}
        \centering
        \includegraphics[width=\linewidth]{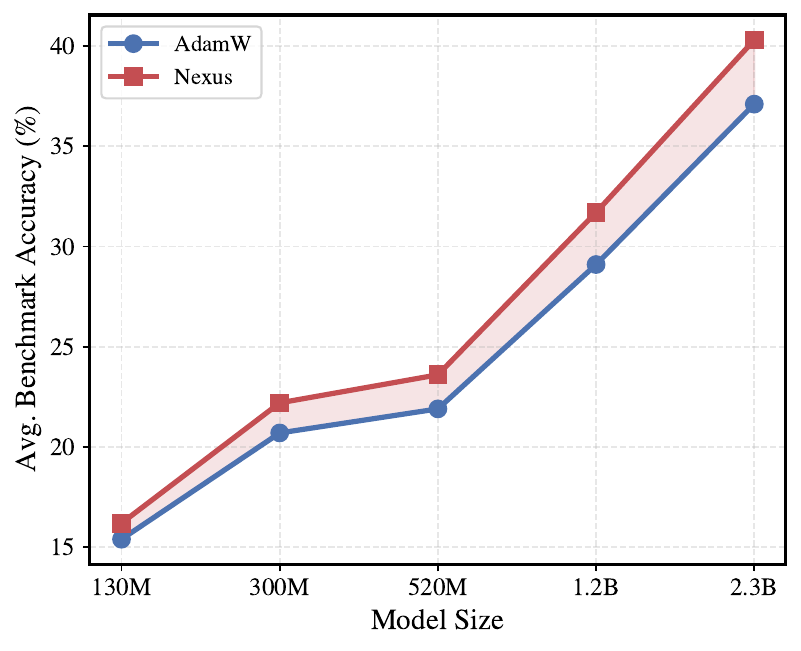}
        \caption{Downstream Benchmark}
        \label{fig:exp:scaling:acc}
    \end{subfigure}
    \hfill
    \begin{subfigure}[b]{0.32\linewidth}
        \centering
        \includegraphics[width=\linewidth]{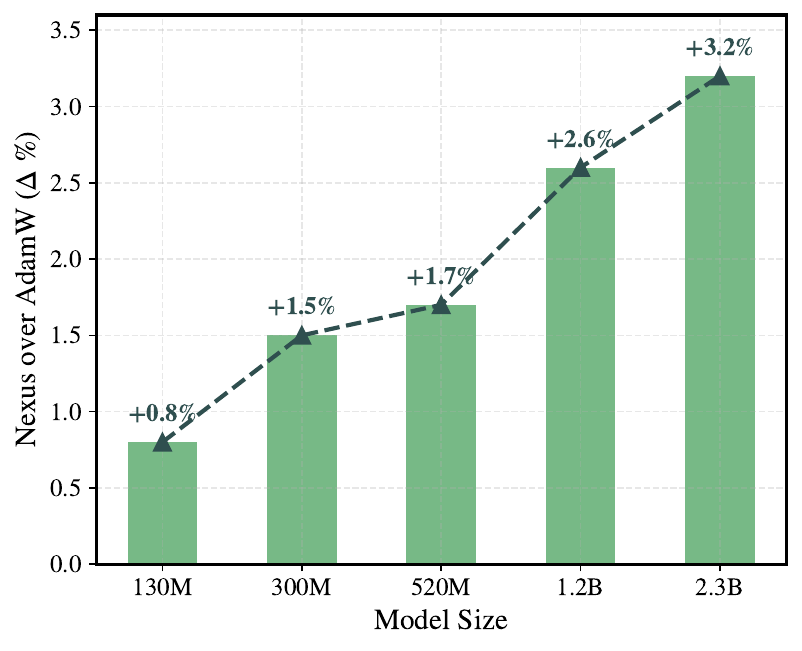}
        \caption{Average Benchmark Gain}
        \label{fig:exp:scaling:gap}
    \end{subfigure}
    \caption{\textbf{Benchmark Performance across Model Scales.} We compare downstream capabilities for models ranging from 130M to 2.3B parameters. Notably, the relative gains of Nexus amplify as model capacity increases, with the average benchmark accuracy improvement growing from +0.8\% on the 130M model to +3.2\% on the 2.3B model.}
    \label{fig:exp:scaling_main}
\end{figure*}

\textbf{Settings.} We evaluate models across five distinct sizes as outlined in \cref{sec:exp:setting}. Please refer to \cref{appendix:hyper_params} for the detailed hyperparameters of each experiment. The results are shown in \cref{tab:exp:d803_scaling,fig:exp:scaling_main}.

\textbf{Universal ``Same Pretraining Loss, Better Downstream''.} Across all model sizes ranging from 130M to 2.3B, Nexus consistently maintains the pretraining validation loss within a negligible margin (defined as $\Delta < 0.01$ in \cref{sec:exp:setting}) compared to the baseline, satisfying "same pretraining loss." Despite this parity in pretraining loss, Nexus achieves non-trivial loss reduction on nearly all downstream tasks. For instance, at the 1.2B scale, while the validation loss difference is merely $0.007$, Nexus reduces MMLU loss by $0.086$, and both BBH and HumanEval losses by $0.03$, more than 7 times larger than the pretraining loss gap. 

\textbf{Performance Gains Amplify with Scale.} We observe that the relative advantage of Nexus over the AdamW baseline expands monotonically as model capacity increases. Specifically, the average benchmark accuracy improvements across the five evaluated scales are +0.8\% (130M), +1.5\% (300M), +1.7\% (520M), +2.6\% (1.2B), and +3.2\% (2.3B). This amplification is particularly pronounced in complex reasoning tasks: the accuracy gap on GSM8k widens from negligible levels at the 130M scale to +15.0\% (59.0 vs. 44.0) at the 2.3B scale, accompanied by a 0.032 reduction in downstream loss. These results demonstrate that Nexus scales favorably with model capacity, effectively leveraging the increased expressive power to enforce the \mbox{geometric closeness bias.}

\subsection{Scaling Analysis on Training Tokens}
\label{sec:discussion:scaling:tokens}

\begin{table*}[t]
    \centering
    \caption{\textbf{Scaling Analysis on Training Tokens.} We extend the pretraining duration of the 3B model from 2$\times$ to 4$\times$ Chinchilla optimal tokens. The results demonstrate that the downstream performance advantage of Nexus over the AdamW baseline persists strictly.}
    \label{tab:exp:scaling:tokens}
    \setlength{\tabcolsep}{0.35ex} 
    \resizebox{\textwidth}{!}{
        \begin{tabular}{cll cc c ccc cc cc c}
            \toprule
            \multirow{2}{*}{\textbf{Chinchila}} & \multirow{2}{*}{\textbf{Optim.}} & \multirow{2}{*}{\textbf{Metric}} & 
            \multicolumn{2}{c}{\textbf{Loss Metrics} ($\downarrow$)} & 
            \multicolumn{1}{c}{\textbf{Gen.}} & 
            \multicolumn{3}{c}{\textbf{Reasoning}} & 
            \multicolumn{2}{c}{\textbf{Math}} & 
            \multicolumn{2}{c}{\textbf{Code}} &
            \multicolumn{1}{c}{\textbf{Avg.}} \\
            \cmidrule(lr){4-5} \cmidrule(lr){6-6} \cmidrule(lr){7-9} \cmidrule(lr){10-11} \cmidrule(lr){12-13} \cmidrule(lr){14-14}
             & & & Pretrain. & OOD & MMLU & GPQA & GPQA-D & BBH & GSM8k & MATH & HumanEval & MBPP & All \\
            \midrule
            
            \multirow{5}{*}{2} 
              & \multirow{2}{*}{AdamW} & Acc. ($\uparrow$) & \multirow{2}{*}{1.606} & \multirow{2}{*}{1.302} & 47.8 & \textbf{32.8} & 22.6 & 36.6 & 44.0 & 32.0 & 43.0 & 38.0 & 37.1 \\
              & & Loss ($\downarrow$) & & & 2.265 & 2.005 & 1.910 & 1.534 & 1.259 & 1.054 & 1.116 & 1.922 & 1.633 \\
              \cmidrule{2-14}
              & \multirow{2}{*}{Nexus} & Acc. ($\uparrow$) & \multirow{2}{*}{1.602} & \multirow{2}{*}{\textbf{1.290}} & 48.9 & 29.6 & 23.4 & 36.6 & \textbf{59.0} & \textbf{40.0} & \textbf{47.0} & 38.0 & \textbf{40.3} \\
              & & Loss ($\downarrow$) & & & \textbf{2.179} & \textbf{1.981} & \textbf{1.881} & \textbf{1.504} & \textbf{1.227} & \textbf{1.026} & \textbf{1.086} & 1.921 & \textbf{1.601} \\
              \cmidrule{2-14}
              & \textit{Improv.} & Loss ($\uparrow$) & +0.004 & \textbf{+0.012} & \textbf{+0.086} & \textbf{+0.024} & \textbf{+0.029} & \textbf{+0.030} & \textbf{+0.032} & \textbf{+0.028} & \textbf{+0.030} & +0.001 & \textbf{+0.032} \\
            \midrule
            
            \multirow{5}{*}{4} 
              & \multirow{2}{*}{AdamW} & Acc. ($\uparrow$) & \multirow{2}{*}{1.591} & \multirow{2}{*}{1.293} & 48.3 & \textbf{23.4} & 21.9 & 35.2 & 54.0 & 33.0 & 45.0 & 43.0 & 38.0 \\
              & & Loss ($\downarrow$) & & & 2.240 & 1.975 & 1.880 & 1.513 & 1.245 & 1.038 & 1.119 & 1.976 & 1.623 \\
              \cmidrule{2-14}
              & \multirow{2}{*}{Nexus} & Acc. ($\uparrow$) & \multirow{2}{*}{1.588} & \multirow{2}{*}{\textbf{1.281}} & \textbf{52.8} & 20.3 & \textbf{25.0} & \textbf{44.1} & \textbf{62.0} & 33.0 & \textbf{49.0} & \textbf{47.0} & \textbf{41.7} \\
              & & Loss ($\downarrow$) & & & \textbf{2.216} & \textbf{1.957} & \textbf{1.863} & \textbf{1.501} & \textbf{1.229} & \textbf{1.008} & \textbf{1.087} & \textbf{1.885} & \textbf{1.593} \\
              \cmidrule{2-14}
              & \textit{Improv.} & Loss ($\uparrow$) & +0.003 & \textbf{+0.012} & \textbf{+0.024} & \textbf{+0.018} & \textbf{+0.017} & \textbf{+0.012} & \textbf{+0.016} & \textbf{+0.030} & \textbf{+0.032} & \textbf{+0.091} & \textbf{+0.030} \\
            
            \bottomrule
        \end{tabular}
    }
\end{table*}

\textbf{Settings.} To evaluate scalability with respect to compute, we extend the training duration of the 3B model from the standard 2$\times$ Chinchilla optimal token count to 4$\times$ Chinchilla optimal (i.e., doubling the original training time). All other configurations, including the data mixture, model architecture, and base optimizer hyperparameters, remain strictly identical to those in the main experiments in \cref{sec:exp:setting}.

\textbf{The advantage of Nexus does not diminish with more training tokens.} As shown in \cref{tab:exp:scaling:tokens}, while the AdamW baseline naturally improves with extended training (average accuracy increasing from 37.1 to 38.0), it still fundamentally lags behind Nexus. Notably, the overall performance gap between Nexus and AdamW does not shrink with more tokens; Nexus at 4$\times$ Chinchilla achieves an average accuracy of 41.7, effectively maintaining and even slightly widening its substantial lead over the baseline. This confirms that the current implicit bias of standard SGD is insufficient to naturally reach optimal geometric closeness, making Nexus's explicit regularization strictly necessary even under extended compute budgets.

\subsection{Robustness to Data Mixing}
\label{sec:exp:data_mixture}

\textbf{Motivation.} In \cref{sec:cwa:optimizing_using_cwa} and \cref{eq:cwa:gradient_similarity_implies_b_decrease_when_optimizing_on_a}, we show that the gradient similarity implies the marginal gains on task $i$ when optimizing on task $j$ (in the first-order sense):
\begin{equation}
    \underbrace{\mathcal{L}_{i}(\bm{\theta}) - \mathcal{L}_{i}(\bm{\theta}- \gamma \nabla \mathcal{L}_{j}(\bm{\theta})) }_{\text{decrease of task i after one GD step on task j}} = \gamma \nabla \mathcal{L}_{i}(\bm{\theta})^\top \nabla \mathcal{L}_{j}(\bm{\theta}) + O(\gamma^2).
\end{equation}
Since Nexus encourages gradient similarity across the training set, optimizing a sample-dense domain implicitly optimizes sample-sparse domains. Therefore, we conjecture that Nexus acts like a dynamic data mixture, which boosts the sample-sparse or harder-to-learn domains within the mixture without manual re-weighting.

\textbf{Setup.} To validate our hypothesis, we construct three distinct data mixtures by explicitly anchoring the sampling weight of the mathematics domain to 10\%, 40\%, and 70\% (denoted as Math10, Math40, and Math70). Accordingly, we downsample the remaining data sources to fulfill the complementary proportion (e.g., Math70 consists of 70\% math and 30\% downsampled other data). We train 3B models on each mixture using both AdamW and Nexus, strictly adhering to the hyperparameter settings detailed in \cref{sec:exp:setting}.

\begin{figure*}[t]
    \centering
    \begin{subfigure}[b]{0.31\linewidth}
        \centering
        \includegraphics[width=\linewidth]{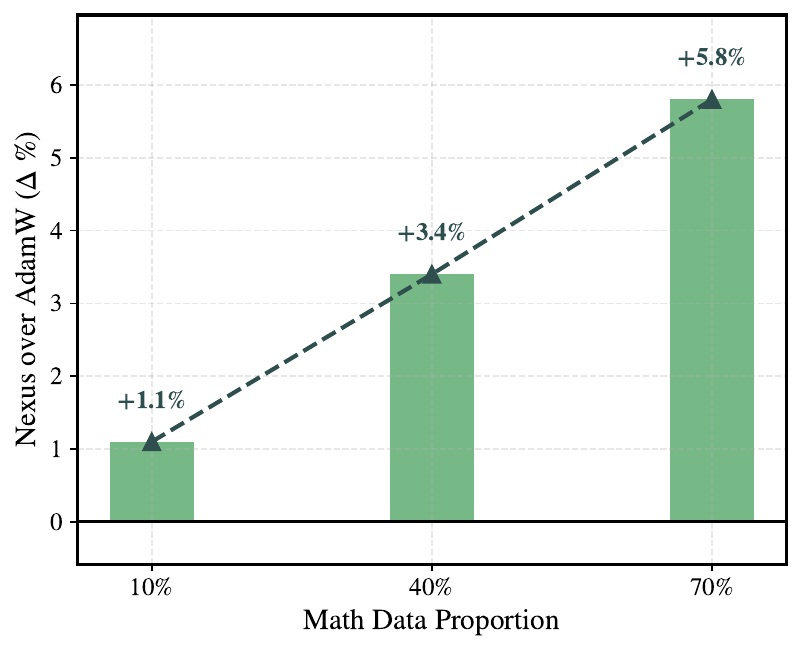}
        \caption{General Domain (MMLU)}
        \label{fig:exp:mixture:gen}
    \end{subfigure}
    \hfill
    \begin{subfigure}[b]{0.31\linewidth}
        \centering
        \includegraphics[width=\linewidth]{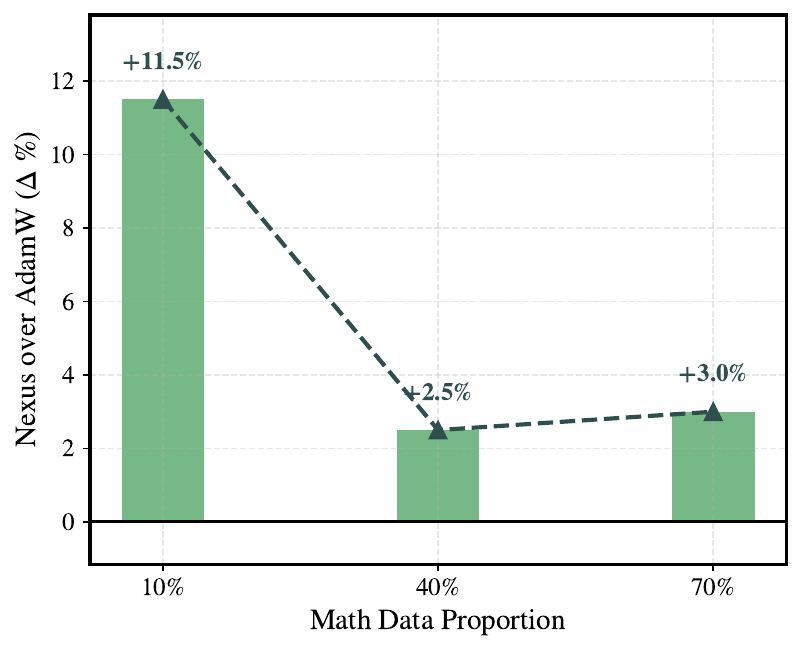}
        \caption{Math Domain (Avg.)}
        \label{fig:exp:mixture:math}
    \end{subfigure}
    \hfill
    \begin{subfigure}[b]{0.31\linewidth}
        \centering
        \includegraphics[width=\linewidth]{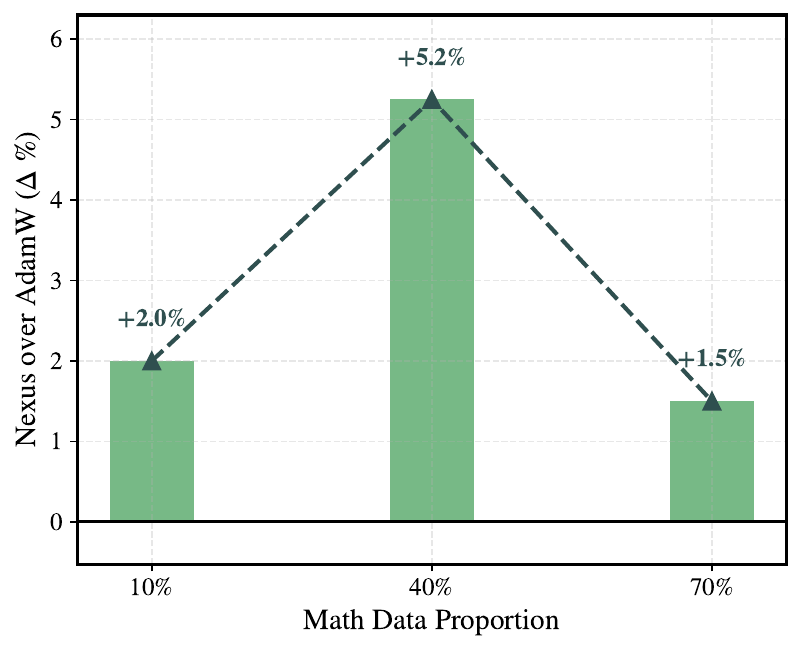}
        \caption{Code Domain (Avg.)}
        \label{fig:exp:mixture:code}
    \end{subfigure}
    \caption{\textbf{Results on varying data mixtures} (3B models).  As the proportion of math data increases (10\% $\to$ 70\%), the relative performance gains of Nexus on math benchmarks gradually diminish, whereas its advantages on General domain progressively expand. This suggests Nexus boosts the sample-sparse or harder-to-learn domains in the mixture.}
    \label{fig:exp:mixture:fig}
\end{figure*}

\textbf{Results.} As shown in \cref{fig:exp:mixture:fig} and \cref{tab:exp:d803_data_mixture}, increasing the proportion of math data from 10\% to 70\% gradually reduces Nexus's relative gain on math reasoning. Conversely, as general data becomes the relative minority, Nexus yields a larger improvement in this domain, increasing its gain from \textbf{+1.1\%} to \textbf{+5.8\%}. Interestingly, the gain on coding tasks exhibits a non-monotonic trend, which we hypothesize is because code generation is a composite capability requiring a complex balance of both logical reasoning and domain knowledge. Furthermore, Nexus mitigates the performance fluctuations observed in the baseline across these mixture shifts. These results support our conjecture that Nexus acts as an implicit balancer, dynamically prioritizing under-optimized tasks without manual mixture tuning.

\subsection{Robustness to Learning Rate Schedule}
\label{sec:exp:lr_schedule}

\textbf{Motivation.} While the Warmup-Stable-Decay (WSD) scheduler \citep{hu2024minicpm} has become increasingly popular in recent LLM pretraining, the Cosine annealing schedule remains a widely adopted standard~\cite{wen2025understanding,wen2025fantastic}. To ensure that our observed generalization benefits are not merely an artifact of a specific learning rate dynamic, we evaluate the robustness of Nexus across different schedulers.

\textbf{Settings.} We conduct an ablation study by replacing the default WSD scheduler with a standard Cosine learning rate scheduler. All other training configurations, including the 3B model architecture, data mixture, and base optimizer hyperparameters, remain strictly identical to the main setup detailed in \cref{sec:exp:main_result}.

\begin{table*}[t]
    \centering
    \caption{\textbf{Results under different learning rate schedulers.} We evaluate the 3B model trained with AdamW and Nexus using both WSD and Cosine schedulers. The results demonstrate that the ``same pretraining loss, better downstream performance'' phenomenon is highly robust regardless of the scheduler.}
    \label{tab:exp:d803_lr_schedule}
    \setlength{\tabcolsep}{0.35ex} 
    \resizebox{\textwidth}{!}{
        \begin{tabular}{lll cc c ccc cc cc c}
            \toprule
            \multirow{2}{*}{\textbf{Schedule}} & \multirow{2}{*}{\textbf{Optim.}} & \multirow{2}{*}{\textbf{Metric}} & 
            \multicolumn{2}{c}{\textbf{Loss Metrics} ($\downarrow$)} & 
            \multicolumn{1}{c}{\textbf{Gen.}} & 
            \multicolumn{3}{c}{\textbf{Reasoning}} & 
            \multicolumn{2}{c}{\textbf{Math}} & 
            \multicolumn{2}{c}{\textbf{Code}} &
            \multicolumn{1}{c}{\textbf{Avg.}} \\
            \cmidrule(lr){4-5} \cmidrule(lr){6-6} \cmidrule(lr){7-9} \cmidrule(lr){10-11} \cmidrule(lr){12-13} \cmidrule(lr){14-14}
             & & & Eval & OOD & MMLU & GPQA & GPQA-D & BBH & GSM8k & MATH & HumanEval & MBPP & All \\
            \midrule
            
            \multirow{5}{*}{WSD} 
              & \multirow{2}{*}{AdamW} & Acc. ($\uparrow$) & \multirow{2}{*}{1.606} & \multirow{2}{*}{1.302} & 47.8 & \textbf{32.8} & 22.6 & 36.6 & 44.0 & 32.0 & 43.0 & 38.0 & 37.1 \\
              & & Loss ($\downarrow$) & & & 2.265 & 2.005 & 1.910 & 1.534 & 1.259 & 1.054 & 1.116 & 1.922 & 1.633 \\
              \cmidrule{2-14}
              & \multirow{2}{*}{Nexus} & Acc. ($\uparrow$) & \multirow{2}{*}{1.602} & \multirow{2}{*}{\textbf{1.290}} & 48.9 & 29.6 & \textbf{23.4} & 36.6 & \textbf{59.0} & \textbf{40.0} & \textbf{47.0} & 38.0 & \textbf{40.3} \\
              & & Loss ($\downarrow$) & & & \textbf{2.179} & \textbf{1.981} & \textbf{1.881} & \textbf{1.504} & \textbf{1.227} & \textbf{1.026} & \textbf{1.086} & 1.921 & \textbf{1.601} \\
              \cmidrule{2-14}
              & \textit{Improv.} & Loss ($\uparrow$) & +0.004 & \textbf{+0.012} & \textbf{+0.086} & \textbf{+0.024} & \textbf{+0.029} & \textbf{+0.030} & \textbf{+0.032} & \textbf{+0.028} & \textbf{+0.030} & +0.001 & \textbf{+0.032} \\
            \midrule
            
            \multirow{5}{*}{Cosine} 
              & \multirow{2}{*}{AdamW} & Acc. ($\uparrow$) & \multirow{2}{*}{1.526} & \multirow{2}{*}{1.255} & 53.2 & 26.6 & 19.5 & \textbf{41.5} & 60.0 & 32.0 & 56.0 & 39.0 & 41.0 \\
              & & Loss ($\downarrow$) & & & 2.195 & 1.924 & 1.829 & 1.480 & 1.212 & 1.022 & 1.045 & 1.867 & 1.572 \\
              \cmidrule{2-14}
              & \multirow{2}{*}{Nexus} & Acc. ($\uparrow$) & \multirow{2}{*}{1.528} & \multirow{2}{*}{1.250} & 54.9 & \textbf{30.5} & \textbf{27.3} & 34.8 & 59.0 & \textbf{41.0} & 54.0 & \textbf{46.0} & \textbf{43.4} \\
              & & Loss ($\downarrow$) & & & \textbf{2.115} & 1.917 & 1.826 & 1.479 & \textbf{1.169} & \textbf{0.994} & \textbf{1.025} & \textbf{1.805} & \textbf{1.541} \\
              \cmidrule{2-14}
              & \textit{Improv.} & Loss ($\uparrow$) & -0.002 & +0.005 & \textbf{+0.080} & +0.007 & +0.003 & +0.001 & \textbf{+0.043} & \textbf{+0.028} & \textbf{+0.020} & \textbf{+0.062} & \textbf{+0.030} \\
            
            \bottomrule
        \end{tabular}
    }
\end{table*}

\textbf{Results.} As demonstrated in \cref{tab:exp:d803_lr_schedule}, the "same pretraining loss, better downstream performance" phenomenon persists consistently across both schedulers. Under the Cosine schedule, Nexus maintains a negligible pretraining loss difference compared to the AdamW baseline (1.528 vs. 1.526) while delivering substantial improvements on downstream metrics, such as a +0.03 loss gain on downstream benchmarks. This confirms that the implicit bias introduced by Nexus is highly robust and orthogonal to the choice of \mbox{learning rate trajectory.}

\section{Discussions}
\label{sec:discussion}

In this section, we conduct several interesting ablation studies of Nexus.

\subsection{Training Dynamic of Nexus}
\label{sec:exp_validate_theory}

\textbf{Settings.} We analyze the training trajectories of the 3B AdamW and 3B Nexus models from \cref{sec:exp:main_result}. During pretraining, we record the gradient cosine similarity between test set and each downstream corpus every 1,000 steps and compute the average to approximate the averaged gradient similarity during training. Upon the completion of pretraining, we perform full batch Gradient Descent using AdamW with learning rate $2\times 10^{-5}$ and weight decay $0$ on each downstream task $\mathcal{L}_{\mathcal{T}}$ to locate the respective task-specific minimizer $\bm{\theta}_{\mathcal{T}}^*$ for subsequent visualization and distance evaluation.

\begin{table*}[t]
    \centering
    \caption{\textbf{Analysis of Gradient Similarity, Loss, and Benchmarks.} By optimizing gradient similarity within the pretraining corpus, Nexus achieves higher gradient similarity between the pretraining corpus and downstream corpus. Consistent with \cref{eq:cwa:gradient_similarity_implies_b_decrease_when_optimizing_on_a}, this first order gradient similarity directly translates into lower zero-th order downstream losses, ultimately yielding better benchmark performance.}
    \label{tab:exp:validate_theory:analysis_grad_loss_bench}
    \resizebox{\textwidth}{!}{
        \begin{tabular}{ll ccccccc}
            \toprule
            \multirow{2}{*}{\textbf{Metric}} & \multirow{2}{*}{\textbf{Optim.}} & 
            \textbf{Pretrain Set} & \textbf{OOD Set} & \textbf{GPQA-D} & \textbf{GSM8k} & \textbf{Math500} & \textbf{HumanEval} & \textbf{MBPP} \\
             & & & & & & & & \\
            \midrule

            \multirow{2}{*}{\textbf{Grad Sim. }($\uparrow$)} 
              & AdamW & 0.4499 & 0.2228 & 0.0824 & \textbf{0.0374} & 0.0422 & 0.0367 & 0.0091 \\
              & Nexus  & \textbf{0.4661} & \textbf{0.2464} & \textbf{0.0924} & 0.0325 & \textbf{0.0427} & \textbf{0.0382} & \textbf{0.0092} \\
            \midrule
            \multirow{2}{*}{\textbf{Param. Closeness. }($\downarrow$)} 
              & AdamW & 1.452 & 2.812 & 4.500 & 3.418 & 3.775 & 4.472 & \textbf{3.645} \\
              & Nexus  & \textbf{1.441} & \textbf{2.806} & \textbf{4.482} & \textbf{3.326} & \textbf{3.766} & \textbf{4.444} & 3.648 \\
            \midrule
            \multirow{2}{*}{\textbf{Loss }($\downarrow$)} 
              & AdamW & 1.606 & 1.302 & 1.910 & 1.259 & 1.054 & 1.116 & 1.922 \\
              & Nexus  & 1.602 & \textbf{1.290} & \textbf{1.881} & \textbf{1.227} & \textbf{1.026} & \textbf{1.086} & 1.921 \\
            \midrule
            
            \multirow{2}{*}{\textbf{Benchmark }($\uparrow$)} 
              & AdamW & - & - & 22.6 & 44.0 & 32.0 & 43.0 & 38.0 \\
              & Nexus  & - & - & 23.4 & \textbf{59.0} & \textbf{40.0} & \textbf{47.0} & 38.0 \\
            
            \bottomrule
        \end{tabular}
    }
\end{table*}

\textbf{Nexus Encourages Training Set Closeness.} As demonstrated in \cref{tab:exp:validate_theory:analysis_grad_loss_bench}, Nexus effectively increases the gradient similarity across the pretraining set, $\mathbb{E}_{i \neq j}[\text{CosSim}(\nabla \mathcal{L}_i, \nabla \mathcal{L}_j)]$, compared to the base optimizer, as analyzed in \cref{thm:cwa_optimize_gradient_similarity}.

\textbf{Training Set Closeness Generalizes to Downstream Closeness.} Fortunately, this gradient closeness generalizes beyond the pretraining corpus to unseen downstream tasks $\mathcal{T}$, effectively increasing the similarity between the training objective and the downstream task, $\text{CosSim}(\nabla \mathcal{L}_{\text{train}}, \nabla \mathcal{L}_{\mathcal{T}})$.

\textbf{Downstream Closeness Yields Smaller Downstream Loss and Better Performance.} Since this gradient similarity $\text{CosSim}(\nabla \mathcal{L}_{\text{train}}, \nabla \mathcal{L}_{\mathcal{T}})$ generalizes, optimizing the pretraining objective inherently optimizes the downstream tasks, as indicated by the first-order approximation in \cref{eq:cwa:gradient_similarity_implies_b_decrease_when_optimizing_on_a}. This gradient closeness translates into lower downstream losses and better benchmark performance.

\textbf{Empirical Landscapes Align with \cref{fig:cwa_illustration}.} As shown in \cref{fig:exp:ablation:loss_landscape}, Nexus reduces the downstream loss $\mathcal{L}_{\mathcal{T}}(\bm{\theta}_{\text{train}}^*)$ by decreasing the geometric distance $\|\bm{\theta}_{\mathcal{T}}^* - \bm{\theta}_{\text{train}}^*\|_2$ between the converged parameter and the task-specific minimizer. This observation matches the analyses in \cref{theorem:closeness_generalization:simple,theorem:closeness_generalization:general}. While Nexus reduces this distance, it does not cause all minima to nearly intersect as depicted in \cref{fig:cwa_illustration:close}---which would theoretically yield a nearly 0\% OOD generalization error. Instead, it achieves a moderate reduction in geometric distance, leading to a proportionally lower downstream loss. We hope future work can design stronger Nexus variants capable of approaching this extreme closeness without introducing significant computational overhead.

\subsection{Implicit Biases of Other Optimizers}
\label{sec:discussion:implicit_bias_of_muon}

\textbf{Motivation.} To explicitly demonstrate the "same pretraining loss, better downstream task" phenomenon and analyze the implicit biases of different optimizers, we visualize the correlation between pretraining and downstream losses, using the results of AdamW, Muon, and AdamW-Nexus from \cref{sec:exp:main_result}. We plot the averaged downstream loss (y-axis) against the pretraining validation loss (x-axis) at corresponding checkpoints. The results are presented in \cref{fig:exp:ablation:loss_bench_plot}.

\textbf{Muon does not possess a superior implicit bias.} As illustrated in \cref{fig:exp:ablation:loss_bench_plot}, the curves for Adam and pure Muon almost completely overlap. This indicates that despite its orthogonalization mechanism, Muon seems not to introduce a favorable implicit bias for downstream generalization beyond what is explained by the pretraining loss itself. This observation aligns with the findings in \citet{wen2025fantastic}, which suggest that for Muon-like optimizers, achieving the same pretraining loss typically translates to the same downstream performance. While recent work \cite{wang2025muon} demonstrates that Muon tends to optimize towards representations with a higher weight rank than Adam, empirical results suggest that this structural difference in weight matrices seems not inherently translate into observable generalization benefits on downstream tasks.

\textbf{Implicit bias does not stem from gradient normalization.} To further isolate the source of Nexus's generalization benefits, we conduct an ablation study where the normalized gradient $\bm{g}/\|\bm{g}\|_2$ is directly fed into the Adam optimizer instead of the raw gradient $\bm{g}$. The results are shown as the NSGD curve in \cref{fig:exp:ablation:loss_bench_plot}. We observe that this variant still fails to introduce any favorable implicit bias, which closely overlaps with the standard Adam baseline. This indicates that the downstream gains of Nexus do not originate from the mere act of normalizing gradients. Mathematically, this ablation is strictly equivalent to executing the Nexus algorithm with an inner loop step count of $K=1$. This empirical observation perfectly aligns with \cref{thm:cwa_optimize_gradient_similarity}: when $K=1$, the coefficient of the gradient similarity regularizer $\frac{K-1}{4K}$ becomes strictly zero, stripping the optimizer of its consensus-seeking property and reducing it to a purely first-order method.

\begin{figure}[t]
    \centering
    \begin{subfigure}[b]{0.31\linewidth}
        \centering
        \includegraphics[width=\linewidth]{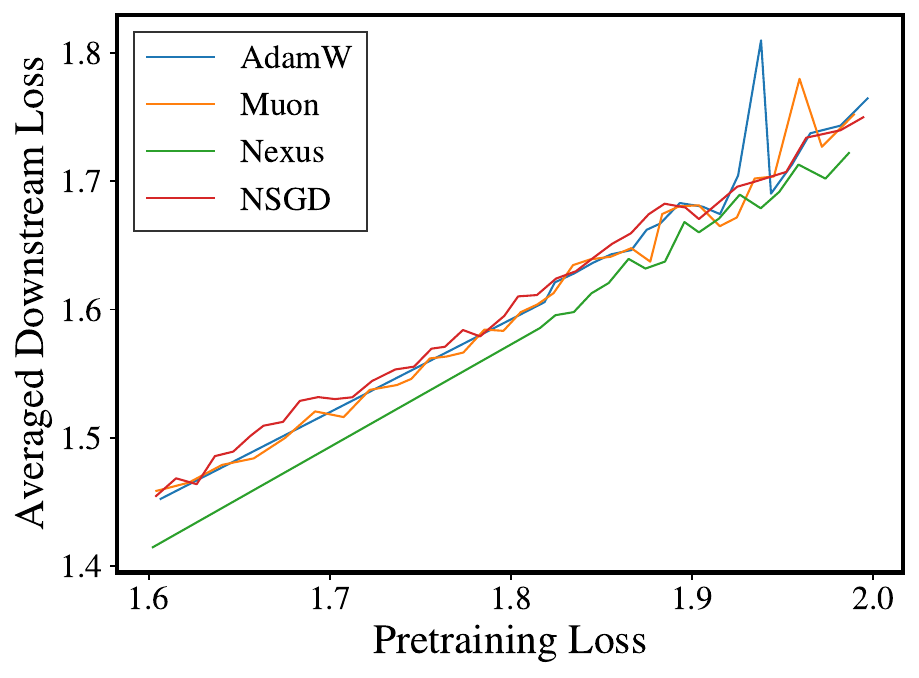}
        \caption{Implicit Biases}
        \label{fig:exp:ablation:loss_bench_plot}
    \end{subfigure}
    \begin{subfigure}[b]{0.31\linewidth}
        \centering
        \includegraphics[width=\linewidth]{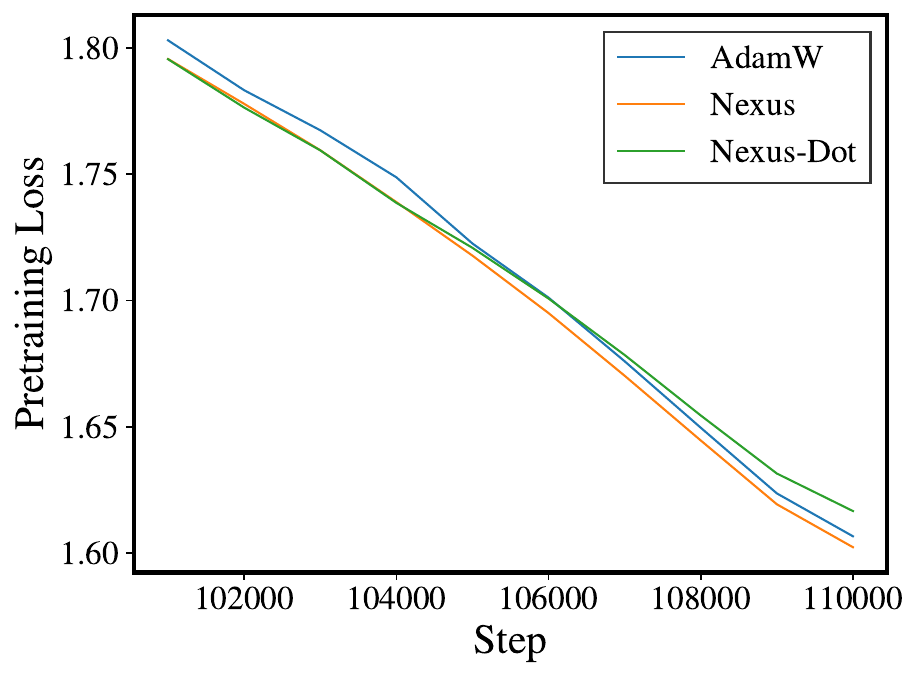}
        \caption{Nexus-Dot}
        \label{fig:exp:ablation:dot_cwa}
    \end{subfigure}
    \begin{subfigure}[b]{0.31\linewidth}
        \centering
        \includegraphics[width=\linewidth]{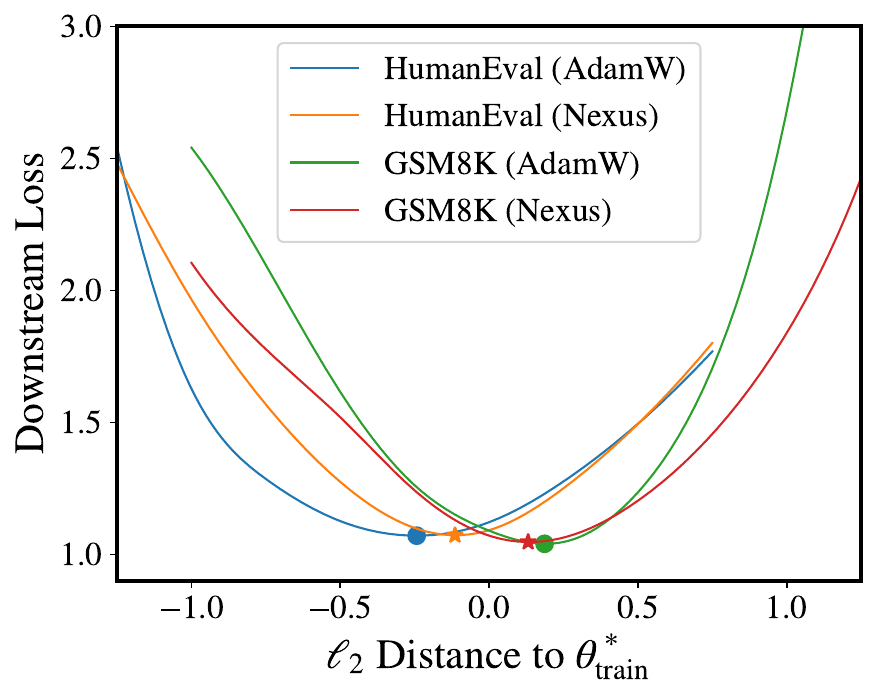}
        \caption{Loss Landscape}
        \label{fig:exp:ablation:loss_landscape}
    \end{subfigure}
    \caption{\textbf{Ablation Studies.} (a) Implicit biases of various optimizers, illustrated by the correlation between pretraining and downstream losses. (b) Pretraining trajectory of Nexus-Dot, demonstrating that optimizing the unnormalized dot product disrupts pretraining loss minimization. (c) Loss landscape visualization of Adam and Nexus.}
    \label{fig:exp:ablation}
\end{figure}

\subsection{Cosine Similarity Instead of Dot Product Similarity}

Although as discussed in \cref{sec:cwa}, the dot product similarity of gradients offers a more direct theoretical connection—yielding a tighter bound for parameter closeness (\cref{thm:gradient_similarity_bound_closeness}) and a more straightforward interpretation for downstream generalization (\cref{eq:cwa:gradient_similarity_implies_b_decrease_when_optimizing_on_a})—it proves practically challenging to optimize. 

This difficulty primarily arises because the dot product objective introduces a pathological optimization shortcut. Specifically, the dot product is highly scale-dependent: if the overall loss magnitude scales by a factor of $k$, the gradient norm scales proportionally by $k$, causing the dot product similarity to artificially inflate by a factor of $k^2$. Consequently, directly maximizing the dot product severely disrupts the primary minimization of the pretraining loss, as the optimizer may exploit this shortcut by inadvertently increasing the gradient norms rather than discovering genuine task consensus.

As demonstrated in \cref{fig:exp:ablation:dot_cwa}, the optimization trajectory of Nexus-Dot lags significantly behind the standard Adam baseline. The resulting degradation in pretraining loss heavily outweighs any potential generalization benefits conferred by its implicit bias. Therefore, we adopt cosine similarity (via normalized gradients) as our primary regularization objective in this work. Note that the progressive deceleration of Nexus-Dot observed in \cref{fig:exp:ablation:dot_cwa} is a persistent geometric phenomenon, occurring consistently regardless of the learning rate scheduler or the choice of base optimizer (e.g., AdamW or Muon). Due to space constraints, we selectively present the ablation results for the 3B model with Adam, corresponding to the main setup in \cref{sec:exp:main_result}.

\section{Conclusion and Limitation}
\label{sec:conclusion}

In this work, we investigate the geometric closeness of minimizers of different losses in LLM pretraining. We show that this closeness strongly correlates with downstream generalization. To optimize this closeness, we propose the Nexus algorithm, which encourages gradient similarity across different tasks. We show that both gradient closeness and geometric closeness generalize to downstream tasks, thus leading to lower downstream loss and better downstream performance. Experimental results across various settings validate our claims. We reckon that as the LLM scaling paradigm transitions from being compute-bound to data-bound, explicitly engineering the implicit biases of optimizers to unlock generalization may serve as a critical frontier for developing more capable language models.

\textbf{Limitations.} Despite its empirical success and theoretical consistency on AdamW, Nexus currently remains incompatible with the Muon optimizer. Specifically, Muon combined with Nexus even underperforms the AdamW-Nexus configuration on downstream tasks, due to its deceleration on Muon (in contrast to the slight acceleration observed with AdamW as demonstrated in \cref{sec:cwa:convergence_rate}). We hypothesize this may be due to several subtle factors, such as numerical sensitivities involving the pseudo-gradient coefficient $\gamma$ (see \cref{eq:cwa_objective}) or potential interactions arising from the Newton-Schulz iterations. We are currently investigating these challenges and aim to resolve this incompatibility in future work.

\section*{Acknowledgement}

This work was supported by the National Natural Science Foundation of China (NSFC) under Grant No. 625B2104. Additional support was provided by the NSFC Special Fund and Major Program under Grants 62550004 and 62595773, and the Fundamental and Interdisciplinary Disciplines Breakthrough Plan of the Ministry of Education of China (No. JYB2025XDXM101).

This work was conducted for research and validation purposes only. The algorithms and methodologies described herein are experimental prototypes and have not been integrated into any commercial products or services of the affiliated organizations.

We thank Kaiyue Wen, Jiacheng You, Haodong Wen, Yan Wu, Jianhui Duan, and Chengyin Xu for their insightful comments and helpful discussions.

\clearpage

\bibliographystyle{plainnat}
\bibliography{ref}

\clearpage

\beginappendix

\crefalias{section}{appendix}
\crefalias{subsection}{appendix}
\crefalias{subsubsection}{appendix}

\section{Preliminaries}
\label{sec:notations_and_assumptions}

\subsection{Related Work}

\textbf{Pretraining Optimizers.} Various optimizers have been proposed for LLM pretraining, such as AdamW~\cite{kingma2014adam,loshchilov2017decoupled}, Muon~\cite{jordan2024muon}, Lion~\cite{chen2023symbolic}, SOAP~\cite{vyas2024soap}, Sophia~\cite{liu2023sophia}, MARS~\cite{yuan2025mars}, Adam-mini~\cite{zhang2024adammini}, Cautious~\cite{liang2024cautious}, and Scion~\cite{pethick2025training_scion}. Most of these optimizers primarily focus on achieving faster optimization and accelerating the convergence of the pretraining loss. In contrast, our work focuses on the generalization aspect, aiming to achieve better downstream task performance under the same pretraining loss. Since Nexus serves as a gradient approximator prior to the optimizer update, our method is largely orthogonal to the choice of the base optimizer and may be used in conjunction with these existing optimization algorithms.

\textbf{Implicit Biases.} Implicit bias is a well-studied topic in deep learning theory, as it is closely related to downstream generalization, especially under the zero training loss or same pretraining loss~\cite{soudry2018implicit,lyu2019gradient,gunasekar2018implicit,li2018algorithmic,ji2018gradient,arora2019implicit,li2020towards,haochen2021shape,liu2023same}. Our work differs from previous studies in two main aspects. First, we focus on gradient and parameter closeness, whereas previous studies primarily investigate sharpness and margins. Second, we evaluate our method in modern auto-regressive LLM pretraining settings, as opposed to standard CNN, DNN, BERT or adversarial attack~\cite{wu2018sgd,chen2023rethinking_model_ensemble}. A concurrent work~\cite{watts2026sharpness} also demonstrates improved downstream performance at the same pretraining loss in auto-regressive LLMs, but their approach relies on flatness rather than closeness.

\textbf{Meta-Learning and Decentralized Learning.} The nested inner-outer loop optimization structure utilized by Nexus has been extensively explored in fields like meta-learning \cite{finn2017model,nichol2018reptile,raghu2019rapid,li2017meta,fallah2020convergence} and decentralized learning \cite{mcmahan2017communication,stich2018local,karimireddy2020scaffold,li2020federated}. However, rather than aiming for few-shot fast adaptation during SFT or communication efficiency, our work repurposes this structure with two distinct contributions. First, we provide a rigorous mathematical characterization demonstrating that this inner-outer loop mechanism inherently optimizes gradient similarity (\cref{thm:cwa_optimize_gradient_similarity,appendix:discussion:generalized_nexus}). Second, we apply this mechanism to LLM pretraining, demonstrating its empirical value in achieving better downstream task performance under the same pretraining loss. Due to these different objectives, the practical algorithmic designs also diverge. For instance, while some meta-learning methods can flexibly incorporate momentum in the inner loop, the inner update in Nexus must strictly remain a memory-free step (such as vanilla or normalized SGD, see \cref{appendix:discussion:generalized_nexus}); otherwise, historical momentum would unequally weight the first-order loss gradients from different steps. We believe other algorithms from the meta-learning literature, such as meta-optimizers~\cite{andrychowicz2016learning,ravi2017optimization}, hypernetworks~\cite{ha2016hypernetworks}, and meta-learned data weighting schemes~\cite{ren2018learning,shu2019meta}, might also yield unexpected implicit regularization benefits when applied to large-scale pretraining, which we leave as an open direction for future work.

\subsection{Notations}
The primary mathematical notations for data mixtures, loss functions, geometries, and optimization dynamics are summarized in \cref{tab:notations}. 

\begin{table}[ht]
\centering
\caption{Summary of key notations used in this paper.}
\label{tab:notations}
\begin{tabular}{ll}
\toprule
\textbf{Notation} & \textbf{Description} \\
\midrule
\multicolumn{2}{c}{\textit{Data and Loss Functions}} \\
\midrule
$K$ & Total number of distinct pretraining data sources (tasks). \\
$\alpha_k$ & Sampling probability (data mixing ratio) for the $k$-th data source. \\
$\mathcal{L}_k(\bm{\theta})$ & The expected / empirical loss on the $k$-th source task. \\
$\mathcal{L}_{\text{train}}(\bm{\theta})$ & The averaged pretraining loss: $\frac{1}{K} \sum_{k=1}^K \mathcal{L}_k(\bm{\theta})$. \\
$\mathcal{L}_{\mathcal{T}}(\bm{\theta})$ & The loss on an unseen downstream evaluation task $\mathcal{T}$. \\
\midrule
\multicolumn{2}{c}{\textit{Geometric and Statistical Variables}} \\
\midrule
$\bm{\theta}_{\text{train}}^*$ & The converged parameter state that minimizes $\mathcal{L}_{\text{train}}$. \\
$\mathcal{S}_k, \mathcal{S}_{\mathcal{T}}$ & The set of local minimizers for task $k$ and downstream task $\mathcal{T}$, respectively. \\
$\bm{\theta}_k^*, \bm{\theta}_{\mathcal{T}}^*$ & The specific local minimizer in $\mathcal{S}_k$ or $\mathcal{S}_{\mathcal{T}}$ closest to the current parameter. \\
$\bm{\mu}$ & The statistical center of task-specific minimizers: $\mathbb{E}[\bm{\theta}_k^*]$. \\
$\sigma^2$ & The intrinsic variance (Closeness) of task-specific minimizers: $\mathbb{E}[\|\bm{\theta}_k^* - \bm{\mu}\|_2^2]$. \\
\midrule
\multicolumn{2}{c}{\textit{Optimization and Nexus Variables}} \\
\midrule
$\gamma$ & The inner learning rate (step size) used in the Nexus gradient approximator. \\
$\hat{\bm{g}}_t$ & The Nexus pseudo-gradient (displacement) passed to the outer optimizer. \\
$\text{CosSim}(\bm{x}, \bm{y})$ & The cosine similarity between two vectors: $\frac{\bm{x}^\top \bm{y}}{\|\bm{x}\|_2 \|\bm{y}\|_2}$. \\
$S_{ij}(\bm{\theta})$ & Shorthand for gradient similarity: $\text{CosSim}(\nabla \mathcal{L}_i(\bm{\theta}), \nabla \mathcal{L}_j(\bm{\theta}))$. \\
\bottomrule
\end{tabular}
\end{table}

\subsection{Assumptions}
The main assumptions used in our analysis are outlined below~\cite{cohen2025understanding,wen2025understanding}. Additional assumptions required for specific analyses will be stated in the respective theorems.

\begin{assumption}[Bounded Gradients]
\label{assum:bounded_gradients}
For all tasks $i \in [1, K]$ and parameters $\bm{\theta}$ along the optimization trajectory, the gradient norm is strictly bounded from below and above: 
$$ 0 < G_{\min} \le \|\nabla \mathcal{L}_i(\bm{\theta})\|_2 \le G $$
This ensures that the Normalized SGD step in Nexus is always well-defined and numerically stable.
\end{assumption}

\begin{assumption}[Smoothness and Bounded Curvature]
\label{assum:smoothness}
The loss function $\mathcal{L}_i$ is $L$-smooth, meaning its Hessian spectral norm is bounded from above. Furthermore, within the local basin of attraction $[\bm{\theta}, \bm{\theta}_k^*]$, the curvature is strictly lower-bounded by $\lambda_{min} > 0$:
$$ \lambda_{min} \le \inf_{\bm{\xi} \in [\bm{\theta}, \bm{\theta}_k^*]} \left( \bm{u}^\top \nabla^2 \mathcal{L}_i(\bm{\xi}) \bm{u} \right) \le \|\nabla^2 \mathcal{L}_i(\bm{\theta})\|_2 \le L $$
where $\bm{u}$ is any unit vector. Prior literature extensively characterizes the local loss landscape of deep neural networks as exhibiting high quadraticity, particularly along meaningful optimization trajectories~\cite{chen2025understanding,visualoss,wen2022does}. Consequently, under the standard premise that the loss landscape can be locally and directionally approximated by a quadratic function, this bounded curvature condition should not be viewed as a restrictive assumption.
\end{assumption}

\begin{assumption}[Hessian Lipschitz Continuous]
\label{assum:hessian_lipschitz}
The Hessian matrix is $\rho$-Lipschitz continuous. For any parameters $\bm{x}, \bm{y}$:
$$ \|\nabla^2 \mathcal{L}_i(\bm{x}) - \nabla^2 \mathcal{L}_i(\bm{y})\|_2 \le \rho \|\bm{x} - \bm{y}\|_2 $$
This assumption is necessary to bound the Jacobian of the normalized gradient during the second-order Taylor expansion in Nexus's inner loop.
\end{assumption}

\clearpage
\section{Proofs of Main Theorems}

\subsection{Proofs for Closeness Improving Generalization}

\subsubsection{Proof for \cref{theorem:closeness_generalization:simple}}
\label{appendix:proof:theorem:closeness_generalization:simple}

\begin{proof}
Fix a basin $\mathcal{B}$. By the stationarity condition of the quadratic pretraining loss, the converged parameter is the mean of the task-specific minimizers:
\begin{equation}
    \bm{\theta}_{\text{train},\mathcal{B}}^* = \frac{1}{K}\sum_{k=1}^K\bm{\theta}_k^*.
\end{equation}

Since the task-specific minimizers are i.i.d. with mean $\bm{\mu}_{\mathcal{B}}$ and variance $\sigma_{\mathcal{B}}^2$, we have
\begin{equation}
    \mathbb{E}_{\{\mathcal{L}_k\}} \left[ \|\bm{\theta}_{\text{train},\mathcal{B}}^* -\bm{\mu}_{\mathcal{B}}\|_2^2 \right] = \frac{\sigma_{\mathcal{B}}^2}{K}.
\end{equation}

Using the variance decomposition,
\begin{equation}
    \frac{1}{K}\sum_{k=1}^K \|\bm{\theta}_{\text{train},\mathcal{B}}^* -\bm{\theta}_k^*\|_2^2
    = \frac{1}{K}\sum_{k=1}^K \|\bm{\theta}_k^* -\bm{\mu}_{\mathcal{B}}\|_2^2 - \|\bm{\theta}_{\text{train},\mathcal{B}}^* -\bm{\mu}_{\mathcal{B}}\|_2^2.
\end{equation}

Therefore, the expected pretraining loss is
\begin{equation}
\label{eq:quadratic_expected_train}
    C_{\text{train}} = \mathbb{E}_{\{\mathcal{L}_k\}} \left[ \mathcal{L}_{\text{train}} (\bm{\theta}_{\text{train},\mathcal{B}}^*) \right]
    = c_{\mathcal{B}} +\frac{a}{2} \left(1-\frac{1}{K}\right) \sigma_{\mathcal{B}}^2.
\end{equation}

For an independent downstream task
$\mathcal{T}\sim\mathcal{P}$, we have
\begin{equation}
    \mathbb{E}_{\{\mathcal{L}_k\},\mathcal{T}} \left[ \mathcal{L}_{\mathcal{T}} (\bm{\theta}_{\text{train},\mathcal{B}}^*) \right]
    = c_{\mathcal{B}} +\frac{a}{2} \mathbb{E}_{\{\mathcal{L}_k\},\mathcal{T}} \left[ \|\bm{\theta}_{\text{train},\mathcal{B}}^* -\bm{\theta}_{\mathcal{T}}^*\|_2^2 \right]
    = c_{\mathcal{B}} +\frac{a}{2} \left(1+\frac{1}{K}\right) \sigma_{\mathcal{B}}^2.
\end{equation}

Substituting \cref{eq:quadratic_expected_train} yields
\begin{equation}
    \mathbb{E}_{\{\mathcal{L}_k\},\mathcal{T}} \left[ \mathcal{L}_{\mathcal{T}} (\bm{\theta}_{\text{train},\mathcal{B}}^*) \right]
    = C_{\text{train}} +\frac{a}{2} \left[ \left(1+\frac{1}{K}\right) - \left(1-\frac{1}{K}\right) \right]\sigma_{\mathcal{B}}^2
    = C_{\text{train}} +\frac{a}{K}\sigma_{\mathcal{B}}^2.
\end{equation}

Thus, for basins achieving the same expected pretraining loss, the expected downstream loss increases linearly with the variance of task-specific minimizers.
\end{proof}

\subsubsection{Proof for \cref{theorem:closeness_generalization:general}}
\label{appendix:proof:theorem:closeness_generalization:general}

We show that the expected generalization gap is controlled by the variance of task-specific minimizers.

\paragraph{Step 1: Closeness Bounds Gradient Variance.}

Since
$\nabla\mathcal{L}(\bm{\theta}_{\mathcal{L}}^*)=\bm{0}$,
the fundamental theorem of calculus gives
\begin{equation}
    \nabla\mathcal{L}(\bm{\theta}^*)  = \mathbf{H}_{\mathcal{L}} (\bm{\theta}^*-\bm{\theta}_{\mathcal{L}}^*),\qquad \mathbf{H}_{\mathcal{L}}  := \int_0^1 \nabla^2\mathcal{L} \left( \bm{\theta}_{\mathcal{L}}^* +s(\bm{\theta}^*-\bm{\theta}_{\mathcal{L}}^*) \right)ds.
\end{equation}

By population stationarity,
$\mathbb{E}_{\mathcal{L}}
[\nabla\mathcal{L}(\bm{\theta}^*)]=\bm{0}$.
Therefore,
\begin{equation}
    \mathbb{E}_{\mathcal{L}} [\mathbf{H}_{\mathcal{L}} (\bm{\theta}^*-\bm{\mu})] = \mathbb{E}_{\mathcal{L}} [\mathbf{H}_{\mathcal{L}} (\bm{\theta}_{\mathcal{L}}^*-\bm{\mu})].
\end{equation}

The curvature bounds hold on the span containing both
$\bm{\theta}^*-\bm{\mu}$ and
$\bm{\theta}_{\mathcal{L}}^*-\bm{\mu}$.
Using
$\mathbb{E}[\bm{\theta}_{\mathcal{L}}^*-\bm{\mu}]
=\bm{0}$,
we obtain
\begin{equation}
    \begin{aligned}
        \lambda_{\min} \|\bm{\theta}^*-\bm{\mu}\|_2^2 &\leq (\bm{\theta}^*-\bm{\mu})^\top \mathbb{E}_{\mathcal{L}} [\mathbf{H}_{\mathcal{L}} (\bm{\theta}_{\mathcal{L}}^*-\bm{\mu})] \\
        &= (\bm{\theta}^*-\bm{\mu})^\top \mathbb{E}_{\mathcal{L}} \left[ \left( \mathbf{H}_{\mathcal{L}} -\frac{\lambda_{\max}+\lambda_{\min}}{2} \mathbf{I} \right) (\bm{\theta}_{\mathcal{L}}^*-\bm{\mu}) \right] \\
        &\leq \frac{\lambda_{\max}-\lambda_{\min}}{2} \|\bm{\theta}^*-\bm{\mu}\|_2 \mathbb{E}_{\mathcal{L}} [\|\bm{\theta}_{\mathcal{L}}^*-\bm{\mu}\|_2] \\
        &\leq \frac{\lambda_{\max}-\lambda_{\min}}{2} \|\bm{\theta}^*-\bm{\mu}\|_2\,\sigma.
    \end{aligned}
\end{equation}
Here the penultimate inequality uses the spectral bounds of the symmetric Hessian compressed to the displacement span. Hence,
\begin{equation}
\label{eq:gen_population_mean}
    \|\bm{\theta}^*-\bm{\mu}\|_2^2 \leq \frac{(\kappa-1)^2}{4}\sigma^2.
\end{equation}

Using the integral representation and the Hessian spectral norm bound,
\begin{equation}
\label{eq:gen_gradient_variance}
    \mathbb{E}_{\mathcal{L}} [\|\nabla\mathcal{L}(\bm{\theta}^*)\|_2^2]
    \leq \lambda_{\max}^2 \mathbb{E}_{\mathcal{L}} [\|\bm{\theta}^*-\bm{\theta}_{\mathcal{L}}^*\|_2^2]
    = \lambda_{\max}^2 \left( \sigma^2+\|\bm{\theta}^*-\bm{\mu}\|_2^2 \right)
    \leq \lambda_{\max}^2 \left(1+\frac{(\kappa-1)^2}{4}\right)\sigma^2.
\end{equation}

\paragraph{Step 2: Bounding the Estimation Error.}

By stationarity of the empirical minimizer and directional strong convexity along
$[\bm{\theta}^*,\bm{\theta}_{\mathrm{train}}^*]$,
\begin{equation}
    \lambda_{\min} \|\bm{\theta}_{\mathrm{train}}^*-\bm{\theta}^*\|_2^2
    \leq -\left\langle \nabla\mathcal{L}_{\mathrm{train}}(\bm{\theta}^*), \bm{\theta}_{\mathrm{train}}^*-\bm{\theta}^* \right\rangle
    \leq \|\nabla\mathcal{L}_{\mathrm{train}}(\bm{\theta}^*)\|_2 \|\bm{\theta}_{\mathrm{train}}^*-\bm{\theta}^*\|_2.
\end{equation}

\begingroup\raggedright
Because the training tasks are i.i.d. and
$\mathbb{E}_{\mathcal{L}}
[\nabla\mathcal{L}(\bm{\theta}^*)]=\bm{0}$,
their gradient cross-terms vanish \mbox{in expectation. Consequently,}
\par\endgroup
\begin{equation}
\label{eq:gen_step1}
    \mathbb{E}_{\{\mathcal{L}_k\}} [\|\bm{\theta}_{\mathrm{train}}^* -\bm{\theta}^*\|_2^2]
    \leq \frac{1}{\lambda_{\min}^2} \mathbb{E}_{\{\mathcal{L}_k\}} [\|\nabla\mathcal{L}_{\mathrm{train}} (\bm{\theta}^*)\|_2^2]
    = \frac{1}{K\lambda_{\min}^2} \mathbb{E}_{\mathcal{L}} [\|\nabla\mathcal{L}(\bm{\theta}^*)\|_2^2].
\end{equation}

\paragraph{Step 3: Bounding the Generalization Gap.}

By independence and identical distribution of the training and downstream tasks,
\begin{equation}
    \mathbb{E}_{\{\mathcal{L}_k\},\mathcal{T}} [\mathcal{L}_{\mathcal{T}} (\bm{\theta}_{\mathrm{train}}^*)] -C_{\mathrm{train}}
    = \mathbb{E}_{\{\mathcal{L}_k\},\mathcal{T}} \left[ \mathcal{L}_{\mathcal{T}} (\bm{\theta}_{\mathrm{train}}^*) -\mathcal{L}_{\mathcal{T}}(\bm{\theta}^*) \right]
    + \mathbb{E}_{\{\mathcal{L}_k\}} \left[ \mathcal{L}_{\mathrm{train}}(\bm{\theta}^*) -\mathcal{L}_{\mathrm{train}} (\bm{\theta}_{\mathrm{train}}^*) \right].
\end{equation}

For the first term, local smoothness and population stationarity imply
\begin{equation}
    \mathbb{E}_{\{\mathcal{L}_k\},\mathcal{T}} \left[ \mathcal{L}_{\mathcal{T}} (\bm{\theta}_{\mathrm{train}}^*) -\mathcal{L}_{\mathcal{T}}(\bm{\theta}^*) \right]
    \leq \frac{\lambda_{\max}}{2} \mathbb{E}_{\{\mathcal{L}_k\}} [\|\bm{\theta}_{\mathrm{train}}^* -\bm{\theta}^*\|_2^2].
\end{equation}

For the second term, directional strong convexity gives
\begin{equation}
    \mathcal{L}_{\mathrm{train}}(\bm{\theta}^*) -\mathcal{L}_{\mathrm{train}} (\bm{\theta}_{\mathrm{train}}^*)
    \leq -\left\langle \nabla\mathcal{L}_{\mathrm{train}}(\bm{\theta}^*), \bm{\theta}_{\mathrm{train}}^*-\bm{\theta}^* \right\rangle -\frac{\lambda_{\min}}{2} \|\bm{\theta}_{\mathrm{train}}^*-\bm{\theta}^*\|_2^2
    \leq \frac{1}{2\lambda_{\min}} \|\nabla\mathcal{L}_{\mathrm{train}} (\bm{\theta}^*)\|_2^2,
\end{equation}
where the last inequality follows by completing the square.

Combining these bounds with
\cref{eq:gen_step1,eq:gen_gradient_variance}
yields
\begin{equation}
    \begin{aligned}
        \mathbb{E}_{\{\mathcal{L}_k\},\mathcal{T}} [\mathcal{L}_{\mathcal{T}} (\bm{\theta}_{\mathrm{train}}^*)] -C_{\mathrm{train}} &\leq \left( \frac{\lambda_{\max}}{2\lambda_{\min}^2} +\frac{1}{2\lambda_{\min}} \right) \frac{1}{K} \mathbb{E}_{\mathcal{L}} [\|\nabla\mathcal{L}(\bm{\theta}^*)\|_2^2] \\
        &\leq \frac{\lambda_{\max}\kappa(\kappa+1)}{2K} \left(1+\frac{(\kappa-1)^2}{4}\right)\sigma^2.
    \end{aligned}
\end{equation}

Thus, for fixed local curvature parameters, the expected generalization gap is bounded by
$O(\sigma^2/K)$.

\newpage
\subsection{Proof of \cref{thm:gradient_similarity_bound_closeness}}
\label{sec:proof_alignment_bound}

\begin{proof}
We prove the bound in three steps: relating parameter
closeness to gradient norms, applying the stationarity
condition, and bounding the resulting inner products
using cosine similarity.

\refstepcounter{nexusclosenessproofstep}\label{proofstep:nexusclosenessproofstep:gradient_norm}
\paragraph{Step \thenexusclosenessproofstep: Relating Closeness to Gradient Norm.}
Since $\bm{\theta}_k^*$ is a local minimizer,
$\nabla\mathcal{L}_k(\bm{\theta}_k^*)=\mathbf{0}$.
By the fundamental theorem of calculus,
\begin{equation}
    \begin{aligned}
        \bm{u}_k^\top\nabla\mathcal{L}_k(\bm{\theta}) &= \bm{u}_k^\top \left( \nabla\mathcal{L}_k(\bm{\theta}) -\nabla\mathcal{L}_k(\bm{\theta}_k^*) \right) \\
        &= \|\bm{\theta}-\bm{\theta}_k^*\|_2 \int_0^1 \bm{u}_k^\top \nabla^2\mathcal{L}_k \left( \bm{\theta}_k^* +s(\bm{\theta}-\bm{\theta}_k^*) \right) \bm{u}_k\,ds \\
        &\geq \lambda_{\min} \|\bm{\theta}-\bm{\theta}_k^*\|_2,
    \end{aligned}
\end{equation}
where the last inequality follows from the directional
curvature lower bound in the theorem.

By Cauchy--Schwarz,
\begin{equation}
    \|\bm{\theta}-\bm{\theta}_k^*\|_2 \leq \frac{1}{\lambda_{\min}} \|\nabla\mathcal{L}_k(\bm{\theta})\|_2.
\end{equation}
Squaring and averaging over all $K$ tasks yields
\begin{equation}
\label{eq:dist_grad_bound}
    \frac{1}{K}\sum_{k=1}^K \|\bm{\theta}-\bm{\theta}_k^*\|_2^2 \leq \frac{1}{K\lambda_{\min}^2} \sum_{k=1}^K \|\nabla\mathcal{L}_k(\bm{\theta})\|_2^2.
\end{equation}

\refstepcounter{nexusclosenessproofstep}\label{proofstep:nexusclosenessproofstep:balance}
\paragraph{Step \thenexusclosenessproofstep: Force Balance Decomposition.}
Since $\bm{\theta}$ satisfies the stationarity condition
$\sum_{k=1}^K\nabla\mathcal{L}_k(\bm{\theta}) =\mathbf{0}$, we have
\begin{equation}
    0  = \left\| \sum_{k=1}^K\nabla\mathcal{L}_k(\bm{\theta}) \right\|_2^2
    = \sum_{k=1}^K \|\nabla\mathcal{L}_k(\bm{\theta})\|_2^2 + \sum_{i\neq j} \nabla\mathcal{L}_i(\bm{\theta})^\top \nabla\mathcal{L}_j(\bm{\theta}).
\end{equation}
Therefore,
\begin{equation}
\label{eq:force_balance_identity}
    \sum_{k=1}^K \|\nabla\mathcal{L}_k(\bm{\theta})\|_2^2 = -\sum_{i\neq j} \nabla\mathcal{L}_i(\bm{\theta})^\top \nabla\mathcal{L}_j(\bm{\theta}).
\end{equation}
Combining this identity with
\cref{eq:dist_grad_bound}
gives the first inequality of the theorem.

\refstepcounter{nexusclosenessproofstep}\label{proofstep:nexusclosenessproofstep:cosine}
\paragraph{Step \thenexusclosenessproofstep: Bounding via Cosine Similarity.}
Let
$\text{CosSim}_{ij} =\text{CosSim}( \nabla\mathcal{L}_i(\bm{\theta}), \nabla\mathcal{L}_j(\bm{\theta}))$.
For any $i\neq j$, we have
\begin{equation}
    -\nabla\mathcal{L}_i(\bm{\theta})^\top \nabla\mathcal{L}_j(\bm{\theta})
    = \|\nabla\mathcal{L}_i(\bm{\theta})\|_2 \|\nabla\mathcal{L}_j(\bm{\theta})\|_2 (1-\text{CosSim}_{ij})
    - \|\nabla\mathcal{L}_i(\bm{\theta})\|_2 \|\nabla\mathcal{L}_j(\bm{\theta})\|_2
    \leq G^2(1-\text{CosSim}_{ij}),
\end{equation}
where we use
$1-\text{CosSim}_{ij}\geq 0$
and
$\|\nabla\mathcal{L}_k(\bm{\theta})\|_2\leq G$.

Summing over all ordered pairs $i\neq j$ yields
\begin{equation}
    \sum_{i\neq j} \left( -\nabla\mathcal{L}_i(\bm{\theta})^\top \nabla\mathcal{L}_j(\bm{\theta}) \right)
    \leq G^2\sum_{i\neq j} \left( 1-\text{CosSim}( \nabla\mathcal{L}_i(\bm{\theta}), \nabla\mathcal{L}_j(\bm{\theta})) \right).
\end{equation}
Combining this with \cref{proofstep:nexusclosenessproofstep:balance} completes the proof.

\end{proof}

\subsection{Implicit Bias of Nexus Optimizer}
\label{sec:appendix_cwa_theory}

In this appendix, we provide the detailed proofs for \cref{thm:cwa_optimize_gradient_similarity}. We rigorously analyze the update dynamics of \cref{alg:cwa_standard} using second-order Taylor expansions and derive the precise form of the implicit optimization objective with explicit non-asymptotic error bounds.

\subsubsection{Preliminaries and Notation}

Let $\mathcal{L}_i: \mathbb{R}^d \to \mathbb{R}$ denote the loss function for the $i$-th task, where $i \in \{1, \dots, k\}$. We denote the gradient and Hessian at parameters $\bm{\theta}$ as $\nabla \mathcal{L}_i(\bm{\theta})$ and $\nabla^2 \mathcal{L}_i(\bm{\theta})$, respectively.
The cosine similarity between the gradients of task $i$ and task $j$ is defined as:
\begin{equation}
    S_{ij}(\bm{\theta}) \triangleq \text{CosSim}(\nabla \mathcal{L}_i(\bm{\theta}), \nabla \mathcal{L}_j(\bm{\theta})) = \frac{\nabla \mathcal{L}_i(\bm{\theta})^\top \nabla \mathcal{L}_j(\bm{\theta})}{\|\nabla \mathcal{L}_i(\bm{\theta})\|_2 \|\nabla \mathcal{L}_j(\bm{\theta})\|_2}.
\end{equation}

\cref{alg:cwa_standard} performs $k$ inner updates in each outer iteration $t$. Let $\bm{\theta}_{t, 0}$ be the parameters at the start of the inner loop (i.e., $\bm{\theta}_{t, 0} = \bm{\theta}_{t-1}$).
At each inner step $m \in \{1, \dots, k\}$, a task index $s_m$ is sampled uniformly from $\{1, \dots, k\}$. The update rule is:
\begin{equation}
\label{eq:update_rule_app}
    \bm{\theta}_{t, m} = \bm{\theta}_{t, m-1} - \gamma \frac{\nabla \mathcal{L}_{s_m}(\bm{\theta}_{t, m-1})}{\|\nabla \mathcal{L}_{s_m}(\bm{\theta}_{t, m-1})\|_2}.
\end{equation}
The Nexus pseudo-gradient passed to the outer optimizer is $\hat{\bm{g}}_t = \bm{\theta}_{t, 0} - \bm{\theta}_{t, k} = \sum_{m=1}^k (\bm{\theta}_{t, m-1} - \bm{\theta}_{t, m})$.

\subsubsection{Assumptions and Derived Constants}
\label{subsec:assumptions}

To derive explicit non-asymptotic bounds, we utilize the following standard assumptions regarding the loss landscape.

\begin{itemize}
    \item \refstepcounter{nexusproofassumption}\label{assump:nexus_proof:bounded_gradients}\textbf{Assumption~\thenexusproofassumption{} (Bounded Gradients):} For all tasks $i$ and parameters $\bm{\theta}$, the gradient norm is bounded from below: $0 < G_{\min} \le \|\nabla \mathcal{L}_i(\bm{\theta})\|_2$.
    \item \refstepcounter{nexusproofassumption}\label{assump:nexus_proof:smoothness}\textbf{Assumption~\thenexusproofassumption{} (Smoothness):} The loss $\mathcal{L}_i$ is $L$-smooth, i.e., $\|\nabla^2 \mathcal{L}_i(\bm{\theta})\|_2 \le L$.
    \item \refstepcounter{nexusproofassumption}\label{assump:nexus_proof:hessian_lipschitz}\textbf{Assumption~\thenexusproofassumption{} (Hessian Lipschitz):} The Hessian is $\rho$-Lipschitz continuous, i.e., $\|\nabla^2 \mathcal{L}_i(\bm{x}) - \nabla^2 \mathcal{L}_i(\bm{y})\|_2 \le \rho \|\bm{x} - \bm{y}\|_2$.
\end{itemize}

Based on the properties above, we further denote $L_1$ and $L_2$ as the Lipschitz constants for the normalized gradient and its Jacobian, respectively:
\begin{enumerate}
    \item The normalized gradient is $L_1$-Lipschitz continuous:
    \begin{equation}
        \left\| \frac{\nabla \mathcal{L}_i(\bm{x})}{\|\nabla \mathcal{L}_i(\bm{x})\|_2} - \frac{\nabla \mathcal{L}_i(\bm{y})}{\|\nabla \mathcal{L}_i(\bm{y})\|_2} \right\|_2 \le L_1 \|\bm{x} - \bm{y}\|_2.
    \end{equation}
    \item The Jacobian of the normalized gradient is $L_2$-Lipschitz continuous:
    \begin{equation}
        \left\| \bm{J}_i(\bm{x}) - \bm{J}_i(\bm{y}) \right\|_2 \le L_2 \|\bm{x} - \bm{y}\|_2,
    \end{equation}
    where $\bm{J}_i(\bm{\theta}) = \frac{\partial}{\partial \bm{\theta}} \left( \frac{\nabla \mathcal{L}_i(\bm{\theta})}{\|\nabla \mathcal{L}_i(\bm{\theta})\|_2} \right)$.
\end{enumerate}

\paragraph{Derivation of Constants.} 
Here, we provide the detailed derivation of $L_1$ and $L_2$ based on \cref{assump:nexus_proof:bounded_gradients,assump:nexus_proof:smoothness,assump:nexus_proof:hessian_lipschitz}.

\textbf{1. Derivation of $L_1$:}
By the Mean Value Theorem, $L_1$ is bounded by the supremum of the spectral norm of the Jacobian $\bm{J}_i(\bm{\theta})$. The Jacobian is explicitly given by:
\begin{equation}
    \bm{J}_i(\bm{\theta}) = \frac{1}{\|\nabla \mathcal{L}_i\|_2} \left( \bm{I} - \frac{\nabla \mathcal{L}_i \nabla \mathcal{L}_i^\top}{\|\nabla \mathcal{L}_i\|_2^2} \right) \nabla^2 \mathcal{L}_i(\bm{\theta}).
\end{equation}
The middle term is an orthogonal projection matrix with spectral norm 1. \looseness=-1 Using the bounds from \mbox{\cref{assump:nexus_proof:bounded_gradients,assump:nexus_proof:smoothness}}:
\begin{equation}
    L_1 \le \sup_{\bm{\theta}} \|\bm{J}_i(\bm{\theta})\|_2 \le \frac{1}{G_{\min}} \cdot 1 \cdot L = \frac{L}{G_{\min}}.
\end{equation}

\textbf{2. Derivation of $L_2$:}
We decompose the Jacobian $\bm{J}_i(\bm{\theta})$ into three components: a scalar term $u(\bm{\theta})$, a projection term $\bm{\Pi}(\bm{\theta})$, and the Hessian $\bm{H}_i(\bm{\theta})$:
\begin{equation}
    \bm{J}_i(\bm{\theta}) = \underbrace{\|\nabla \mathcal{L}_i(\bm{\theta})\|_2^{-1}}_{u(\bm{\theta})} \cdot \underbrace{\left( \bm{I} - \frac{\nabla \mathcal{L}_i \nabla \mathcal{L}_i^\top}{\|\nabla \mathcal{L}_i\|_2^2} \right)}_{\bm{\Pi}(\bm{\theta})} \cdot \underbrace{\nabla^2 \mathcal{L}_i(\bm{\theta})}_{\bm{H}_i(\bm{\theta})}.
\end{equation}
We apply the product Lipschitz rule. For a product of three functions $f = abc$, the Lipschitz constant satisfies $L_f \le L_a M_b M_c + M_a L_b M_c + M_a M_b L_c$, where $M_{(\cdot)}$ denotes the upper bound of the magnitude and $L_{(\cdot)}$ denotes the Lipschitz constant.

\begin{itemize}
    \item \textbf{Part 1: Scalar $u(\bm{\theta}) = \|\nabla \mathcal{L}_i\|_2^{-1}$.} \\
    \textbf{Magnitude ($M_u$):} By \cref{assump:nexus_proof:bounded_gradients}, $|u| \le \frac{1}{G_{\min}}$. \\
    \textbf{Lipschitz ($L_u$):} The gradient of $u$ is $\nabla u = -\|\nabla \mathcal{L}_i\|_2^{-2} \frac{\nabla^2 \mathcal{L}_i \nabla \mathcal{L}_i}{\|\nabla \mathcal{L}_i\|_2} = -\frac{\bm{H}_i \nabla \mathcal{L}_i}{\|\nabla \mathcal{L}_i\|_2^3}$.
    Taking the norm, we have $\|\nabla u\|_2 \le \frac{\|\bm{H}_i\|_2 \|\nabla \mathcal{L}_i\|_2}{\|\nabla \mathcal{L}_i\|_2^3} = \frac{\|\bm{H}_i\|_2}{\|\nabla \mathcal{L}_i\|_2^2}$.
    Using the bounds $L$ and $G_{\min}$, we get $L_u = \frac{L}{G_{\min}^2}$.
    
    \item \textbf{Part 2: Projection $\bm{\Pi}(\bm{\theta}) = \bm{I} - \bm{h}_i\bm{h}_i^\top$.} \\
    \textbf{Magnitude ($M_\Pi$):} The spectral norm is $\|\bm{\Pi}\|_2 = 1$. \\
    \textbf{Lipschitz ($L_\Pi$):} $\bm{\Pi}$ depends on the normalized gradient $\bm{h}_i$, which is $L_1$-Lipschitz. For any unit vectors $\bm{x}, \bm{y}$, we have $\|\bm{x}\bm{x}^\top - \bm{y}\bm{y}^\top\|_2 \le \|\bm{x}(\bm{x}-\bm{y})^\top\|_2 + \|(\bm{x}-\bm{y})\bm{y}^\top\|_2 = 2\|\bm{x}-\bm{y}\|_2$.
    By the chain rule, $L_\Pi = 2 L_1 = \frac{2L}{G_{\min}}$.
    
    \item \textbf{Part 3: Hessian $\bm{H}_i(\bm{\theta}) = \nabla^2 \mathcal{L}_i$.} \\
    \textbf{Magnitude ($M_H$):} By \cref{assump:nexus_proof:smoothness}, $\|\bm{H}_i\|_2 \le L$. \\
    \textbf{Lipschitz ($L_H$):} By \cref{assump:nexus_proof:hessian_lipschitz}, $L_H = \rho$.
\end{itemize}

Substituting these values into the product rule formula:
\begin{equation}
    \begin{aligned}
        L_2 &\le L_u M_\Pi M_H + M_u L_\Pi M_H + M_u M_\Pi L_H \\
        &\le \left( \frac{L}{G_{\min}^2} \cdot 1 \cdot L \right) + \left( \frac{1}{G_{\min}} \cdot \frac{2L}{G_{\min}} \cdot L \right) + \left( \frac{1}{G_{\min}} \cdot 1 \cdot \rho \right) \\
        &= \frac{L^2}{G_{\min}^2} + \frac{2L^2}{G_{\min}^2} + \frac{\rho}{G_{\min}}.
    \end{aligned}
\end{equation}
Combining terms yields the final constant:
\begin{equation}
    L_2 = \frac{3L^2 + \rho G_{\min}}{G_{\min}^2}.
\end{equation}

\subsubsection{Derivation of the Update Direction}

We now derive the expansion of the total pseudo-gradient $\hat{\bm{g}}_t$ and bound the error terms.

\refstepcounter{nexusupdateproofstep}\label{proofstep:nexusupdateproofstep:expansion}
\paragraph{Step \thenexusupdateproofstep: Expansion of the Normalized Gradient}
We aim to expand the normalized gradient at the shifted parameters $\bm{\theta}_{t, m-1}$ around the initial point $\bm{\theta}_{t, 0}$.
Let $\Delta \bm{\theta}_{m-1} = \bm{\theta}_{t, m-1} - \bm{\theta}_{t, 0}$.

The Jacobian of the normalized gradient is given explicitly by the projection of the Hessian:
\begin{equation}
    \bm{J}_{s_m}(\bm{\theta}) = \frac{1}{\|\nabla \mathcal{L}_{s_m}(\bm{\theta})\|_2} \left( \bm{I} - \frac{\nabla \mathcal{L}_{s_m}(\bm{\theta})\nabla \mathcal{L}_{s_m}(\bm{\theta})^\top}{\|\nabla \mathcal{L}_{s_m}(\bm{\theta})\|_2^2} \right) \nabla^2 \mathcal{L}_{s_m}(\bm{\theta}).
\end{equation}
Applying Taylor's theorem with the Lagrange remainder form:
\begin{equation}
\label{eq:taylor_expansion}
    \frac{\nabla \mathcal{L}_{s_m}(\bm{\theta}_{t, m-1})}{\|\nabla \mathcal{L}_{s_m}(\bm{\theta}_{t, m-1})\|_2} = \frac{\nabla \mathcal{L}_{s_m}(\bm{\theta}_{t, 0})}{\|\nabla \mathcal{L}_{s_m}(\bm{\theta}_{t, 0})\|_2} + \bm{J}_{s_m}(\bm{\theta}_{t, 0}) \Delta \bm{\theta}_{m-1} + \bm{r}_{m}.
\end{equation}
Using the $L_2$-Lipschitz property of the Jacobian, the residual vector $\bm{r}_{m}$ is bounded by:
\begin{equation}
    \|\bm{r}_{m}\|_2 \le \frac{L_2}{2} \|\Delta \bm{\theta}_{m-1}\|_2^2.
\end{equation}

\refstepcounter{nexusupdateproofstep}\label{proofstep:nexusupdateproofstep:recursion}
\paragraph{Step \thenexusupdateproofstep: Recursive Substitution}
The displacement $\Delta \bm{\theta}_{m-1}$ is the sum of previous updates. Using the zeroth-order approximation:
\begin{equation}
    \Delta \bm{\theta}_{m-1} = \sum_{l=1}^{m-1} (\bm{\theta}_{t, l} - \bm{\theta}_{t, l-1}) = -\gamma \sum_{l=1}^{m-1} \frac{\nabla \mathcal{L}_{s_l}(\bm{\theta}_{t, l-1})}{\|\nabla \mathcal{L}_{s_l}(\bm{\theta}_{t, l-1})\|_2}.
\end{equation}
We approximate the terms in the sum using the zeroth-order expansion around $\bm{\theta}_{t, 0}$. Using the $L_1$-Lipschitz property of the normalized gradient:
\begin{equation}
    \left\| \frac{\nabla \mathcal{L}_{s_l}(\bm{\theta}_{t, l-1})}{\|\nabla \mathcal{L}_{s_l}(\bm{\theta}_{t, l-1})\|_2} - \frac{\nabla \mathcal{L}_{s_l}(\bm{\theta}_{t, 0})}{\|\nabla \mathcal{L}_{s_l}(\bm{\theta}_{t, 0})\|_2} \right\|_2 
    \le L_1 \|\bm{\theta}_{t, l-1} - \bm{\theta}_{t, 0}\|_2 
    = L_1 \left\|-\sum_{j=1}^{l-1}\gamma \frac{\nabla \mathcal{L}_{s_l}(\bm{\theta}_{t, j})}{\|\nabla \mathcal{L}_{s_l}(\bm{\theta}_{t, j})\|_2} \right\| 
    \le L_1 (l-1)\gamma.
\end{equation}
Thus, we can write:
\begin{equation}
    \Delta \bm{\theta}_{m-1} = -\gamma \sum_{l=1}^{m-1} \frac{\nabla \mathcal{L}_{s_l}(\bm{\theta}_{t, 0})}{\|\nabla \mathcal{L}_{s_l}(\bm{\theta}_{t, 0})\|_2} + \bm{\delta}_{m-1},
\end{equation}
where the accumulated error $\bm{\delta}_{m-1}$ is bounded by summing the individual errors:
\begin{equation}
    \|\bm{\delta}_{m-1}\|_2 \le \gamma \sum_{l=1}^{m-1} L_1 (l-1)\gamma = L_1 \gamma^2 \frac{(m-1)(m-2)}{2} \le \frac{L_1}{2} (m-1)^2 \gamma^2.
\end{equation}

Substituting this expression for $\Delta \bm{\theta}_{m-1}$ back into \cref{eq:taylor_expansion}:
\begin{equation}
\begin{aligned}
    \frac{\nabla \mathcal{L}_{s_m}(\bm{\theta}_{t, m-1})}{\|\nabla \mathcal{L}_{s_m}(\bm{\theta}_{t, m-1})\|_2} &= \frac{\nabla \mathcal{L}_{s_m}(\bm{\theta}_{t, 0})}{\|\nabla \mathcal{L}_{s_m}(\bm{\theta}_{t, 0})\|_2} - \gamma \sum_{l=1}^{m-1} \bm{J}_{s_m}(\bm{\theta}_{t, 0}) \frac{\nabla \mathcal{L}_{s_l}(\bm{\theta}_{t, 0})}{\|\nabla \mathcal{L}_{s_l}(\bm{\theta}_{t, 0})\|_2} + \bm{\mathcal{E}}_{m}.
\end{aligned}
\end{equation}
Here, the total error at step $m$, denoted $\bm{\mathcal{E}}_{m}$, consists of the Taylor residual $\bm{r}_m$ and the propagation error from $\bm{\delta}_{m-1}$ scaled by the Jacobian.
Using $\|\bm{J}_{s_m}\|_2 \le L_1$ and $\|\Delta \bm{\theta}_{m-1}\|_2 \le (m-1)\gamma$:
\begin{equation}
    \|\bm{\mathcal{E}}_{m}\|_2 \le \frac{L_2}{2} (m-1)^2 \gamma^2 + L_1 \left( \frac{L_1}{2} (m-1)^2 \gamma^2 \right) = \frac{L_2 + L_1^2}{2} (m-1)^2 \gamma^2.
\end{equation}

\refstepcounter{nexusupdateproofstep}\label{proofstep:nexusupdateproofstep:aggregation}
\paragraph{Step \thenexusupdateproofstep: Aggregation of the Pseudo-Gradient}
The total pseudo-gradient is $\hat{\bm{g}}_t = \gamma \sum_{m=1}^k \frac{\nabla \mathcal{L}_{s_m}(\bm{\theta}_{t, m-1})}{\|\nabla \mathcal{L}_{s_m}(\bm{\theta}_{t, m-1})\|_2}$. Substituting the result from \cref{proofstep:nexusupdateproofstep:recursion}:
\begin{equation}
    \hat{\bm{g}}_t = \gamma \sum_{m=1}^k \frac{\nabla \mathcal{L}_{s_m}}{\|\nabla \mathcal{L}_{s_m}\|_2} - \gamma^2 \sum_{m=1}^k \sum_{l=1}^{m-1} \bm{J}_{s_m} \frac{\nabla \mathcal{L}_{s_l}}{\|\nabla \mathcal{L}_{s_l}\|_2} + \bm{\mathcal{E}}_{total}.
\end{equation}
(We omit the argument $\bm{\theta}_{t, 0}$ for brevity; all terms are evaluated at $\bm{\theta}_{t, 0}$).
The total error vector $\bm{\mathcal{E}}_{total} = \gamma \sum_{m=1}^k \bm{\mathcal{E}}_{m}$ is explicitly bounded by summing the bounds from \cref{proofstep:nexusupdateproofstep:recursion}:
\begin{equation}
    \|\bm{\mathcal{E}}_{total}\|_2 \le \gamma \sum_{m=1}^k \frac{L_2 + L_1^2}{2} (m-1)^2 \gamma^2 \le \frac{L_2 + L_1^2}{6} k^3 \gamma^3.
\end{equation}
\refstepcounter{nexusupdateproofstep}\label{proofstep:nexusupdateproofstep:expectation}
\paragraph{Step \thenexusupdateproofstep: Expectation Analysis and Connection to Cosine Similarity}
We now compute the expectation of $\hat{\bm{g}}_t$ over the independent uniform sampling of indices $s_1, \dots, s_k$ and relate the second-order term to the gradient of the cosine similarity.

\subparagraph{Linear Term.}
Let $\mathcal{T}_{linear} = \sum_{m=1}^k \frac{\nabla \mathcal{L}_{s_m}}{\|\nabla \mathcal{L}_{s_m}\|_2}$. By linearity of expectation:
\begin{equation}
    \mathbb{E}[\mathcal{T}_{linear}] = k \cdot \frac{1}{k} \sum_{i=1}^k \frac{\nabla \mathcal{L}_i}{\|\nabla \mathcal{L}_i\|_2} = \sum_{i=1}^k \frac{\nabla \mathcal{L}_i}{\|\nabla \mathcal{L}_i\|_2}.
\end{equation}

\subparagraph{Interaction Term.}
Let $\mathcal{T}_{interact} = \sum_{m=1}^k \sum_{l=1}^{m-1} \bm{J}_{s_m} \frac{\nabla \mathcal{L}_{s_l}}{\|\nabla \mathcal{L}_{s_l}\|_2}$.
The double summation contains $\frac{k(k-1)}{2}$ terms. Since $m > l$, $s_m$ and $s_l$ are independent. Summing over all pairs yields:
\begin{equation}
    \mathbb{E}[\mathcal{T}_{interact}] = \frac{k-1}{2k} \sum_{i, j} \bm{J}_i \frac{\nabla \mathcal{L}_j}{\|\nabla \mathcal{L}_j\|_2} = \frac{k-1}{4k} \sum_{i \neq j} \left( \bm{J}_i \frac{\nabla \mathcal{L}_j}{\|\nabla \mathcal{L}_j\|_2} + \bm{J}_j \frac{\nabla \mathcal{L}_i}{\|\nabla \mathcal{L}_i\|_2} \right).
\end{equation}

Note that the exact gradient of the cosine similarity between task $i$ and task $j$ is formulated using the transposed Jacobian (VJP):
\begin{equation}
    \nabla_{\bm{\theta}} \text{CosSim}(\nabla \mathcal{L}_i, \nabla \mathcal{L}_j) = \bm{J}_i^\top \frac{\nabla \mathcal{L}_j}{\|\nabla \mathcal{L}_j\|_2} + \bm{J}_j^\top \frac{\nabla \mathcal{L}_i}{\|\nabla \mathcal{L}_i\|_2}.
\end{equation}

To relate the interaction term to the cosine similarity gradient, we add and subtract the transposed Jacobian terms:
\begin{equation}
\begin{aligned}
    \bm{J}_i \frac{\nabla \mathcal{L}_j}{\|\nabla \mathcal{L}_j\|_2} + \bm{J}_j \frac{\nabla \mathcal{L}_i}{\|\nabla \mathcal{L}_i\|_2} 
    &= \left( \bm{J}_i^\top \frac{\nabla \mathcal{L}_j}{\|\nabla \mathcal{L}_j\|_2} + \bm{J}_j^\top \frac{\nabla \mathcal{L}_i}{\|\nabla \mathcal{L}_i\|_2} \right) + (\bm{J}_i - \bm{J}_i^\top) \frac{\nabla \mathcal{L}_j}{\|\nabla \mathcal{L}_j\|_2} + (\bm{J}_j - \bm{J}_j^\top) \frac{\nabla \mathcal{L}_i}{\|\nabla \mathcal{L}_i\|_2} \\
    &= \nabla_{\bm{\theta}} \text{CosSim}(\nabla \mathcal{L}_i, \nabla \mathcal{L}_j) + \bm{\mathcal{E}}_{\text{sym}, i, j} + \bm{\mathcal{E}}_{\text{sym}, j, i},
\end{aligned}
\end{equation}
where $\bm{\mathcal{E}}_{\text{sym}, i, j}$ is the Jacobian non-symmetric error defined as:
\begin{equation}
    \bm{\mathcal{E}}_{\text{sym}, i, j} = \left( \nabla \frac{\nabla \mathcal{L}_i}{\|\nabla \mathcal{L}_i\|_2} - \left( \nabla \frac{\nabla \mathcal{L}_i}{\|\nabla \mathcal{L}_i\|_2} \right)^\top \right) \frac{\nabla \mathcal{L}_j}{\|\nabla \mathcal{L}_j\|_2}.
\end{equation}

The residual $\bm{\mathcal{E}}_{\text{sym}, i, j}$ exhibits two properties:
\begin{enumerate}
    \item \textbf{Orthogonality to the target gradient:} The matrix $\bm{A}_{i} = \nabla (\frac{\nabla \mathcal{L}_i}{\|\nabla \mathcal{L}_i\|_2}) - \nabla (\frac{\nabla \mathcal{L}_i}{\|\nabla \mathcal{L}_i\|_2})^\top$ is anti-symmetric ($\bm{A}_{i}^\top = -\bm{A}_{i}$). For any anti-symmetric matrix $\bm{A}$ and vector $\bm{x}$, the quadratic form is zero ($\bm{x}^\top \bm{A} \bm{x} = 0$). By setting $\bm{x} = \frac{\nabla \mathcal{L}_j}{\|\nabla \mathcal{L}_j\|_2}$, we have $\langle \bm{\mathcal{E}}_{\text{sym}, i, j}, \nabla \mathcal{L}_j \rangle = 0$, which implies $\bm{\mathcal{E}}_{\text{sym}, i, j} \perp \nabla \mathcal{L}_j$.
    
    \item \textbf{Vanishing under gradient-Hessian alignment:} Recent literature on the Edge of Stability (EoS) and the "river valley" phenomenon observes that gradients tend to align with the top eigenvectors of the Hessian during training \cite{gur2018gradient,song2024does,cohen2021gradient,damian2022self}. When the gradient aligns with an eigenvector, the commutator between the orthogonal projection and the Hessian becomes zero, which implies $\bm{\mathcal{E}}_{\text{sym}, i, j} = \mathbf{0}$.
\end{enumerate}
Substituting this expansion into the interaction term expectation yields:
\begin{equation}
    \mathbb{E}[\mathcal{T}_{interact}] = \frac{k-1}{4k} \sum_{i \neq j} \left( \nabla_{\bm{\theta}} \text{CosSim}(\nabla \mathcal{L}_i, \nabla \mathcal{L}_j) + \bm{\mathcal{E}}_{\text{sym}, i, j} + \bm{\mathcal{E}}_{\text{sym}, j, i} \right).
\end{equation}

\subsubsection{Proof Conclusion}
Combining the linear term and the interaction term, the expected Nexus update direction is:
\begin{equation}
    \mathbb{E}[\hat{\bm{g}}_t] = \gamma \sum_{i=1}^k \frac{\nabla \mathcal{L}_i(\bm{\theta}_{t, 0})}{\|\nabla \mathcal{L}_i(\bm{\theta}_{t, 0})\|_2} - \gamma^2 \frac{k-1}{4k} \sum_{i \neq j} \left( \nabla_{\bm{\theta}} \text{CosSim}(\nabla \mathcal{L}_i, \nabla \mathcal{L}_j) + \bm{\mathcal{E}}_{\text{sym}, i, j} + \bm{\mathcal{E}}_{\text{sym}, j, i} \right) + \bm{\mathcal{E}}_{total}.
\end{equation}
Substituting the constants derived in \cref{subsec:assumptions}, the residual is bounded by:
\begin{equation}
    \|\bm{\mathcal{E}}_{total}\|_2 \le \frac{1}{6} \left( \frac{\rho G_{\min} + 4L^2}{G_{\min}^2} \right) k^3 \gamma^3 = \mathcal{O}(\gamma^3).
\end{equation}
This confirms that the update direction follows the gradient of the loss plus the similarity alignment term, subject to the non-symmetric residual and a bounded cubic Taylor error.
\hfill $\square$

\clearpage

\section{Additional Theoretical Results}
\label{appendix:discussion}

\subsection{A Gradient-Flow Interpretation of Downstream Generalization}
\label{sec:appendix:gradient_flow}

\textbf{Motivation.} In \cref{eq:cwa:gradient_similarity_implies_b_decrease_when_optimizing_on_a}, we show that gradient similarity determines the first-order decrease in downstream loss after a single gradient descent step on the pretraining objective. A natural question is whether we can extend this intuition to the entire training trajectory: \textbf{if gradient similarity improves by $\delta$, how much additional downstream loss reduction can we expect at the same pretraining loss?}

\textbf{Gradient-flow interpretation.} Consider the continuous-time gradient flow of the pretraining objective:
\begin{equation}
\dot{\theta}(t)=-\eta(t)\nabla L_{\mathrm{train}}(\theta(t)).
\end{equation}

By the chain rule, the changes in downstream and pretraining losses satisfy
\begin{equation}
\begin{aligned}
L_{\mathcal T}(T)-L_{\mathcal T}(0)
&=-\int_0^T \eta(t)\nabla L_{\mathcal T}(t)^\top\nabla L_{\mathrm{train}}(t)\,dt,\\
L_{\mathrm{train}}(T)-L_{\mathrm{train}}(0)
&=-\int_0^T \eta(t)\|\nabla L_{\mathrm{train}}(t)\|_2^2\,dt.
\end{aligned}
\end{equation}

Since the pretraining loss decreases monotonically under gradient flow, we can use $L_{\mathrm{train}}$ instead of time $t$ to parameterize the training trajectory. This gives
\begin{equation}
L_{\mathcal T}(T)-L_{\mathcal T}(0)
=\int_{L_{\mathrm{train}}(0)}^{L_{\mathrm{train}}(T)}
\frac{\nabla L_{\mathcal T}^{\top}\nabla L_{\mathrm{train}}}
{\|\nabla L_{\mathrm{train}}\|_2^2}
\,dL_{\mathrm{train}}.
\end{equation}

This equation provides a quantitative interpretation of gradient alignment: \textbf{the normalized gradient dot product determines how much downstream loss decreases per unit reduction in pretraining loss.}

\textbf{Implication for downstream generalization.} Suppose the normalized gradient dot product improves by a constant $\delta$ throughout training, when compared at the same pretraining loss. Assuming the same initial downstream loss and final pretraining loss, the additional downstream loss reduction is
\begin{equation}
\delta\big(L_{\mathrm{train}}(0)-L_{\mathrm{train}}(T)\big).
\end{equation}

Therefore, even when two training trajectories achieve exactly the same pretraining loss reduction, better gradient alignment can lead to a lower downstream loss.

To connect this result with the cosine similarity optimized by Nexus, note that
\begin{equation}
\frac{\nabla L_{\mathcal T}^{\top}\nabla L_{\mathrm{train}}}
{\|\nabla L_{\mathrm{train}}\|_2^2}
=\operatorname{CosSim}(\nabla L_{\mathrm{train}},\nabla L_{\mathcal T})
\cdot
\frac{\|\nabla L_{\mathcal T}\|_2}
{\|\nabla L_{\mathrm{train}}\|_2}.
\end{equation}

Thus, if cosine similarity improves by $\delta$ while the gradient norm ratio remains approximately unchanged, the additional downstream loss reduction is proportional to $\delta$ and the accumulated pretraining loss reduction, weighted by the gradient norm ratio.

Although practical optimizers such as AdamW and Nexus do not follow exact gradient flow, this analysis offers an intuitive explanation for our central observation: \textbf{models can achieve the same pretraining loss while exhibiting different downstream performance, because their optimization trajectories differ in how effectively pretraining progress transfers to downstream tasks.}

\subsection{Convergence Rate of Nexus}
\label{sec:cwa:convergence_rate}

All of our analyses are based on the assumption that Nexus should not be slower than its base optimizer. This ensures that both can achieve the "same training loss," allowing the implicit bias of Nexus to subsequently achieve "better downstream performance." One might concern that since Nexus optimizes two joint objectives (see \cref{thm:cwa_optimize_gradient_similarity}), it may be slower than its base optimizer. Consequently, the downstream gains might not offset the speed loss, potentially leading to worse overall downstream performance.

Fortunately, this concern does not hold in practice. Empirically, across all experiments, Nexus is not slower, and sometimes even slightly faster, than its base optimizer (see \cref{sec:exp}). Intuitively, Nexus makes the gradients of each $\mathcal{L}_i$ similar; thus, optimizing $\mathcal{L}_i$ effectively optimizes $\mathcal{L}_j$ simultaneously, as analyzed in \cref{sec:cwa:optimizing_using_cwa}. This "constructive interference" can lead to slightly faster convergence. 

We can also adopt the framework of \citet{svrgoptimizer} (assuming each $\mathcal{L}_i$ is $L$-smooth and $\mu$-strongly convex) to obtain further theoretical intuition. In this setting, standard SGD typically achieves only an $O(1/T)$ convergence rate. However, if Nexus succeeds in finding a region where these tasks share common minimizers, it can achieve exponential convergence:

\begin{theorem}
\label{theorem:cwa_convergence}
    Suppose each $\mathcal{L}_i$ is $L$-smooth and $\mu$-strongly convex. That is, for any $\bm{\theta}_1, \bm{\theta}_2$, we have:
    \begin{equation}
        \begin{aligned}
            \mathcal{L}_i(\bm{\theta}_1) &\leq \mathcal{L}_i(\bm{\theta}_2) + \nabla \mathcal{L}_i(\bm{\theta}_2)^\top (\bm{\theta}_1 - \bm{\theta}_2) + \frac{L}{2} \| \bm{\theta}_1 - \bm{\theta}_2\|_2^2, \\
            \mathcal{L}_i(\bm{\theta}_1) &\geq \mathcal{L}_i(\bm{\theta}_2) + \nabla \mathcal{L}_i(\bm{\theta}_2)^\top (\bm{\theta}_1 - \bm{\theta}_2) + \frac{\mu}{2} \| \bm{\theta}_1 - \bm{\theta}_2\|_2^2.
        \end{aligned}
    \end{equation}
    Additionally, assume there exists a common minimizer \(\bm{\theta}^*\) such that $\nabla \mathcal{L}_i(\bm{\theta}^*)=0$ for all $i \in [K]$. Then, for the sequence $\{\bm{\theta}_0, \bm{\theta}_1, \dots, \bm{\theta}_T\}$ generated by Nexus with step size $\gamma \in (0, \frac{2}{L+\mu})$, we have:
    \begin{equation}
    \mathbb{E}[\|\bm{\theta}_T - \bm{\theta}^*\|^2] \leq \left(1-\frac{2\gamma \mu L}{L+\mu}\right)^T \|\bm{\theta}_{0} - \bm{\theta}^*\|^2.
    \end{equation}
    Specifically, setting $\gamma = \frac{2}{L+\mu}$ and defining the condition number $\kappa = L/\mu$, we obtain the convergence rate:
    \begin{equation}
    \mathbb{E}[\|\bm{\theta}_T - \bm{\theta}^*\|^2] \leq \left(\frac{\kappa-1}{\kappa+1}\right)^{2T} \|\bm{\theta}_{0} - \bm{\theta}^*\|^2.
    \end{equation}
\end{theorem}

Therefore, if Nexus guides the parameters into a locally convex and smooth regime where a common minimizer exists, it guarantees exponential convergence.

\subsubsection{Proof of Convergence Rate (\cref{theorem:cwa_convergence})}
\label{sec:appendix_convergence_proof}

In this section, we provide the detailed proof for \cref{theorem:cwa_convergence}. Let $\bm{\theta}^*$ be the common minimizer such that $\nabla \mathcal{L}_i(\bm{\theta}^*) = 0$ for all $i \in [k]$.
Consider the update at step $m$: $\bm{\theta}_m = \bm{\theta}_{m-1} - \gamma \nabla \mathcal{L}_{s_m}(\bm{\theta}_{m-1})$, where $s_m$ is the task index sampled uniformly at random.

First, we expand the squared distance to the optimum for a specific realization of $s_m$:
\begin{equation}
\label{eq:appendix:update_expansion}
\begin{aligned}
\|\bm{\theta}_m - \bm{\theta}^*\|^2 &= \|\bm{\theta}_{m-1} - \gamma \nabla \mathcal{L}_{s_m}(\bm{\theta}_{m-1}) - \bm{\theta}^*\|^2 \\
&= \|\bm{\theta}_{m-1} - \bm{\theta}^*\|^2 - 2\gamma \langle \nabla \mathcal{L}_{s_m}(\bm{\theta}_{m-1}), \bm{\theta}_{m-1} - \bm{\theta}^* \rangle + \gamma^2 \|\nabla \mathcal{L}_{s_m}(\bm{\theta}_{m-1})\|^2.
\end{aligned}
\end{equation}

To bound the inner product term, we utilize the property of smooth and strongly convex functions. Define the auxiliary function $\phi_i(\bm{\theta}) = \mathcal{L}_i(\bm{\theta}) - \frac{\mu}{2} \|\bm{\theta}\|^2$. Since each $\mathcal{L}_i$ is $L$-smooth and $\mu$-strongly convex, $\phi_i(\bm{\theta})$ is convex and $(L-\mu)$-smooth.
By the co-coercivity property of convex smooth functions, for any $\bm{\theta}$, we have:
\begin{equation}
\langle \nabla \phi_i(\bm{\theta}) - \nabla \phi_i(\bm{\theta}^*), \bm{\theta} - \bm{\theta}^* \rangle \ge \frac{1}{L-\mu} \|\nabla \phi_i(\bm{\theta}) - \nabla \phi_i(\bm{\theta}^*)\|^2.
\end{equation}
Substituting $\nabla \phi_i(\bm{\theta}) = \nabla \mathcal{L}_i(\bm{\theta}) - \mu \bm{\theta}$ and noting that $\nabla \mathcal{L}_i(\bm{\theta}^*) = 0$, we substitute back:
\begin{equation}
\begin{aligned}
\langle \nabla \mathcal{L}_i(\bm{\theta}) - \mu (\bm{\theta} - \bm{\theta}^*), \bm{\theta} - \bm{\theta}^* \rangle
&\ge \frac{1}{L-\mu} \|\nabla \mathcal{L}_i(\bm{\theta}) - \mu (\bm{\theta} - \bm{\theta}^*)\|^2 \\
&= \frac{1}{L-\mu} \left( \|\nabla \mathcal{L}_i(\bm{\theta})\|^2 - 2\mu \langle \nabla \mathcal{L}_i(\bm{\theta}), \bm{\theta} - \bm{\theta}^* \rangle + \mu^2 \|\bm{\theta} - \bm{\theta}^*\|^2 \right).
\end{aligned}
\end{equation}
Rearranging the terms, we obtain the following inequality which holds for any task index $i$, and thus specifically for the sampled index $s_m$:
\begin{equation}
\label{eq:appendix:cocoercivity_bound}
    \langle \nabla \mathcal{L}_{s_m}(\bm{\theta}), \bm{\theta} - \bm{\theta}^* \rangle \geq \frac{1}{L+\mu}\|\nabla \mathcal{L}_{s_m}(\bm{\theta}) \|^2 + \frac{\mu L}{L+\mu}\|\bm{\theta} -\bm{\theta}^*\|^2.
\end{equation}

Substituting \cref{eq:appendix:cocoercivity_bound} back into \cref{eq:appendix:update_expansion} with $\bm{\theta} = \bm{\theta}_{m-1}$:
\begin{equation}
\begin{aligned}
\|\bm{\theta}_m - \bm{\theta}^*\|^2 &\leq \|\bm{\theta}_{m-1} - \bm{\theta}^*\|^2 - 2 \gamma \left(\frac{1}{L+\mu}\|\nabla \mathcal{L}_{s_m}(\bm{\theta}_{m-1}) \|^2 + \frac{\mu L}{L+\mu}\|\bm{\theta}_{m-1} -\bm{\theta}^*\|^2 \right)  + \gamma^2 \|\nabla \mathcal{L}_{s_m}(\bm{\theta}_{m-1})\|^2 \\
&= \left(1-\frac{2\gamma \mu L}{L+\mu}\right)\|\bm{\theta}_{m-1} - \bm{\theta}^*\|^2 + \left(\gamma^2 -\frac{2 \gamma}{L+\mu}\right)\|\nabla \mathcal{L}_{s_m}(\bm{\theta}_{m-1})\|^2.
\end{aligned}
\end{equation}

Provided that the step size satisfies $\gamma \in (0, \frac{2}{L+\mu}]$, the coefficient $\left(\gamma^2 -\frac{2 \gamma}{L+\mu}\right)$ is non-positive. Since $\|\nabla \mathcal{L}_{s_m}(\bm{\theta}_{m-1})\|^2 \ge 0$, we can drop the gradient norm term to obtain an upper bound:
\begin{equation}
    \|\bm{\theta}_m - \bm{\theta}^*\|^2 \leq \left(1-\frac{2\gamma \mu L}{L+\mu}\right)\|\bm{\theta}_{m-1} - \bm{\theta}^*\|^2.
\end{equation}

Since this inequality holds for any realization of the random sample $s_m$, we take the expectation over the sampling distribution. Let $\mathbb{E}[\cdot]$ denote the total expectation over the sequence of random indices $\{s_1, \dots, s_m\}$. We have:
\begin{equation}
    \mathbb{E}[\|\bm{\theta}_m - \bm{\theta}^*\|^2] \leq \left(1-\frac{2\gamma \mu L}{L+\mu}\right) \mathbb{E}[\|\bm{\theta}_{m-1} - \bm{\theta}^*\|^2].
\end{equation}
Applying this recurrence relation recursively for $T$ steps yields:
\begin{equation}
    \mathbb{E}[\|\bm{\theta}_T - \bm{\theta}^*\|^2] \leq \left(1-\frac{2\gamma \mu L}{L+\mu}\right)^T \|\bm{\theta}_{0} - \bm{\theta}^*\|^2.
\end{equation}

Specifically, when choosing the step size $\gamma = \frac{2}{L+\mu}$:
\begin{equation}
    1-\frac{2\gamma \mu L}{L+\mu} = 1 - \frac{4 \mu L}{(L+\mu)^2} = \frac{(L-\mu)^2}{(L+\mu)^2} = \left(\frac{\kappa-1}{\kappa+1}\right)^2,
\end{equation}
where $\kappa = L/\mu$ is the condition number. Thus, we obtain the convergence rate:
\begin{equation}
    \mathbb{E}[\|\bm{\theta}_T - \bm{\theta}^*\|^2] \leq \left(\frac{\kappa-1}{\kappa+1}\right)^{2T} \|\bm{\theta}_{0} - \bm{\theta}^*\|^2.
\end{equation}
\hfill $\square$

\Needspace{9\baselineskip}
\subsection{Implicit Bias of Normalized SGD}
\label{sec:nsgd_implicit_bias}

Interestingly, Nexus also offers a novel perspective on the success of Normalized SGD (NSGD). We observe that NSGD can be mathematically interpreted as a special case of Nexus, revealing that NSGD does not merely minimize the scalar loss but also implicitly optimizes gradient closeness. This implicit regularization provides a geometric explanation for why NSGD often generalizes better than standard Gradient Descent.

\begin{theorem}[Implicit Bias of NSGD]
\label{thm:nsgd_implicit_bias}
    Let $\{\bm{\theta}_t\}$ be the sequence generated by Normalized SGD with learning rate $\gamma$. This sequence implicitly minimizes the following expected joint objective:
    \begin{equation}
        \mathcal{J}_{\text{NSGD}}(\bm{\theta}) = \mathbb{E}_{\bm{x} \sim \mathcal{D}}[\mathcal{L}(\bm{x}; \bm{\theta})] - \frac{\gamma}{8} \mathbb{E}_{\bm{x}, \bm{x}' \stackrel{\text{i.i.d}}{\sim} \mathcal{D}}[\text{CosSim}(\nabla \mathcal{L}(\bm{x}; \bm{\theta}), \nabla \mathcal{L}(\bm{x}'; \bm{\theta}))],
    \end{equation}
    subject to a Jacobian non-symmetric error and a discretization error bounded by $\frac{4}{3} \left( \frac{4L^2 + \rho G_{\min}}{G_{\min}^2} \right) \gamma^3$.
\end{theorem}

\begin{proof}
    A sequence of $n$ updates of NSGD is algebraically equivalent to performing Nexus with $k$ inner steps (using NSGD) and $n/k$ outer steps (using SGD with step size 1). By \cref{thm:cwa_optimize_gradient_similarity}, the magnitude of the gradient alignment signal scales as $S(k) \approx \gamma^2 \frac{k(k-1)}{4}$, while the residual error is bounded by $N(k) \le \frac{1}{6} \left( \frac{4L^2 + \rho G_{\min}}{G_{\min}^2} \right) k^3 \gamma^3$. Defining the signal-to-noise ratio as $\rho(k) \triangleq S(k)/N(k)$ and maximizing it with respect to $k$ yields $k=2$. Thus, viewing the NSGD updates through the lens of Nexus with $k=2$ yields the stated results.

\end{proof}

A related implicit bias on Hessian alignment is analyzed by \citet{wu2018sgd} in over-parameterized regimes. In their setting, the individual loss functions share identical minimizers, causing first-order gradients to vanish and leaving second-order properties as the sole means to characterize closeness. In contrast, modern LLM pretraining is far from this over-parameterized state. Consequently, first-order gradients remain non-zero, enabling the use of first-order gradient similarity to characterize the implicit regularization effect.

\subsection{Other Approximators for Hessian Gradient Product}

While the Hessian-vector product can theoretically be implemented via the Jacobian-vector product (JVP) in PyTorch~\cite{pytorch} with only a constant factor of computational overhead, implementing exact Hessian-gradient products in practical LLM pretraining remains prohibitive. First, standard Hessian-vector product implementations are often incompatible with memory-efficient kernels like FlashAttention~\cite{dao2022flashattention,dao2023flashattention2} (which typically do not support second-order differentiation efficiently), leading to significantly higher memory usage and computational costs. Second, the constant margin of memory overhead poses significant infrastructure challenges for large-scale distributed training.

\subsection{Extension to Generalized Similarity Metrics}
\label{appendix:discussion:generalized_nexus}

In the main text, we instantiate the Nexus algorithm using Normalized SGD (NSGD) in the inner loop to maximize the cosine similarity between task gradients. Nexus may also be extended to optimize a broader class of generalized similarity metrics.

Consider a gradient transformation function $\Phi: \mathbb{R}^d \to \mathbb{R}^d$. We define a generalized gradient similarity metric between two tasks as the inner product of their transformed gradients:
\begin{equation}
    \text{Sim}_{\Phi}(\nabla \mathcal{L}_i(\bm{\theta}), \nabla \mathcal{L}_j(\bm{\theta})) \triangleq \langle \Phi(\nabla \mathcal{L}_i(\bm{\theta})), \Phi(\nabla \mathcal{L}_j(\bm{\theta})) \rangle.
\end{equation}

To maximize this generalized similarity, we replace the normalized gradient in the inner loop with the transformation $\Phi$. The inner update rule becomes:
\begin{equation}
    \bm{\theta}_{t,m} = \bm{\theta}_{t,m-1} - \gamma \Phi(\nabla \mathcal{L}_{s_m}(\bm{\theta}_{t,m-1})).
\end{equation}

Following the Taylor expansion logic in the main text, evaluating the transformed gradient of task $j$ at the displaced parameter state from task $i$ yields:
\begin{equation}
    \Phi(\nabla \mathcal{L}_j(\bm{\theta}_{t,1})) - \Phi(\nabla \mathcal{L}_j(\bm{\theta}_{t,0})) = \left( \nabla_{\bm{\theta}_{t,0}} \Phi(\nabla \mathcal{L}_j(\bm{\theta}_{t,0})) \right) (\bm{\theta}_{t,1} - \bm{\theta}_{t,0}) + \mathcal{O}(\gamma^2)
\end{equation}
\begin{equation}
    = -\gamma \left( \nabla_{\bm{\theta}_{t,0}} \Phi(\nabla \mathcal{L}_j(\bm{\theta}_{t,0})) \right) \Phi(\nabla \mathcal{L}_i(\bm{\theta}_{t,0})) + \mathcal{O}(\gamma^2).
\end{equation}
Taking the expectation over task permutations and pairing symmetric counterparts yields:
\begin{equation}
    \overset{\mathbb{E}}{=} - \frac{\gamma}{2} \left( \nabla_{\bm{\theta}} \text{Sim}_{\Phi}(\nabla \mathcal{L}_i(\bm{\theta}_{t,0}), \nabla \mathcal{L}_j(\bm{\theta}_{t,0})) + \mathcal{E}_{\text{sym}, \Phi, i, j} + \mathcal{E}_{\text{sym}, \Phi, j, i} \right) + \mathcal{O}(\gamma^2),
\end{equation}
\begingroup\raggedright
where $\mathcal{E}_{\text{sym}, \Phi, i, j} \triangleq \left( \nabla_{\bm{\theta}} \Phi(\nabla \mathcal{L}_i(\bm{\theta})) - \left(\nabla_{\bm{\theta}} \Phi(\nabla \mathcal{L}_i(\bm{\theta}))\right)^\top \right) \Phi(\nabla \mathcal{L}_j(\bm{\theta}))$ represents the generalized \mbox{Jacobian non-symmetric error.}
\par\endgroup

This derivation shows that for a given choice of the transformation function $\Phi$, the sequential evaluation in the inner loop constructs a Jacobian-vector product that aligns with the gradient of $\text{Sim}_{\Phi}$, subject to \mbox{the non-symmetric error.} 

The generalized update covers several specific formulations:
\begin{itemize}
    \item When $\Phi(\nabla \mathcal{L}) = \nabla \mathcal{L}$, Nexus optimizes the unnormalized dot product.
    \item When $\Phi(\nabla \mathcal{L}) = \bm{A} \nabla \mathcal{L}$ for a full-rank matrix $\bm{A}$, $\text{Sim}_{\Phi}$ induces a Mahalanobis inner product. For non-linear $\Phi$, $\text{Sim}_{\Phi}$ acts as a kernel function mapping gradients into a Reproducing Kernel Hilbert Space (RKHS).
    \item When $\Phi(\nabla \mathcal{L}) = \frac{\nabla \mathcal{L}}{\|\nabla \mathcal{L}\|_2}$, the update recovers the standard cosine similarity maximization.
    \item Other choices, such as $\Phi(\nabla \mathcal{L}) = \text{sign}(\nabla \mathcal{L})$ or the orthogonalized gradient in Muon~\cite{jordan2024muon}, enforce different geometric consensus choices across tasks.
\end{itemize}

\begin{algorithm}[t]
\caption{Generalized Nexus Algorithm}
\label{alg:cwa_generalized}
\begin{algorithmic}[1]
    \REQUIRE Initial params $\bm{\theta}_0$, losses $\{\mathcal{L}_i\}_{i=1}^K$, total iterations $T$.
    \REQUIRE Gradient transformation function $\Phi: \mathbb{R}^d \to \mathbb{R}^d$.
    \REQUIRE Optimizers: $\text{Opt}_{\text{inner}}$ (Memory-free update via $\Phi$), $\text{Opt}_{\text{outer}}$ (e.g., AdamW). Inner learning rate $\gamma$.
    \textbf{Implicit Objective:} Maximize expected generalized similarity $\mathbb{E}_{i,j}[\langle \Phi(\nabla \mathcal{L}_i(\bm{\theta})), \Phi(\nabla \mathcal{L}_j(\bm{\theta})) \rangle]$.
    \FOR{$t=1$ {\bfseries to} $T$}
        \STATE $\bm{\theta}_{t, 0} \leftarrow \bm{\theta}_{t-1}$ \COMMENT{Initialize inner loop}
        \FOR{$m=1$ {\bfseries to} $K$}
            \STATE Sample task index $s_m \sim \text{Uniform}(\{1, \dots, K\})$
            \STATE $\bm{g} \leftarrow \nabla \mathcal{L}_{s_m}(\bm{\theta}_{t, m-1})$
            \STATE $\bm{\theta}_{t, m} \leftarrow \bm{\theta}_{t, m-1} - \gamma \cdot \Phi(\bm{g})$ \COMMENT{Update inner trajectory with $\Phi$}
        \ENDFOR
        \STATE $\hat{\bm{g}}_t \leftarrow \bm{\theta}_{t, 0} - \bm{\theta}_{t, K}$ \COMMENT{Compute generalized pseudo-gradient}
        \STATE $\bm{\theta}_{t} \leftarrow \text{Opt}_{\text{outer}}(\bm{\theta}_{t-1}, \hat{\bm{g}}_t)$ \COMMENT{Outer update}
    \ENDFOR
    \textbf{Return}:  $\bm{\theta}_T$
\end{algorithmic}
\end{algorithm}

\clearpage

\section{Additional Experiments}

\subsection{Detailed Hyper-parameters}
\label{appendix:hyper_params}

Our hyperparameter configurations strictly follow the baseline established in \citet{wen2025fantastic}. To verify the optimality of these settings on our specific pretraining corpus, we conduct a learning rate grid search with a multiplier of 2 (i.e., evaluating $0.5\times$ and $2.0\times$), which confirms that the original configurations remain optimal. For clarity and reproducibility, the key hyperparameters are summarized in \cref{tab:appendix:all_hyper_parameters}. Regarding the learning rate schedule, the main experiments that utilize the Warmup-Stable-Decay (WSD) scheduler employ 1,000 warmup steps and 10,000 decay steps following \citet{wen2025fantastic}, while the schedules for various ablation studies are clarified in their respective sections. Across all settings, we maintain a global batch size of 256, an Adam $\beta$ of $(0.9, 0.95)$, an Adam $\epsilon$ of $10^{-10}$, and a gradient clipping norm of 1.0.

\begin{table}[h!]
    \centering
    \caption{Summary of key hyperparameters for the main pretraining experiments.}
    \label{tab:appendix:all_hyper_parameters}
    \setlength{\tabcolsep}{0.7ex} 
    \begin{tabular}{cl ccccc c}
    \toprule
        \shortstack{\textbf{Model}\\\textbf{Size}} & \textbf{Optimizer} & \textbf{Outer LR} & \textbf{Inner LR ($\gamma$)} & \textbf{Chinchilla} & \textbf{Tokens (B)} & \shortstack{\textbf{Weight}\\\textbf{Decay}} & \textbf{Reference} \\
    \midrule
        \multirow{2}{*}{1B} & Adam & 0.002 & - & $4\times$ & 50 & 0.2 & \multirow{2}{*}{\cref{sec:discussion:scaling:num_params}} \\
        & Nexus & 0.002 & 0.04 & $4\times$ & 50 & 0.2 & \\
    \midrule
        \multirow{3}{*}{3B} & Adam & 0.001 & - & $2\times$ & 110 & 0.2 & \multirow{3}{*}{\begin{tabular}{c} \cref{sec:exp:main_result,sec:exp:data_mixture} \\ \cref{sec:discussion:scaling:tokens,sec:exp:nvidia_nemotron} \end{tabular}} \\
        & Muon & 0.001 & - & $2\times$ & 110 & 0.1 & \\
        & Nexus & 0.001 & 0.1 & $2\times$ & 110 & 0.2 & \\
    \bottomrule
    \end{tabular}
\end{table}

Compared to the base optimizers, Nexus introduces only one additional hyperparameter: the inner-loop learning rate $\gamma$. We tune $\gamma$ based on a straightforward yet grounded empirical criterion: maximizing the regularization strength while strictly operating within the "same pretraining loss" regime. Specifically, we select the largest value of $\gamma$ that incurs no discernible degradation in pretraining loss convergence, thereby maximizing the effect of the Nexus regularizer.

\subsection{Detailed Results for Data Mixture}

\begin{table*}[ht!]
    \centering
    \caption{\textbf{Results on varying data mixtures} (3B models). Hyperparameters follow \cref{sec:exp:setting,sec:exp:main_result}. As the proportion of math data increases (10\% $\to$ 70\%), the relative performance gains of Nexus on math benchmarks gradually diminish, whereas its advantages on other domains (General, Reasoning) progressively expand. This suggests Nexus boosts the sample-sparse or harder-to-learn domains in the mixture.}
    \label{tab:exp:d803_data_mixture}
    
    \setlength{\tabcolsep}{0.35ex} 
    \resizebox{\textwidth}{!}{
        \begin{tabular}{lll cc c ccc cc cc c}
            \toprule
            \multirow{2}{*}{\textbf{Data}} & \multirow{2}{*}{\textbf{Optim.}} & \multirow{2}{*}{\textbf{Metric}} & 
            \multicolumn{2}{c}{\textbf{Loss Metrics} ($\downarrow$)} & 
            \multicolumn{1}{c}{\textbf{Gen.}} & 
            \multicolumn{3}{c}{\textbf{Reasoning}} & 
            \multicolumn{2}{c}{\textbf{Math}} & 
            \multicolumn{2}{c}{\textbf{Code}} &
            \multicolumn{1}{c}{\textbf{Avg.}} \\
            \cmidrule(lr){4-5} \cmidrule(lr){6-6} \cmidrule(lr){7-9} \cmidrule(lr){10-11} \cmidrule(lr){12-13} \cmidrule(lr){14-14}
             & & & Pretrain. & OOD & MMLU & GPQA & GPQA-D & BBH & GSM8k & MATH & HumanEval & MBPP & All \\
            \midrule
            
            \multirow{5}{*}{Math10} 
              & \multirow{2}{*}{AdamW} & Acc. ($\uparrow$) & \multirow{2}{*}{1.606} & \multirow{2}{*}{1.302} & 47.8 & \textbf{32.8} & 22.6 & 36.6 & 44.0 & 32.0 & 43.0 & 38.0 & 37.1 \\
              & & Loss ($\downarrow$) & & & 2.265 & 2.005 & 1.910 & 1.534 & 1.259 & 1.054 & 1.116 & 1.922 & 1.633 \\
              \cmidrule{2-14}
              & \multirow{2}{*}{Nexus} & Acc. ($\uparrow$) & \multirow{2}{*}{1.602} & \multirow{2}{*}{\textbf{1.290}} & 48.9 & 29.6 & 23.4 & 36.6 & \textbf{59.0} & \textbf{40.0} & \textbf{47.0} & 38.0 & \textbf{40.3} \\
              & & Loss ($\downarrow$) & & & \textbf{2.179} & \textbf{1.981} & \textbf{1.881} & \textbf{1.504} & \textbf{1.227} & \textbf{1.026} & \textbf{1.086} & 1.921 & \textbf{1.601} \\
              \cmidrule{2-14}
              & \textit{Improv.} & Loss ($\uparrow$) & +0.004 & \textbf{+0.012} & \textbf{+0.086} & \textbf{+0.024} & \textbf{+0.029} & \textbf{+0.030} & \textbf{+0.032} & \textbf{+0.028} & \textbf{+0.030} & +0.001 & \textbf{+0.032} \\
            \midrule
            
            \multirow{5}{*}{Math40} 
              & \multirow{2}{*}{AdamW} & Acc. ($\uparrow$) & \multirow{2}{*}{1.336} & \multirow{2}{*}{1.330} & 47.8 & 29.6 & 22.6 & 38.1 & 64.0 & 44.0 & 41.0 & 38.0 & 40.6 \\
              & & Loss ($\downarrow$) & & & 2.210 & 1.989 & 1.891 & 1.522 & 1.171 & 0.969 & 1.144 & 1.976 & 1.609 \\
              \cmidrule{2-14}
              & \multirow{2}{*}{Nexus} & Acc. ($\uparrow$) & \multirow{2}{*}{1.339} & \multirow{2}{*}{1.331} & \textbf{51.2} & \textbf{33.5} & \textbf{27.3} & \textbf{41.1} & \textbf{70.0} & 43.0 & \textbf{45.0} & \textbf{44.5} & \textbf{44.4} \\
              & & Loss ($\downarrow$) & & & \textbf{2.182} & 1.990 & 1.889 & \textbf{1.511} & \textbf{1.117} & \textbf{0.929} & \textbf{1.132} & \textbf{1.876} & \textbf{1.578} \\
              \cmidrule{2-14}
              & \textit{Improv.} & Loss ($\uparrow$) & -0.003 & -0.001 & \textbf{+0.028} & -0.001 & +0.002 & \textbf{+0.011} & \textbf{+0.054} & \textbf{+0.040} & \textbf{+0.012} & \textbf{+0.100} & \textbf{+0.031} \\
            \midrule
            
            \multirow{5}{*}{Math70} 
              & \multirow{2}{*}{AdamW} & Acc. ($\uparrow$) & \multirow{2}{*}{1.033} & \multirow{2}{*}{\textbf{1.399}} & 44.0 & 27.3 & 23.4 & 41.1 & 77.0 & 45.0 & 38.0 & 38.0 & 41.7 \\
              & & Loss ($\downarrow$) & & & 2.252 & 2.025 & 1.923 & 1.541 & 1.111 & 0.923 & 1.178 & 1.897 & 1.606 \\
              \cmidrule{2-14}
              & \multirow{2}{*}{Nexus} & Acc. ($\uparrow$) & \multirow{2}{*}{1.040} & \multirow{2}{*}{1.409} & \textbf{49.8} & \textbf{30.4} & 23.4 & 42.5 & 76.0 & \textbf{52.0} & \textbf{41.0} & 38.0 & \textbf{44.1} \\
              & & Loss ($\downarrow$) & & & \textbf{2.221} & 2.037 & 1.936 & 1.548 & \textbf{1.082} & \textbf{0.911} & 1.176 & \textbf{1.872} & 1.598 \\
              \cmidrule{2-14}
              & \textit{Improv.} & Loss ($\uparrow$) & -0.007 & -0.010 & \textbf{+0.031} & -0.012 & -0.013 & -0.007 & \textbf{+0.029} & \textbf{+0.012} & +0.002 & \textbf{+0.025} & +0.008 \\
            
            \bottomrule
        \end{tabular}
    }
\end{table*}

As shown in \cref{tab:exp:d803_data_mixture}, we observe a dynamic trade-off mechanism:
\begin{itemize}
    \item \textbf{In the sample-sparse regime (Math10):} Where math data is scarce, the baseline optimizer struggles to generalize on reasoning tasks. Nexus provides the most significant gains here (e.g., \textbf{+15.0} on GSM8k), effectively "mining" the rare training signals to build robust reasoning capabilities.
    \item \textbf{In the sample-dense regime (Math70):} As math data becomes abundant, the baseline catches up on math benchmarks. However, Nexus automatically shifts its advantage to the \textit{now-relative-minority} domains. It significantly boosts General Knowledge (MMLU: \textbf{+5.8}) and broad Reasoning (GPQA: \textbf{+3.1}) compared to the baseline, which begins to suffer from domain dominance.
    \item \textbf{Lower sensitivity to mixture shifts:} Nexus also demonstrates higher stability against drastic changes in data mixture. When shifting from a math-heavy (Math70) to a math-sparse (Math10) mixture, the performance variance of Nexus is significantly smaller than that of the baseline. For instance, while the baseline's GSM8k score drops precipitously by \textbf{33.0 points} (from 77.0 to 44.0), Nexus mitigates this degradation, dropping only \textbf{17.0 points} (from 76.0 to 59.0). Similarly, on MMLU, while the baseline fluctuates by \textbf{3.8 points}, Nexus remains highly stable with a variation of less than \textbf{1.0 point} (49.8 vs. 48.9), demonstrating its stability against data mixture changes.
\end{itemize}
This suggests that Nexus reduces sensitivity to manual data mixing ratios, acting as an automatic balancer that prioritizes representations for the most under-optimized tasks in the mixture.

\subsection{Experiments on a Public Dataset}
\label{sec:exp:nvidia_nemotron}

\textbf{Motivation.} While our primary analyses utilize strictly cleaned data to avoid confounding factors, many popular open-source pretraining datasets inevitably suffer from data contamination, inadvertently including benchmark training sets (e.g., GSM8k). We evaluate Nexus on a public dataset from \citet{basant2025nvidia_nemotron} to investigate whether its consensus-seeking mechanism remains robust and mitigates shortcut over-memorization in the presence of such noisy, contaminated signals.

\textbf{Settings.} We train the 1B and 3B models on a public dataset \citep{basant2025nvidia_nemotron}. All other training configurations, including model architectures and base optimizer hyperparameters, are kept strictly identical to the main experiments detailed in \cref{sec:exp:main_result}.

\begin{table*}[h!]
    \centering
    \caption{\textbf{Results on a public pretraining dataset~\cite{basant2025nvidia_nemotron}.} The Adam baseline exhibits artificial performance inflation on leaked benchmarks. In contrast, Nexus effectively resists shortcut over-memorization, successfully reallocating model capacity to uncontaminated tasks and achieving superior overall OOD generalization.}
    \label{tab:exp:nvidia_nemotroncc_main}
    
    \setlength{\tabcolsep}{0.35ex} 
    \resizebox{\textwidth}{!}{
        \begin{tabular}{lll cc c ccc cc cc c}
            \toprule
            \multirow{2}{*}{\textbf{Model}} & \multirow{2}{*}{\textbf{Optim.}} & \multirow{2}{*}{\textbf{Metric}} & 
            \multicolumn{2}{c}{\textbf{Loss Metrics} ($\downarrow$)} & 
            \multicolumn{1}{c}{\textbf{Gen.}} & 
            \multicolumn{3}{c}{\textbf{Reasoning}} & 
            \multicolumn{2}{c}{\textbf{Math}} & 
            \multicolumn{2}{c}{\textbf{Code}} &
            \multicolumn{1}{c}{\textbf{Avg.}} \\
            \cmidrule(lr){4-5} \cmidrule(lr){6-6} \cmidrule(lr){7-9} \cmidrule(lr){10-11} \cmidrule(lr){12-13} \cmidrule(lr){14-14}
             & & & Pretrain. & OOD & MMLU & GPQA & GPQA-D & BBH & GSM8k & MATH & HumanEval & MBPP & All \\
            \midrule
            
            \multirow{5}{*}{1B} 
              & \multirow{2}{*}{Adam} & Acc. ($\uparrow$) & \multirow{2}{*}{1.331} & \multirow{2}{*}{1.863} & \textbf{34.2} & 25.8 & 18.8 & 24.8 & 18.0 & 14.0 & 38.0 & 1.0 & 21.8 \\
              & & Loss ($\downarrow$) & & & 2.552 & 2.280 & 2.191 & 1.700 & \textbf{1.708} & 1.346 & 1.324 & 2.991 & 2.011 \\
              \cmidrule{2-14}
              & \multirow{2}{*}{Nexus} & Acc. ($\uparrow$) & \multirow{2}{*}{1.338} & \multirow{2}{*}{\textbf{1.835}} & 31.4 & 25.0 & 18.8 & 23.0 & \textbf{22.0} & 12.0 & \textbf{41.0} & \textbf{15.0} & 23.5 \\
              & & Loss ($\downarrow$) & & & \textbf{2.446} & \textbf{2.261} & \textbf{2.172} & \textbf{1.689} & 1.749 & \textbf{1.325} & 1.325 & \textbf{2.908} & \textbf{1.984} \\
              \cmidrule{2-14}
              & \textit{Improv.} & Loss ($\uparrow$) & -0.007 & \textbf{+0.028} & \textbf{+0.106} & \textbf{+0.019} & \textbf{+0.019} & \textbf{+0.011} & \textbf{-0.041} & \textbf{+0.021} & -0.001 & \textbf{+0.083} & \textbf{+0.027} \\
            \midrule
            
            \multirow{5}{*}{3B} 
              & \multirow{2}{*}{Adam} & Acc. ($\uparrow$) & \multirow{2}{*}{1.330} & \multirow{2}{*}{1.623} & 55.1 & 25.0 & \textbf{27.3} & \textbf{47.4} & 44.0 & 31.0 & 59.0 & 23.0 & 39.0 \\
              & & Loss ($\downarrow$) & & & \textbf{2.356} & 2.062 & 1.975 & \textbf{1.529} & \textbf{1.519} & 1.121 & 1.205 & 2.732 & 1.812 \\
              \cmidrule{2-14}
              & \multirow{2}{*}{Nexus} & Acc. ($\uparrow$) & \multirow{2}{*}{1.338} & \multirow{2}{*}{\textbf{1.606}} & 56.2 & 25.8 & 23.4 & 44.4 & \textbf{47.0} & 32.0 & \textbf{63.0} & \textbf{38.0} & \textbf{41.2} \\
              & & Loss ($\downarrow$) & & & 2.380 & \textbf{2.047} & \textbf{1.957} & 1.540 & 1.533 & \textbf{1.106} & 1.199 & \textbf{2.530} & \textbf{1.786} \\
              \cmidrule{2-14}
              & \textit{Improv.} & Loss ($\uparrow$) & -0.008 & \textbf{+0.017} & \textbf{-0.024} & \textbf{+0.015} & \textbf{+0.018} & \textbf{-0.011} & \textbf{-0.014} & \textbf{+0.015} & +0.006 & \textbf{+0.202} & \textbf{+0.026} \\
            
            \bottomrule
        \end{tabular}
    }
\end{table*}

\textbf{Results.} As shown in \cref{tab:exp:nvidia_nemotroncc_main}, the Adam baseline exhibits artificial performance inflation on potentially contaminated benchmarks like GSM8k. In contrast, Nexus resists overfitting to these leaked signals and effectively reallocates the model's capacity to uncontaminated, sparse domains. This dynamic balancing is evidenced by the striking improvements on coding tasks—such as MBPP accuracy increasing from 1.0\% to 15.0\% (1B) and 23.0\% to 38.0\% (3B)—ultimately leading to a consistently lower OOD loss across both scales.

\subsection{Detailed Results for Model Size Scaling}

This section provides the detailed experimental results corresponding to the model size scaling analysis discussed in \cref{sec:discussion:scaling:num_params}.

\begin{table*}[h!]
    \centering
   \caption{\textbf{Benchmark Performance across Model Scales.} We compare downstream capabilities for models ranging from 130M to 2.3B parameters. Notably, the relative gains of Nexus over the base optimizer amplify as model capacity increases, with the average benchmark accuracy improvement growing from +0.8\% on the 130M model to +3.2\% on the 2.3B model.}
    \label{tab:exp:d803_scaling}
    
    \setlength{\tabcolsep}{0.35ex} 
    \resizebox{\textwidth}{!}{
        \begin{tabular}{lll cc c ccc cc cc c}
            \toprule
            \multirow{2}{*}{\textbf{Size}} & \multirow{2}{*}{\textbf{Optim.}} & \multirow{2}{*}{\textbf{Metric}} & 
            \multicolumn{2}{c}{\textbf{Loss Metrics} ($\downarrow$)} & 
            \multicolumn{1}{c}{\textbf{Gen.}} & 
            \multicolumn{3}{c}{\textbf{Reasoning}} & 
            \multicolumn{2}{c}{\textbf{Math}} & 
            \multicolumn{2}{c}{\textbf{Code}} &
            \multicolumn{1}{c}{\textbf{Avg.}} \\
            \cmidrule(lr){4-5} \cmidrule(lr){6-6} \cmidrule(lr){7-9} \cmidrule(lr){10-11} \cmidrule(lr){12-13} \cmidrule(lr){14-14}
             & & & Pretrain. & OOD & MMLU & GPQA & GPQA-D & BBH & GSM8k & MATH & HumanEval & MBPP & All \\
            \midrule
            
            \multirow{6}{*}{130M} 
              & \multirow{2}{*}{AdamW} & Acc. ($\uparrow$) & \multirow{2}{*}{2.038} & \multirow{2}{*}{1.559} & 28.0 & 22.6 & \textbf{24.2} & 25.1 & 7.0 & 6.0 & 0.0 & 10.0 & 15.4 \\
              & & Loss ($\downarrow$) & & & 2.555 & 2.438 & 2.338 & 1.793 & 1.612 & 1.360 & 1.407 & 2.230 & 1.967 \\
              \cmidrule{2-14}
              & \multirow{2}{*}{Nexus} & Acc. ($\uparrow$) & \multirow{2}{*}{2.031} & \multirow{2}{*}{1.549} & 27.0 & \textbf{25.7} & 21.8 & \textbf{28.8} & 5.0 & \textbf{11.0} & 0.0 & 10.0 & 16.2 \\
              & & Loss ($\downarrow$) & & & \textbf{2.523} & \textbf{2.414} & \textbf{2.312} & 1.792 & \textbf{1.601} & \textbf{1.330} & \textbf{1.384} & \textbf{2.183} & \textbf{1.942} \\
              \cmidrule{2-14}
              & \multirow{2}{*}{\textit{Improv.}} & Acc. ($\uparrow$) & - & - & -1.0 & \textbf{+3.1} & -2.4 & \textbf{+3.7} & -2.0 & \textbf{+5.0} & 0.0 & 0.0 & +0.8 \\
              & & Loss ($\uparrow$) & +0.007 & +0.010 & \textbf{+0.032} & \textbf{+0.024} & \textbf{+0.026} & +0.001 & \textbf{+0.011} & \textbf{+0.030} & \textbf{+0.023} & \textbf{+0.047} & \textbf{+0.024} \\
            \midrule
            
            \multirow{6}{*}{300M} 
              & \multirow{2}{*}{AdamW} & Acc. ($\uparrow$) & \multirow{2}{*}{1.909} & \multirow{2}{*}{1.474} & \textbf{33.3} & 26.5 & 21.8 & 31.8 & 15.0 & 15.0 & 6.0 & 16.0 & 20.7 \\
              & & Loss ($\downarrow$) & & & 2.495 & 2.296 & 2.196 & 1.704 & 1.471 & 1.255 & 1.322 & 2.141 & 1.860 \\
              \cmidrule{2-14}
              & \multirow{2}{*}{Nexus} & Acc. ($\uparrow$) & \multirow{2}{*}{1.901} & \multirow{2}{*}{1.469} & 30.3 & 27.3 & \textbf{25.0} & 30.7 & 14.0 & 16.0 & \textbf{13.0} & \textbf{21.0} & 22.2 \\
              & & Loss ($\downarrow$) & & & \textbf{2.381} & \textbf{2.278} & \textbf{2.177} & 1.707 & 1.470 & \textbf{1.237} & \textbf{1.298} & \textbf{2.046} & \textbf{1.824} \\
              \cmidrule{2-14}
              & \multirow{2}{*}{\textit{Improv.}} & Acc. ($\uparrow$) & - & - & -3.0 & +0.8 & \textbf{+3.2} & -1.1 & -1.0 & +1.0 & \textbf{+7.0} & \textbf{+5.0} & +1.5 \\
              & & Loss ($\uparrow$) & +0.008 & +0.005 & \textbf{+0.114} & \textbf{+0.018} & \textbf{+0.019} & -0.003 & +0.001 & \textbf{+0.018} & \textbf{+0.024} & \textbf{+0.095} & \textbf{+0.036} \\
            \midrule
            
            \multirow{6}{*}{520M} 
              & \multirow{2}{*}{AdamW} & Acc. ($\uparrow$) & \multirow{2}{*}{1.826} & \multirow{2}{*}{1.433} & 32.1 & 25.0 & 21.8 & 29.6 & 18.0 & 13.0 & 19.0 & 17.0 & 21.9 \\
              & & Loss ($\downarrow$) & & & 2.363 & 2.221 & 2.124 & 1.640 & 1.429 & 1.204 & 1.270 & 2.035 & 1.786 \\
              \cmidrule{2-14}
              & \multirow{2}{*}{Nexus} & Acc. ($\uparrow$) & \multirow{2}{*}{1.826} & \multirow{2}{*}{1.428} & 33.5 & \textbf{30.4} & 21.8 & 29.3 & 20.0 & 13.0 & 19.0 & \textbf{22.0} & 23.6 \\
              & & Loss ($\downarrow$) & & & \textbf{2.316} & \textbf{2.201} & \textbf{2.102} & 1.638 & \textbf{1.396} & \textbf{1.176} & 1.261 & \textbf{1.977} & \textbf{1.758} \\
              \cmidrule{2-14}
              & \multirow{2}{*}{\textit{Improv.}} & Acc. ($\uparrow$) & - & - & +1.4 & \textbf{+5.4} & 0.0 & -0.3 & +2.0 & 0.0 & 0.0 & \textbf{+5.0} & +1.7 \\
              & & Loss ($\uparrow$) & 0.000 & +0.005 & \textbf{+0.047} & \textbf{+0.020} & \textbf{+0.022} & +0.002 & \textbf{+0.033} & \textbf{+0.028} & +0.009 & \textbf{+0.058} & \textbf{+0.027} \\
            \midrule
            
            \multirow{6}{*}{1.2B} 
              & \multirow{2}{*}{AdamW} & Acc. ($\uparrow$) & \multirow{2}{*}{1.714} & \multirow{2}{*}{1.364} & \textbf{44.4} & 22.6 & 24.2 & 27.4 & 30.0 & 23.0 & 30.0 & 31.0 & 29.1 \\
              & & Loss ($\downarrow$) & & & 2.626 & 2.410 & 2.000 & 1.799 & 1.338 & 1.112 & 1.199 & 2.023 & 1.813 \\
              \cmidrule{2-14}
              & \multirow{2}{*}{Nexus} & Acc. ($\uparrow$) & \multirow{2}{*}{1.707} & \multirow{2}{*}{1.358} & 41.7 & \textbf{25.7} & 23.4 & \textbf{32.9} & \textbf{37.0} & \textbf{28.0} & \textbf{35.0} & 30.0 & \textbf{31.7} \\
              & & Loss ($\downarrow$) & & & \textbf{2.466} & \textbf{2.373} & \textbf{1.984} & 1.792 & \textbf{1.325} & \textbf{1.109} & \textbf{1.179} & \textbf{1.987} & \textbf{1.777} \\
              \cmidrule{2-14}
              & \multirow{2}{*}{\textit{Improv.}} & Acc. ($\uparrow$) & - & - & -2.7 & \textbf{+3.1} & -0.8 & \textbf{+5.5} & \textbf{+7.0} & \textbf{+5.0} & \textbf{+5.0} & -1.0 & \textbf{+2.6} \\
              & & Loss ($\uparrow$) & +0.007 & +0.006 & \textbf{+0.160} & \textbf{+0.037} & \textbf{+0.016} & +0.007 & \textbf{+0.013} & +0.003 & \textbf{+0.020} & \textbf{+0.036} & \textbf{+0.036} \\
            \midrule
            
            \multirow{6}{*}{2.3B} 
              & \multirow{2}{*}{AdamW} & Acc. ($\uparrow$) & \multirow{2}{*}{1.606} & \multirow{2}{*}{1.302} & 47.8 & \textbf{32.8} & 22.6 & 36.6 & 44.0 & 32.0 & 43.0 & 38.0 & 37.1 \\
              & & Loss ($\downarrow$) & & & 2.265 & 2.005 & 1.910 & 1.534 & 1.259 & 1.054 & 1.116 & 1.922 & 1.633 \\
              \cmidrule{2-14}
              & \multirow{2}{*}{Nexus} & Acc. ($\uparrow$) & \multirow{2}{*}{1.602} & \multirow{2}{*}{\textbf{1.290}} & 48.9 & 29.6 & \textbf{23.4} & 36.6 & \textbf{59.0} & \textbf{40.0} & \textbf{47.0} & 38.0 & \textbf{40.3} \\
              & & Loss ($\downarrow$) & & & \textbf{2.179} & \textbf{1.981} & \textbf{1.881} & \textbf{1.504} & \textbf{1.227} & \textbf{1.026} & \textbf{1.086} & 1.921 & \textbf{1.601} \\
              \cmidrule{2-14}
              & \multirow{2}{*}{\textit{Improv.}} & Acc. ($\uparrow$) & - & - & +1.1 & -3.2 & +0.8 & 0.0 & \textbf{+15.0} & \textbf{+8.0} & \textbf{+4.0} & 0.0 & \textbf{+3.2} \\
              & & Loss ($\uparrow$) & +0.004 & \textbf{+0.012} & \textbf{+0.086} & \textbf{+0.024} & \textbf{+0.029} & \textbf{+0.030} & \textbf{+0.032} & \textbf{+0.028} & \textbf{+0.030} & +0.001 & \textbf{+0.032} \\
            
            \bottomrule
        \end{tabular}
    }
\end{table*}

As demonstrated above and analyzed in \cref{sec:discussion:scaling:num_params}, Nexus consistently outperforms the base optimizer across all evaluated model scales, with average benchmark accuracy improvements of +0.8\% (130M), +1.5\% (300M), +1.7\% (520M), +2.6\% (1.2B), and +3.2\% (2.3B).

\newpage
\subsection{Experiments on Muon Optimizers}

This section provides the detailed experimental results discussed in \cref{sec:exp:main_result,sec:discussion:implicit_bias_of_muon}.

\begin{table*}[h!]
    \centering
    \caption{\textbf{Comparison with Muon Optimizer on 3B Models.} As shown, Muon improves downstream performance by decreasing the pretraining loss. While Nexus achieves nearly the same pretraining loss as AdamW, it achieves comparable performance to Muon on downstream tasks.}
    \label{tab:exp:d803:muon}
    \setlength{\tabcolsep}{0.35ex} 
    \resizebox{\textwidth}{!}{
        \begin{tabular}{ll cc c ccc cc cc c}
            \toprule
            \multirow{2}{*}{\textbf{Optim.}} & \multirow{2}{*}{\textbf{Metric}} & 
            \multicolumn{2}{c}{\textbf{Loss Metrics} ($\downarrow$)} & 
            \multicolumn{1}{c}{\textbf{Gen.}} & 
            \multicolumn{3}{c}{\textbf{Reasoning}} & 
            \multicolumn{2}{c}{\textbf{Math}} & 
            \multicolumn{2}{c}{\textbf{Code}} &
            \multicolumn{1}{c}{\textbf{Avg.}} \\
            \cmidrule(lr){3-4} \cmidrule(lr){5-5} \cmidrule(lr){6-8} \cmidrule(lr){9-10} \cmidrule(lr){11-12} \cmidrule(lr){13-13}
             & & Pretrain. & OOD & MMLU & GPQA & GPQA-D & BBH & GSM8k & MATH & HumanEval & MBPP & All \\
            \midrule
            
            \multirow{2}{*}{AdamW} & Acc. ($\uparrow$) & \multirow{2}{*}{1.606} & \multirow{2}{*}{1.302} & 47.8 & 32.8 & 22.6 & 36.6 & 44.0 & 32.0 & 43.0 & 38.0 & 37.1 \\
             & Loss ($\downarrow$) & & & 2.265 & 2.005 & 1.910 & 1.534 & 1.259 & 1.054 & 1.116 & 1.922 & 1.633 \\
            \midrule
            
            \multirow{2}{*}{Adam+Nexus} & Acc. ($\uparrow$) & \multirow{2}{*}{1.602} & \multirow{2}{*}{\textbf{1.290}} & 48.9 & 29.6 & 23.4 & 36.6 & \textbf{59.0} & \textbf{40.0} & \textbf{47.0} & 38.0 & \textbf{40.3} \\
             & Loss ($\downarrow$) & & & \textbf{2.179} & \textbf{1.981} & \textbf{1.881} & \textbf{1.504} & \textbf{1.227} & \textbf{1.026} & \textbf{1.086} & 1.921 & \textbf{1.601} \\
            \cmidrule{2-13}
            \multirow{2}{*}{\textit{(- AdamW)}} & Acc. ($\uparrow$) & - & - & +1.1 & -3.2 & +0.8 & 0.0 & \textbf{+15.0} & \textbf{+8.0} & \textbf{+4.0} & 0.0 & \textbf{+3.2} \\
             & Loss ($\uparrow$) & +0.004 & \textbf{+0.012} & \textbf{+0.086} & \textbf{+0.024} & \textbf{+0.029} & \textbf{+0.030} & \textbf{+0.032} & \textbf{+0.028} & \textbf{+0.030} & +0.001 & \textbf{+0.032} \\
            \midrule
            
            \multirow{2}{*}{Muon} & Acc. ($\uparrow$) & \multirow{2}{*}{1.577} & \multirow{2}{*}{1.285} & 49.8 & 32.0 & 24.2 & 41.9 & 46.0 & 38.0 & 40.0 & 43.0 & 39.4 \\
             & Loss ($\downarrow$) & & & 2.188 & 1.968 & 1.874 & 1.502 & 1.236 & 1.035 & 1.091 & 1.951 & 1.606 \\
            \cmidrule{2-13}
            \multirow{2}{*}{\textit{(- AdamW)}} & Acc. ($\uparrow$) & - & - & +2.0 & -0.8 & +1.6 & +5.3 & +2.0 & +6.0 & -3.0 & +5.0 & +2.3 \\
             & Loss ($\uparrow$) & +0.029 & +0.017 & +0.077 & +0.037 & +0.036 & +0.032 & +0.023 & +0.019 & +0.025 & -0.029 & +0.027 \\

            \bottomrule
        \end{tabular}
    }
\end{table*}

As demonstrated above, Nexus achieves comparable downstream performance to Muon, despite maintaining a pretraining loss that is nearly identical to the AdamW baseline. These results explicitly indicate that while Muon improves downstream performance primarily by reaching a significantly lower pretraining loss, the gains from Nexus stem directly from its favorable implicit bias.

\subsection{Downstream SFT}

\textbf{Motivation and Settings.} To verify whether the performance gains of Nexus are merely a result of "pre-consuming" the potential improvements of the SFT phase in advance, we evaluate the supervised fine-tuning (SFT) performance of our checkpoints. We use an SFT dataset similar to \cite{seed2025seed-oss} and branch off from the 100,000-step checkpoints of the experiments in \cref{sec:exp:lr_schedule}. Training is conducted on the SFT data with a learning rate of $2 \times 10^{-5}$ and a global batch size of 256, which matches the pretraining learning rate and batch size at the 100,000-step mark. This setup can be viewed as continuing the learning rate decay on the SFT dataset, consistent with standard practices~\cite{qwen3technicalreport,yang2024qwen25,seed2025seed-oss}.

\begin{table*}[h!]
    \centering
    \caption{\textbf{Downstream SFT Results.} As shown, Nexus does not prematurely compromise the model's SFT capabilities; on the contrary, it continues to outperform AdamW after SFT.}
    \label{tab:exp:sft_results}
    
    \setlength{\tabcolsep}{0.35ex} 
    \resizebox{\textwidth}{!}{
        \begin{tabular}{lll cc c ccc cc cc c}
            \toprule
            \multirow{2}{*}{\textbf{Phase}} & \multirow{2}{*}{\textbf{Optim.}} & \multirow{2}{*}{\textbf{Metric}} & 
            \multicolumn{2}{c}{\textbf{Loss Metrics} ($\downarrow$)} & 
            \multicolumn{1}{c}{\textbf{Gen.}} & 
            \multicolumn{3}{c}{\textbf{Reasoning}} & 
            \multicolumn{2}{c}{\textbf{Math}} & 
            \multicolumn{2}{c}{\textbf{Code}} &
            \multicolumn{1}{c}{\textbf{Avg.}} \\
            \cmidrule(lr){4-5} \cmidrule(lr){6-6} \cmidrule(lr){7-9} \cmidrule(lr){10-11} \cmidrule(lr){12-13} \cmidrule(lr){14-14}
             & & & SFT & OOD & MMLU & GPQA & GPQA-D & BBH & GSM8k & MATH & HumanEval & MBPP & All \\
            \midrule
            
            \multirow{5}{*}{\textbf{Pre-SFT}} 
              & \multirow{2}{*}{AdamW} & Acc. ($\uparrow$) & \multirow{2}{*}{1.655} & \multirow{2}{*}{1.263} & 52.7 & \textbf{28.9} & 25.8 & \textbf{41.1} & 54.0 & 37.0 & \textbf{50.0} & 42.0 & \textbf{41.4} \\
              & & Loss ($\downarrow$) & & & 2.221 & 1.938 & 1.842 & 1.484 & 1.233 & 1.031 & 1.053 & 1.859 & 1.583 \\
              \cmidrule{2-14}
              & \multirow{2}{*}{Nexus} & Acc. ($\uparrow$) & \multirow{2}{*}{1.647} & \multirow{2}{*}{1.258} & 50.9 & 22.7 & 22.7 & 35.9 & \textbf{57.0} & \textbf{40.0} & 48.0 & 42.0 & 39.9 \\
              & & Loss ($\downarrow$) & & & \textbf{2.138} & 1.929 & 1.838 & 1.489 & \textbf{1.179} & \textbf{1.006} & \textbf{1.030} & \textbf{1.803} & \textbf{1.552} \\
              \cmidrule{2-14}
              & \multicolumn{2}{l}{\textit{Improv.} Loss ($\uparrow$)} & +0.008 & +0.005 & \textbf{+0.083} & +0.009 & +0.004 & -0.005 & \textbf{+0.054} & \textbf{+0.025} & \textbf{+0.023} & \textbf{+0.056} & \textbf{+0.031} \\
            \midrule
            
            \multirow{6}{*}{\textbf{Post-SFT}} 
              & \multirow{2}{*}{AdamW} & Acc. ($\uparrow$) & \multirow{2}{*}{1.035} & \multirow{2}{*}{1.278} & 51.4 & 28.9 & \textbf{33.6} & \textbf{45.6} & 58.0 & 28.0 & 40.0 & 40.0 & 40.7 \\
              & & Loss ($\downarrow$) & & & 2.244 & 2.006 & 1.915 & 1.575 & 1.377 & 1.111 & 1.077 & 1.957 & 1.658 \\
              \cmidrule{2-14}
              & \multirow{2}{*}{Nexus} & Acc. ($\uparrow$) & \multirow{2}{*}{1.028} & \multirow{2}{*}{1.274} & \textbf{54.7} & 28.9 & 29.7 & 39.3 & \textbf{62.0} & \textbf{35.0} & \textbf{46.0} & \textbf{46.0} & \textbf{42.7} \\
              & & Loss ($\downarrow$) & & & \textbf{2.220} & \textbf{1.990} & \textbf{1.901} & \textbf{1.551} & \textbf{1.299} & \textbf{1.076} & \textbf{1.060} & 1.952 & \textbf{1.631} \\
              \cmidrule{2-14}
              & \multirow{2}{*}{\textit{Improv.}} & Acc. ($\uparrow$) & - & - & \textbf{+3.3} & 0.0 & -3.9 & -6.3 & \textbf{+4.0} & \textbf{+7.0} & \textbf{+6.0} & \textbf{+6.0} & \textbf{+2.0} \\
              & & Loss ($\uparrow$) & +0.007 & +0.004 & \textbf{+0.024} & \textbf{+0.016} & \textbf{+0.014} & \textbf{+0.024} & \textbf{+0.078} & \textbf{+0.035} & \textbf{+0.017} & +0.005 & \textbf{+0.027} \\
            
            \bottomrule
        \end{tabular}
    }
\end{table*}

\textbf{Nexus still outperforms AdamW after SFT.} As shown in \cref{tab:exp:sft_results}, after supervised fine-tuning, Nexus achieves an average accuracy of 42.7\%, surpassing the AdamW baseline by 2.0\%. Specifically, Nexus outperforms AdamW by 7.0\% on MATH, 6.0\% on HumanEval, and 6.0\% on MBPP. These results indicate that Nexus does not prematurely compromise the model's capacity for downstream alignment.

\textbf{Nexus maintains lower SFT loss than AdamW throughout training.} We observe that for the pre-SFT checkpoints, Nexus already yields a lower loss on the SFT dataset compared to AdamW (1.647 vs. 1.655). This result demonstrates that the geometric properties optimized by Nexus during pretraining translate into better generalization even before any explicit fine-tuning. Furthermore, this lower SFT loss is consistently maintained throughout the entire training process, as evidenced by the post-SFT loss (1.028 for Nexus vs. 1.035 for AdamW). These observations indicate the potential of Nexus for continual training and \mbox{extended optimization phases.}

\subsection{Experiment using SGDM}
\label{appendix:more_exp:sgdm}

\begin{figure}[htbp]
    \centering
    \small
    \begin{tabular}{p{0.46\linewidth} | p{0.48\linewidth}}
    \toprule
    \textbf{Algorithm} Newton-Schulz Iteration & \textbf{Algorithm} RMS-Aligned Update \\
    \midrule
    \vspace{-2ex}
    \begin{verbatim}
def zeropower_via_ns(G, steps=5):
  a, b, c = 3.4445, -4.7750, 2.0315
  X = G / (torch.norm(G) + 1e-7)
  for _ in range(steps):
      A = X @ X.mT
      X = a*X + (b*A + c*A@A) @ X
  return X
    \end{verbatim}
    &
    \vspace{-2ex}
    \begin{verbatim}
def sgd_alignRMS_update(G):
  K = min(G.shape[-2:])
  # Normalize to match Muon scale
  norm = torch.norm(G) + 1e-7
  return G * (K**0.5) / norm
    \end{verbatim}
    \vspace{-3ex} \\
    \bottomrule
    \end{tabular}
    \vspace{1ex}
    \caption{Muon-inspired SGDM. \textbf{Left:} Quintic Newton-Schulz iterations for orthogonalization. \textbf{Right:} Per-tensor RMS normalization ensuring consistent update scale as Muon.}
    \label{fig:appendix:sgdm_pseudo_code}
\end{figure}

\textbf{Motivation.} We also conduct an experiment using SGD with Momentum (SGDM). However, SGDM typically converges slower than AdamW in LLM pretraining~\cite{zhang2024transformers}.

\textbf{Muon-Inspired SGDM.} To this end, we adapt the Muon optimizer \citep{jordan2024muon} to construct a faster SGDM baseline. Specifically, we control the RMS norm of SGDM updates to be the same as the orthogonalized updates in Muon, as shown in \cref{fig:appendix:sgdm_pseudo_code}. This design is inspired by the empirical findings in \citet{liu2025muon}, which suggest that maintaining a consistent per-tensor RMS norm aligns with the principles of Maximal Update Parameterization ($\mu$P). This may stabilize update dynamics and mitigate sensitivity to hyperparameter selection, and typically facilitate acceleration~\cite{qiu2025hyperparameter}. Surprisingly, in our LLM pretraining settings, this significantly accelerates SGDM, reaching speeds competitive with AdamW.

\textbf{Settings.} Due to computational constraints, we only conduct experiments on 1B models, strictly adhering to the experimental settings detailed in \cref{sec:exp:setting}.

\textbf{Results.} On this SGDM baseline, we observe the following results (\cref{appendix:fig:all_logs:sgdm}):

\begin{itemize}
    \item \textbf{Pretraining Acceleration.} Surprisingly, Nexus significantly accelerates pretraining loss minimization on SGDM. While the baseline SGDM is notably slower than AdamW, SGDM equipped with Nexus even surpasses the AdamW baseline. This acceleration may be attributed to the mechanism \mbox{described in \cref{sec:appendix_convergence_proof}.}
    \item \textbf{Robust Implicit Bias.} Since Nexus accelerates pretraining loss, direct comparisons of absolute downstream performance are confounded by the lower pretraining loss. Thus, we analyze the correlation between pretraining and downstream losses (\cref{appendix:fig:all_logs:sgdm}). As shown in \cref{appendix:fig:all_logs:sgdm}, the Nexus-SGDM curve lies consistently below the baseline SGDM, confirming that the favorable implicit bias persists even under \mbox{different base optimizers.}
\end{itemize}

\begin{figure}[t!]
    \centering
    \begin{subfigure}[b]{0.32\linewidth}
        \centering
        \includegraphics[width=\linewidth]{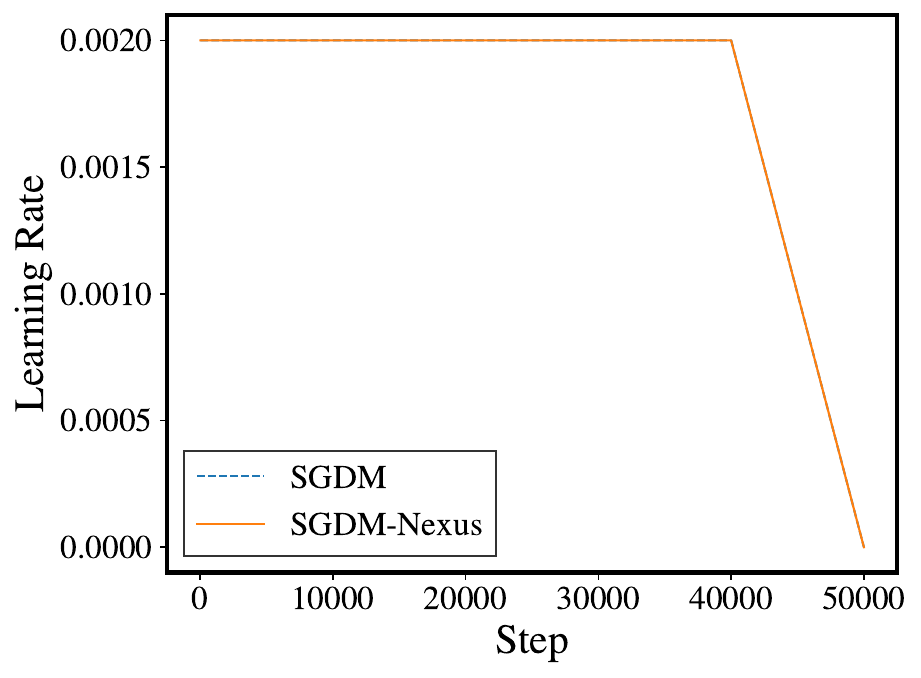}
        \caption{Learning Rate}
    \end{subfigure}\hfill
    \begin{subfigure}[b]{0.32\linewidth}
        \centering
        \includegraphics[width=\linewidth]{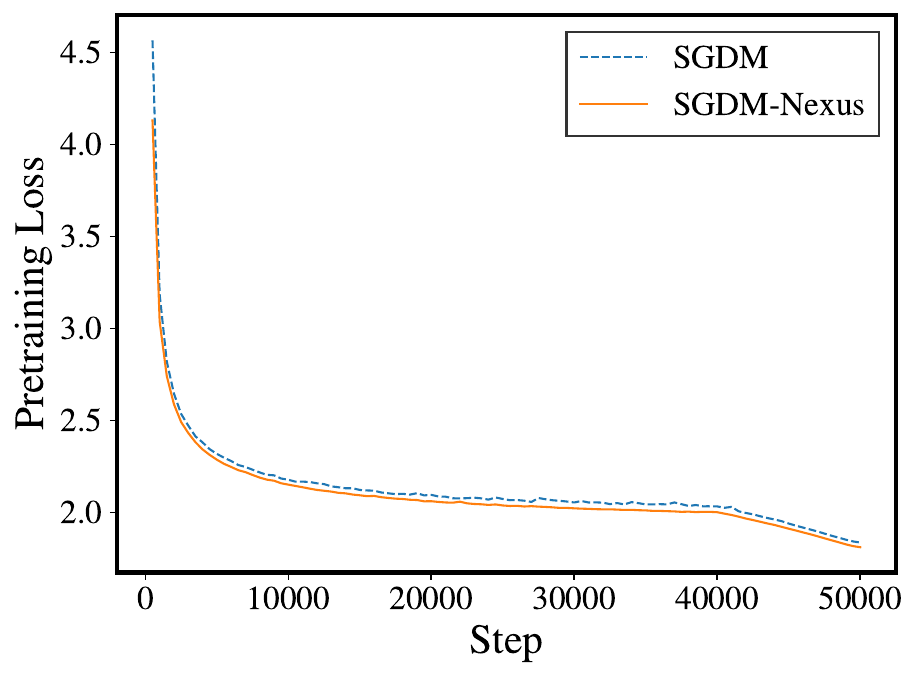}
        \caption{Pretraining Loss}
    \end{subfigure}\hfill
    \begin{subfigure}[b]{0.32\linewidth}
        \centering
        \includegraphics[width=\linewidth]{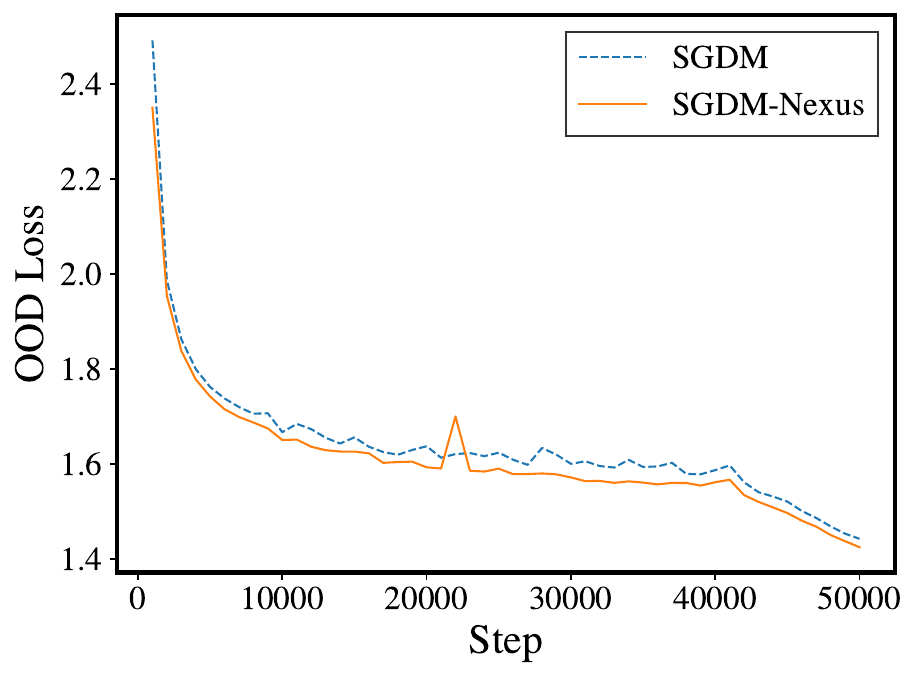}
        \caption{OOD Loss}
    \end{subfigure}
    
    \vspace{1.5ex}
    \begin{subfigure}[b]{0.32\linewidth}
        \centering
        \includegraphics[width=\linewidth]{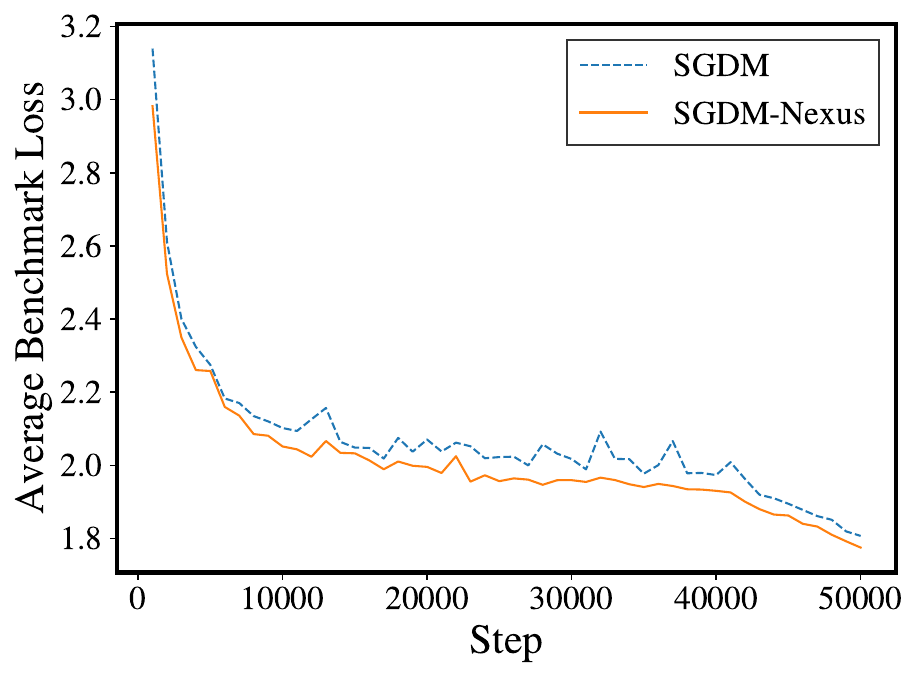}
        \caption{Avg. Benchmark Loss}
    \end{subfigure}\hfill
    \begin{subfigure}[b]{0.32\linewidth}
        \centering
        \includegraphics[width=\linewidth]{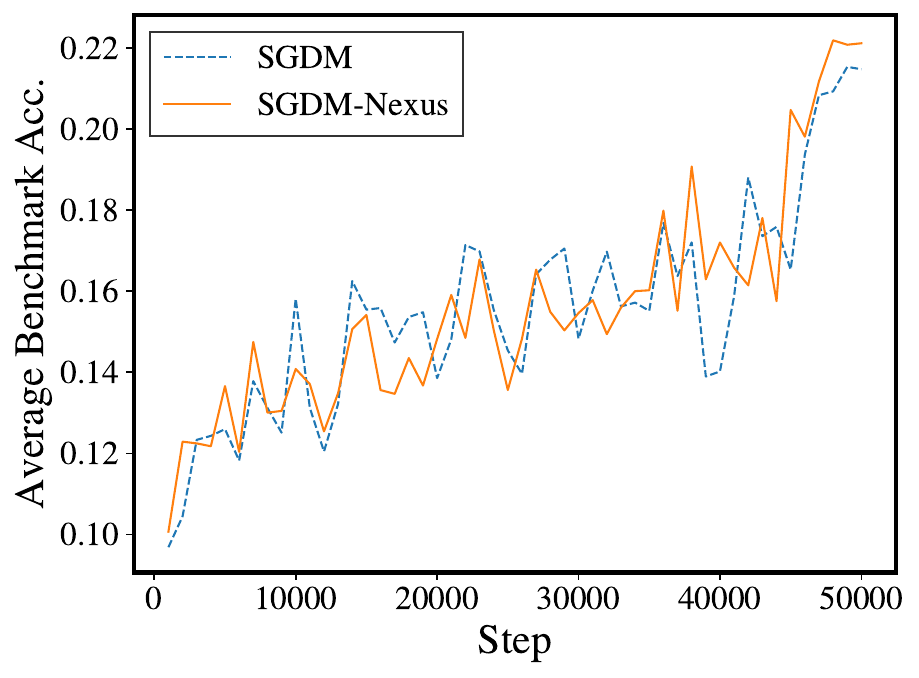}
        \caption{Avg. Benchmark Acc.}
    \end{subfigure}\hfill
    \begin{subfigure}[b]{0.32\linewidth}
        \centering
        \includegraphics[width=\linewidth]{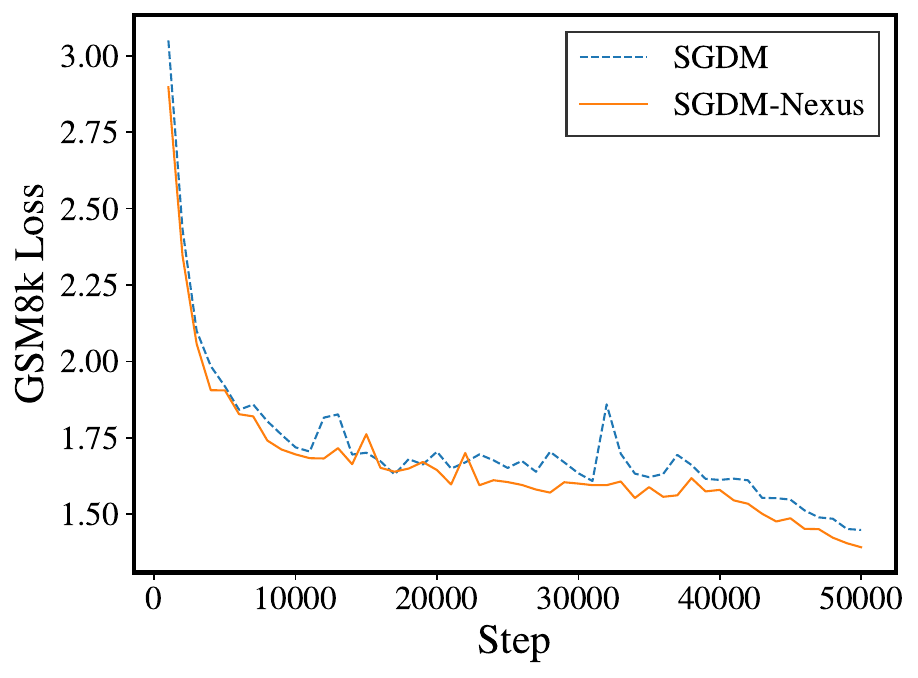}
        \caption{GSM8k Loss}
    \end{subfigure}
    
    \vspace{1.5ex}
    \begin{subfigure}[b]{0.32\linewidth}
        \centering
        \includegraphics[width=\linewidth]{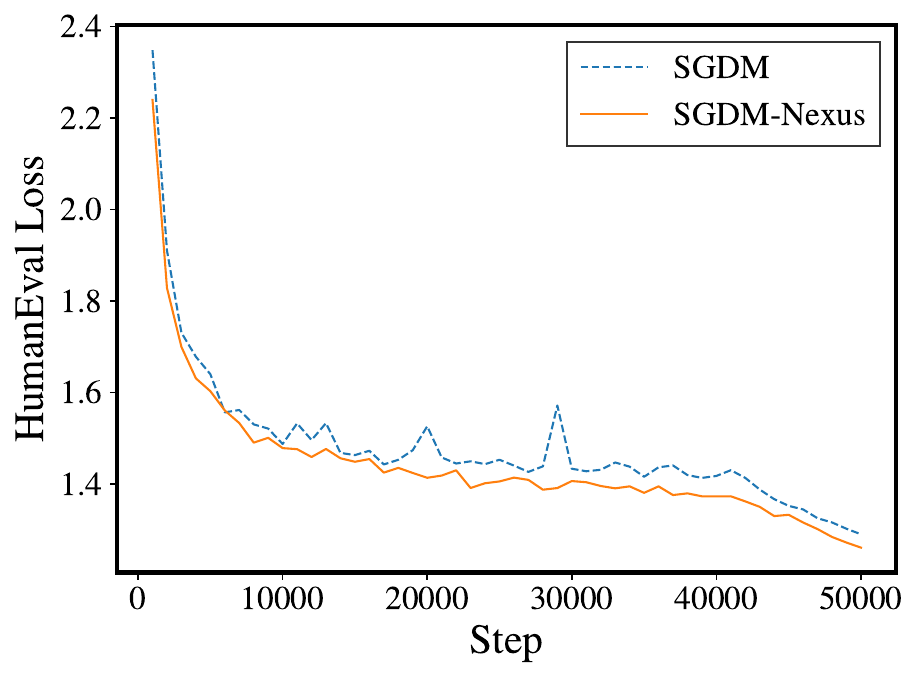}
        \caption{HumanEval Loss}
    \end{subfigure}\hfill
    \begin{subfigure}[b]{0.32\linewidth}
        \centering
        \includegraphics[width=\linewidth]{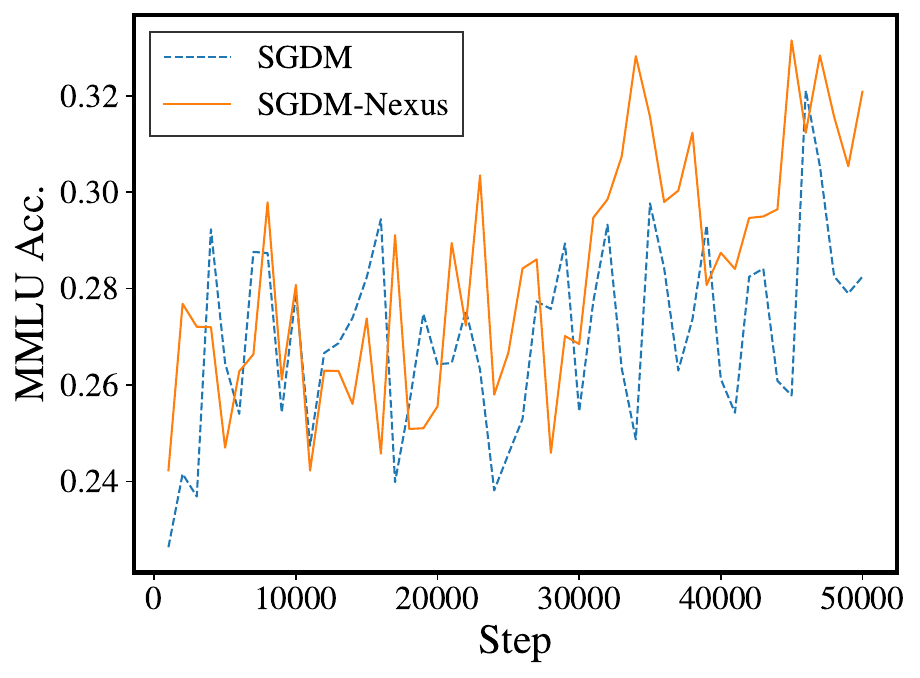}
        \caption{MMLU Acc.}
    \end{subfigure}\hfill
    \begin{subfigure}[b]{0.32\linewidth}
        \centering
        \includegraphics[width=\linewidth]{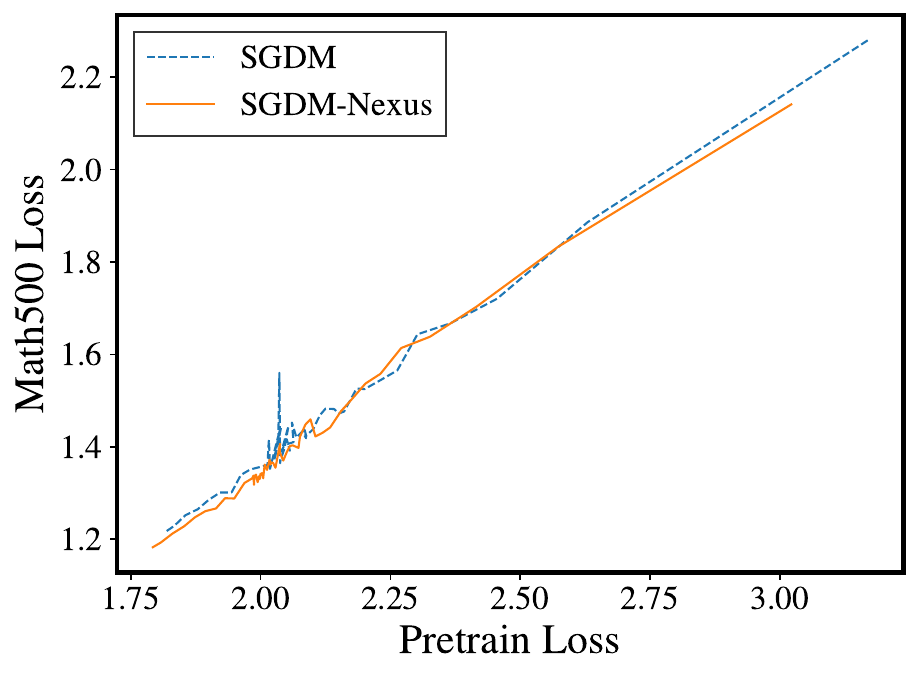}
        \caption{Math500 Correlation}
    \end{subfigure}
    
    \vspace{1.5ex}
    \begin{subfigure}[b]{0.32\linewidth}
        \centering
        \includegraphics[width=\linewidth]{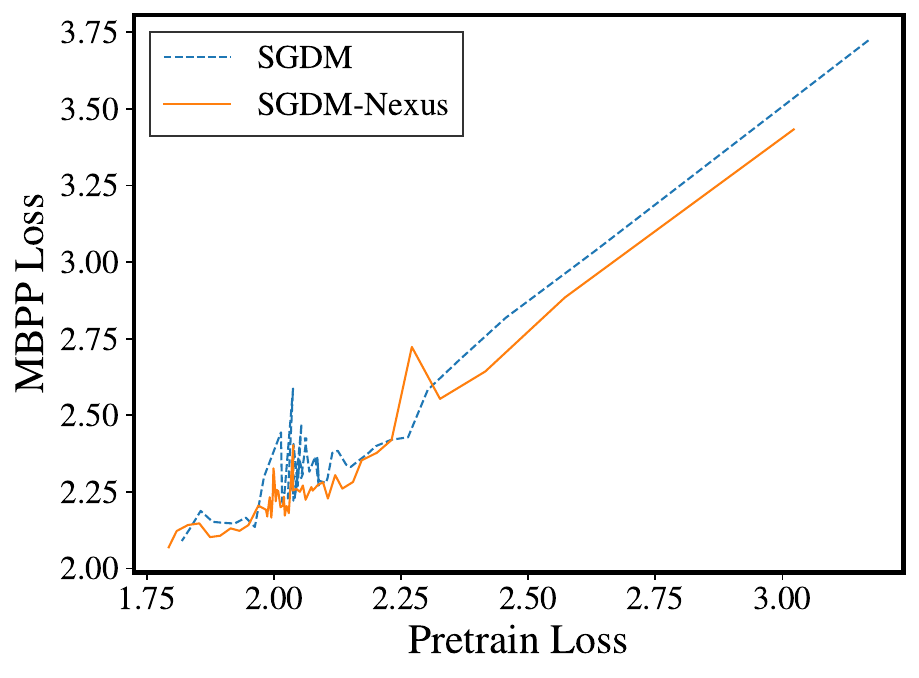}
        \caption{MBPP Correlation}
    \end{subfigure}\hfill
    \begin{subfigure}[b]{0.32\linewidth}
        \centering
        \includegraphics[width=\linewidth]{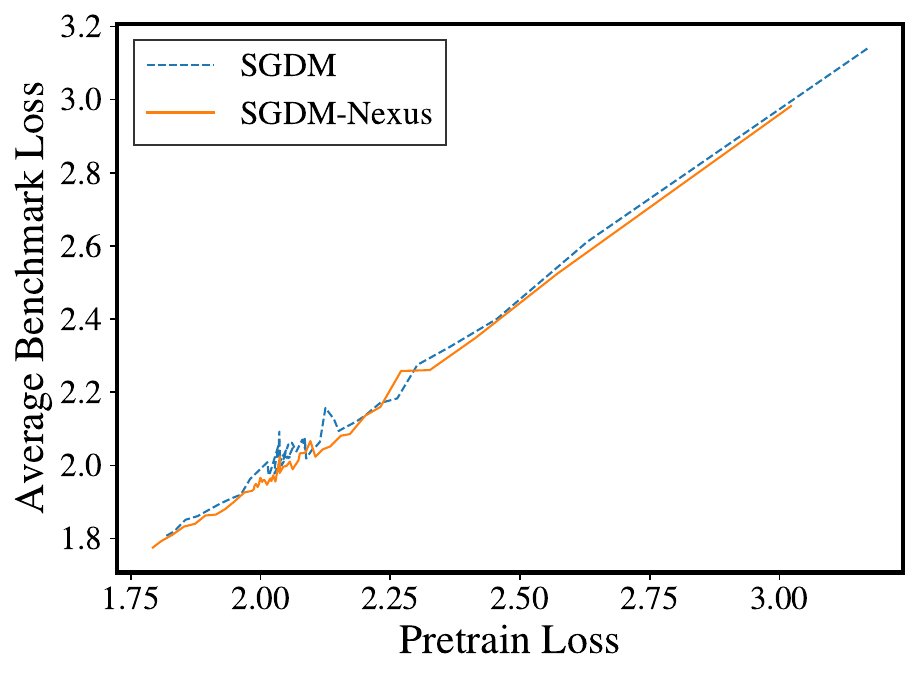}
        \caption{Avg. Loss Correlation}
    \end{subfigure}\hfill
    \begin{subfigure}[b]{0.32\linewidth}
        \centering
        \includegraphics[width=\linewidth]{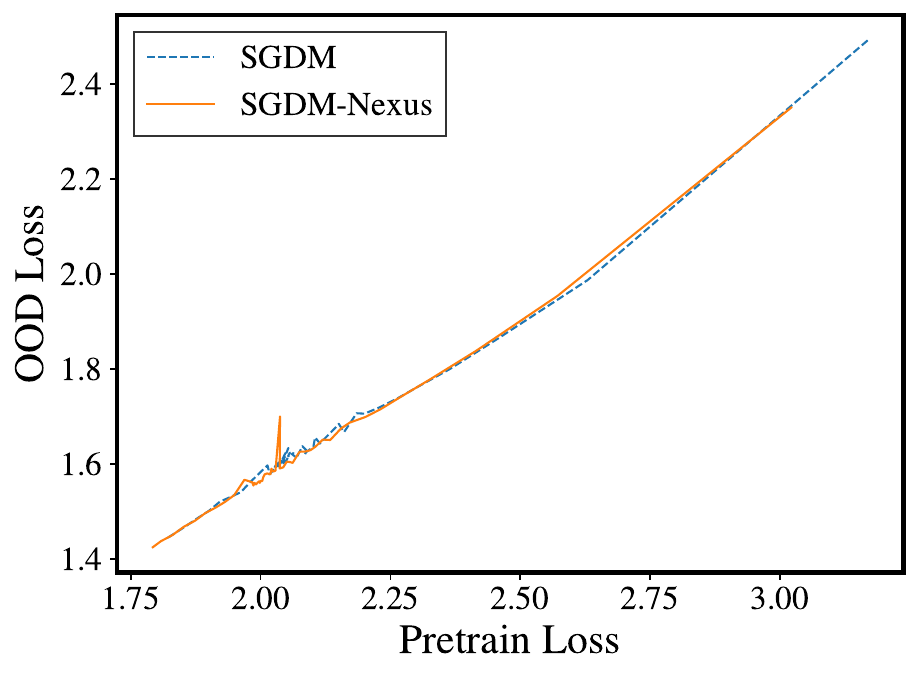}
        \caption{OOD Correlation}
    \end{subfigure}
    \caption{\textbf{Experimental Results using SGDM.} Nexus both gives acceleration and implicit bias.}
    \label{appendix:fig:all_logs:sgdm}
\end{figure}

\newpage
\subsection{Experiment using Smaller Batch Size}

\begin{figure}[t]
    \centering
    \begin{subfigure}[b]{0.32\linewidth}
        \centering
        \includegraphics[width=\linewidth]{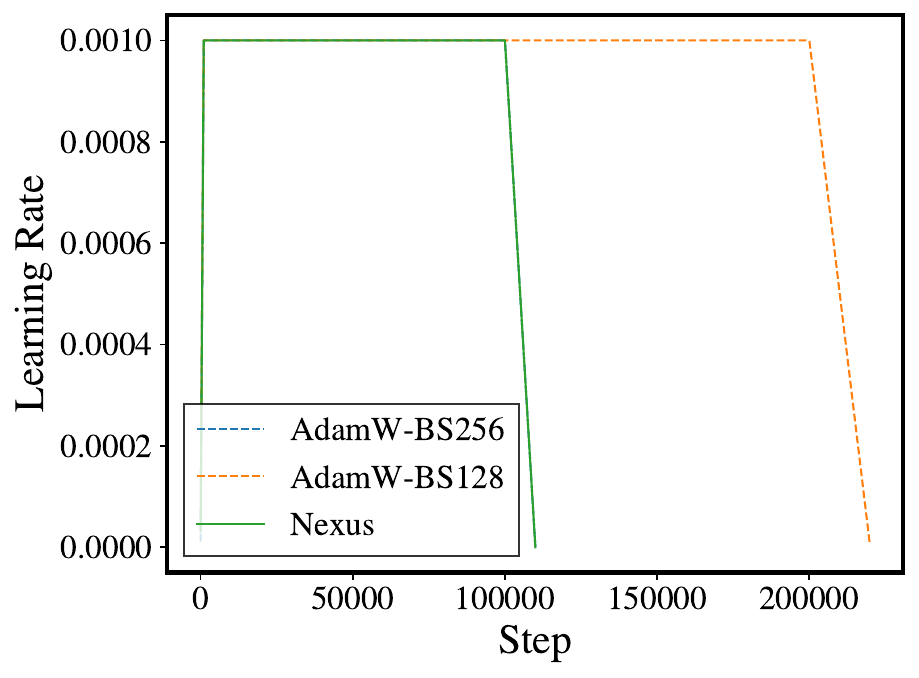}
        \caption{Learning Rate}
    \end{subfigure}\hfill
    \begin{subfigure}[b]{0.32\linewidth}
        \centering
        \includegraphics[width=\linewidth]{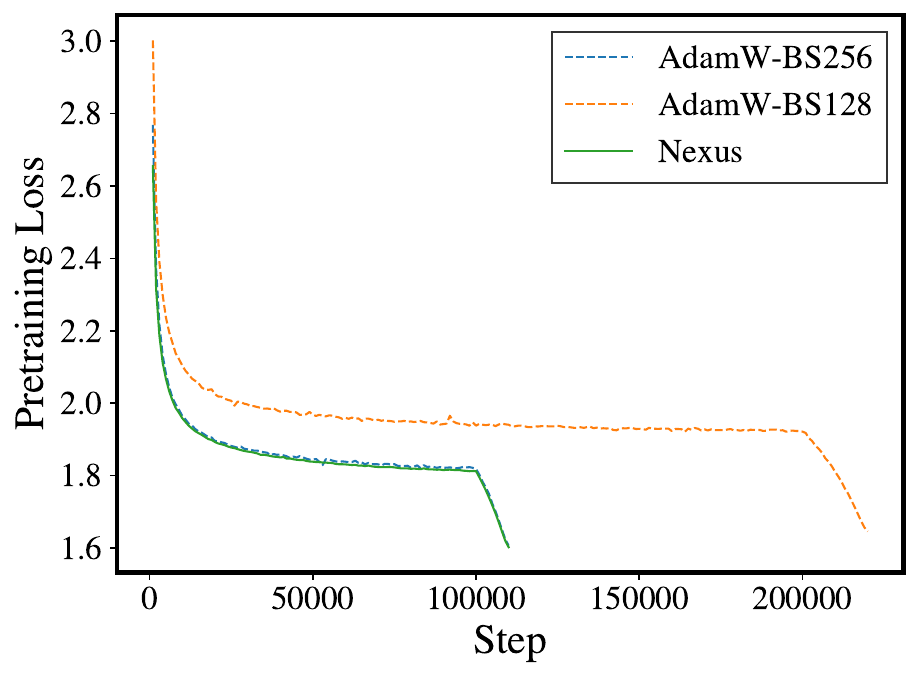}
        \caption{Pretraining Loss}
    \end{subfigure}\hfill
    \begin{subfigure}[b]{0.32\linewidth}
        \centering
        \includegraphics[width=\linewidth]{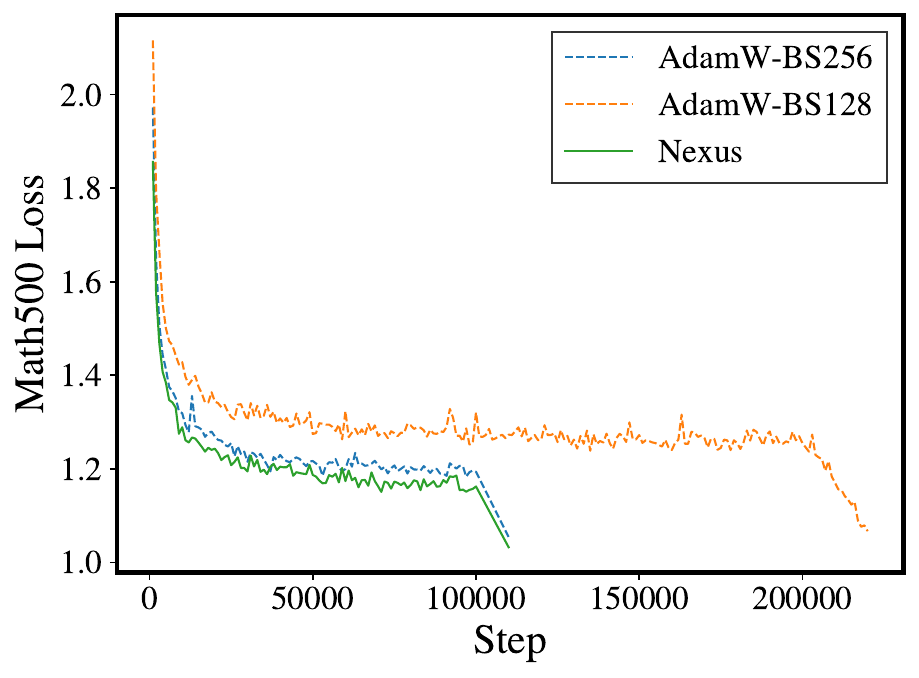}
        \caption{Math500 Loss}
    \end{subfigure}
    
    \vspace{1.5ex}
    \begin{subfigure}[b]{0.32\linewidth}
        \centering
        \includegraphics[width=\linewidth]{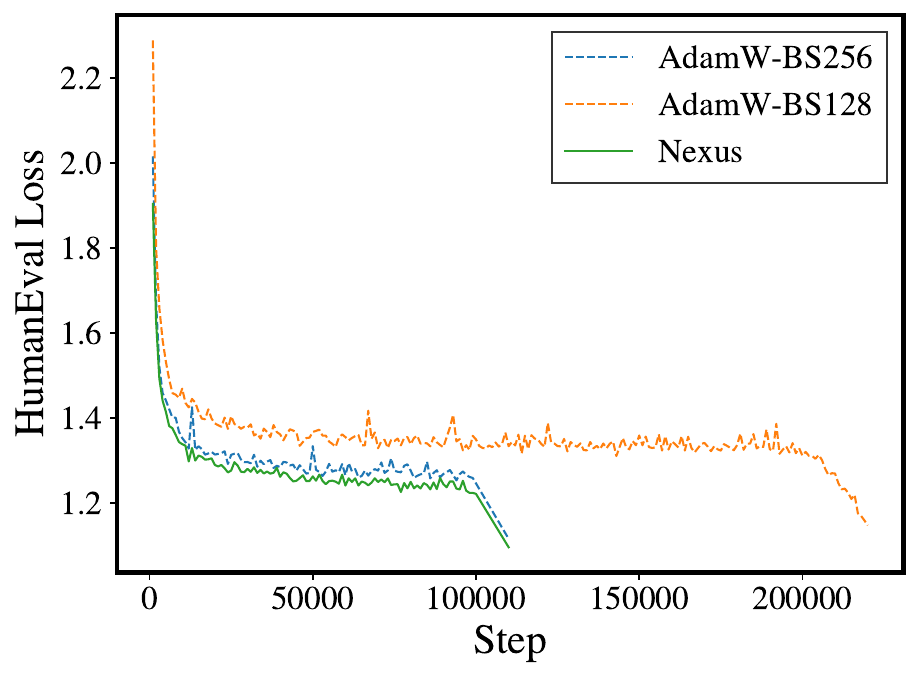}
        \caption{HumanEval Loss}
    \end{subfigure}\hfill
    \begin{subfigure}[b]{0.32\linewidth}
        \centering
        \includegraphics[width=\linewidth]{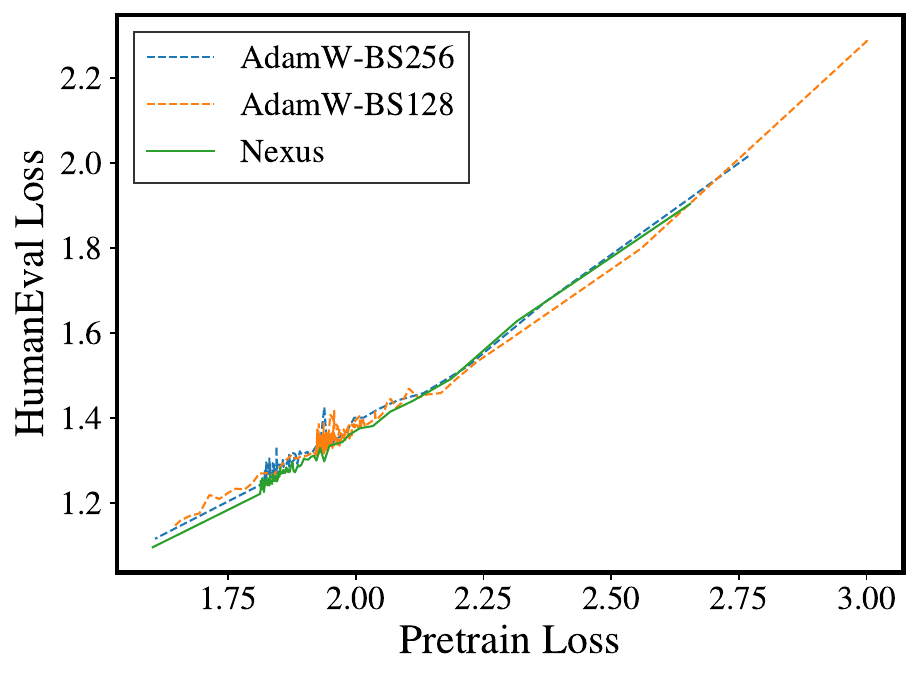}
        \caption{HumanEval Correlation}
    \end{subfigure}\hfill
    \begin{subfigure}[b]{0.32\linewidth}
        \centering
        \includegraphics[width=\linewidth]{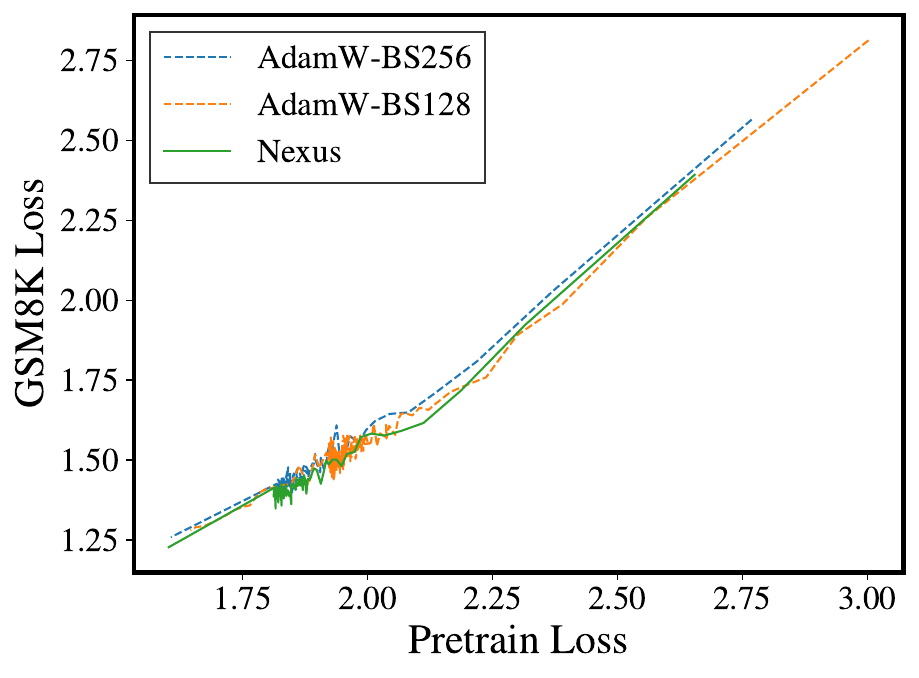}
        \caption{GSM8K Correlation}
    \end{subfigure}
    
    \caption{\textbf{Experimental Results including the BS-128 Baseline.} The standard AdamW and Nexus trajectories are retained from the main trials in \cref{sec:exp:main_result} for direct comparison.}
    \label{appendix:fig:all_logs:bs128}
\end{figure}

\textbf{Motivation.} It is well-established that smaller batch sizes can introduce an implicit bias that enhances generalization \cite{keskar2016large,jastrzkebski2017three,smith2018bayesian,masters2018revisiting}. To distinguish the geometric closeness bias of Nexus from this small-batch effect, we evaluate a baseline using AdamW with a reduced batch size of 128 on the 1B and 3B models. All hyperparameters remain identical to those in \cref{sec:exp:setting}, except for the learning rate schedule, where we double the duration of both the stable and decay phases, as illustrated in \cref{appendix:fig:all_logs:bs128}(a).

\begin{table*}[t]
    \centering
    \caption{\textbf{Results of Small Batch Size Baseline.} Validation losses and downstream capabilities for AdamW trained with a reduced batch size of 128. Compared to the optimal BS-256 baseline reported in the main text (\cref{tab:exp:d803_main_results}), reducing the batch size results in significant performance degradation.}
    \label{tab:appendix:bs128_results}
    \setlength{\tabcolsep}{0.35ex} 
    \resizebox{\textwidth}{!}{
        \begin{tabular}{lcl cc c ccc cc cc c}
            \toprule
            \multirow{2}{*}{\textbf{Model}} & \multirow{2}{*}{\textbf{Configuration}} & \multirow{2}{*}{\textbf{Metric}} & 
            \multicolumn{2}{c}{\textbf{Loss Metrics} ($\downarrow$)} & 
            \multicolumn{1}{c}{\textbf{Gen.}} & 
            \multicolumn{3}{c}{\textbf{Reasoning}} & 
            \multicolumn{2}{c}{\textbf{Math}} & 
            \multicolumn{2}{c}{\textbf{Code}} &
            \multicolumn{1}{c}{\textbf{Avg.}} \\
            \cmidrule(lr){4-5} \cmidrule(lr){6-6} \cmidrule(lr){7-9} \cmidrule(lr){10-11} \cmidrule(lr){12-13} \cmidrule(lr){14-14}
             & & & Pretrain. & OOD & MMLU & GPQA & GPQA-D & BBH & GSM8k & MATH & HumanEval & MBPP & All \\
            \midrule
            
            \multirow{2}{*}{1B} 
              & \multirow{2}{*}{AdamW (BS=128)} & Acc. ($\uparrow$) & \multirow{2}{*}{1.937} & \multirow{2}{*}{1.506} & 29.3 & 25.0 & 20.3 & 23.3 & 8.0 & 5.0 & 11.0 & 14.0 & 17.0 \\
              & & Loss ($\downarrow$) & & & 2.396 & 2.333 & 2.231 & 1.728 & 1.510 & 1.286 & 1.340 & 2.091 & 1.864 \\
            \midrule
            
            \multirow{2}{*}{3B} 
              & \multirow{2}{*}{AdamW (BS=128)} & Acc. ($\uparrow$) & \multirow{2}{*}{1.627} & \multirow{2}{*}{1.326} & 45.4 & 29.7 & 21.9 & 40.7 & 39.0 & 26.0 & 37.0 & 39.0 & 34.8 \\
              & & Loss ($\downarrow$) & & & 2.262 & 2.033 & 1.937 & 1.556 & 1.281 & 1.067 & 1.148 & 1.951 & 1.654 \\
            \bottomrule
        \end{tabular}
    }
\end{table*}

On this baseline, we observe the following:

\textbf{BS-128 is suboptimal compared to BS-256.} As detailed in \cref{sec:exp:setting}, our hyperparameter configurations are strictly aligned with the optimal settings established by \citet{wen2025fantastic}. Deviating from these settings compromises optimization efficiency. As demonstrated in \cref{tab:appendix:bs128_results}, reducing the global batch size from 256 to 128 causes the training speed to severely lag behind. Consequently, the AdamW-BS128 baseline substantially underperforms the standard AdamW-BS256 across both pretraining convergence and downstream benchmark evaluations by a loss margin of more than 0.02.

\textbf{BS-128 introduces implicit bias, but less significantly than Nexus.} As illustrated in the correlation plots of \cref{appendix:fig:all_logs:bs128}, AdamW-BS128 indeed exhibits a favorable implicit bias compared to the standard AdamW, particularly during the early stages of training where the pretraining loss is high. However, in the later stages of training, the correlation curve of BS-128 lies between the standard AdamW and Nexus, indicating that its implicit bias is not as pronounced as that of Nexus. This confirms that the strong implicit bias of Nexus genuinely originates from its explicit geometric regularization, rather than merely mimicking the gradient noise effect of small batch sizes. Furthermore, while Nexus successfully induces this bias while maintaining an identical pretraining loss to the baseline, BS-128 fails to do so (suffering a pretraining loss degradation of approximately 0.02).

\subsection{Ablation Studies on Regularization Term}

\begin{algorithm}[t]
\caption{Controllable Nexus Variant for Ablation Study}
\label{alg:nexus_ablation}
\begin{algorithmic}[1]
\REQUIRE Control coefficient $\lambda \geq 0$
\STATE $\hat{\bm{g}}_t \leftarrow \text{\cref{alg:cwa_standard,alg:cwa_adapted}}(\nabla \mathcal{L}(\bm{\theta}_t))$ \COMMENT{Standard Nexus pseudo gradient}
\STATE $\bm{u}_{\text{ablation}} \leftarrow \hat{\bm{g}}_t - \gamma \lambda \frac{\nabla \mathcal{L}(\bm{\theta}_t)}{\|\nabla \mathcal{L}(\bm{\theta}_t)\|_2 + \epsilon}$ \COMMENT{Decouple optimization and regularization}
\STATE \textbf{Return:} $\bm{u}_{\text{ablation}}$
\end{algorithmic}
\end{algorithm}

\textbf{Motivation.} The standard formulations of Nexus (\cref{alg:cwa_standard,alg:cwa_adapted}) intrinsically couple the optimization of the primary loss and the regularization term at a fixed, implicit ratio of $1 : \gamma \frac{K-1}{4K}$. To systematically decouple and investigate the specific regularization effect of Nexus, we design a controllable variant. As demonstrated in \cref{alg:nexus_ablation}, after obtaining the Nexus pseudo gradient $\hat{\bm{g}}_t$, we subtract a scaled, $L_2$-normalized loss gradient $\gamma \lambda \frac{\nabla \mathcal{L}(\bm{\theta}_t)}{\|\nabla \mathcal{L}(\bm{\theta}_t)\|_2}$. Consequently, the expectation of the actual modified update $\bm{u}_{\text{ablation}}$ applied to the parameters is given by:
\begin{equation}
    \mathbb{E}[\bm{u}_{\text{ablation}}] = \gamma (1-\lambda)\sum_{i=1}^K \frac{\nabla \mathcal{L}_i(\bm{\theta}_t)}{\|\nabla \mathcal{L}_i(\bm{\theta}_t)\|_2}  - \gamma^2 \frac{K-1}{4K} \Big( \nabla_{\bm{\theta}} \sum_{i \neq j} \text{CosSim}(\nabla \mathcal{L}_i, \nabla \mathcal{L}_j) + \bm{\mathcal{E}}_{\text{sym},i,j} \Big) + \bm{\mathcal{E}}_{\text{2nd}}.
\end{equation}
In this formulation, $\lambda \in [0, 1)$ serves as a direct dial to control the regularization strength. As $\lambda$ approaches $1$, the first-order loss-minimization component is explicitly suppressed, effectively amplifying the relative influence of the geometric regularization term.

However, this modified algorithm is intended exclusively for ablation purposes, as it incurs double the computational overhead per step (despite maintaining an identical token budget). Consequently, we evaluate the impact of the regularization-to-loss ratio using the 1B model configuration described in \cref{sec:exp:setting}, sweeping across $\lambda \in \{0, 0.1, 0.2, 0.4, 0.8\}$, where $\lambda=0$ corresponds to the original Nexus algorithm. The results are summarized as follows:

\begin{figure}[t]
    \centering
    \begin{subfigure}[b]{0.32\linewidth}
        \centering
        \includegraphics[width=\linewidth]{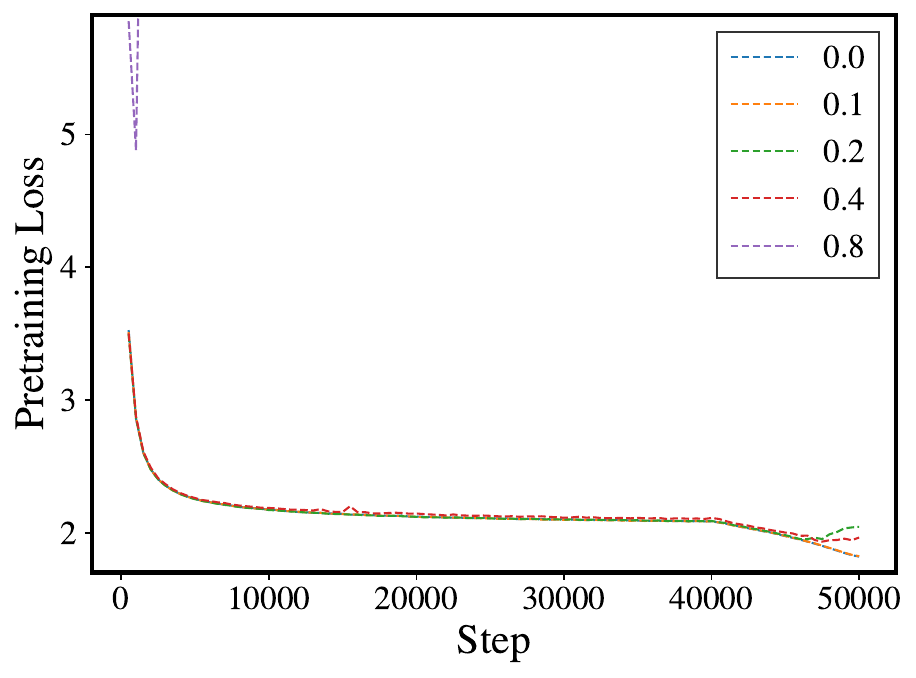}
        \caption{Pretraining Loss}
    \end{subfigure}\hfill
    \begin{subfigure}[b]{0.32\linewidth}
        \centering
        \includegraphics[width=\linewidth]{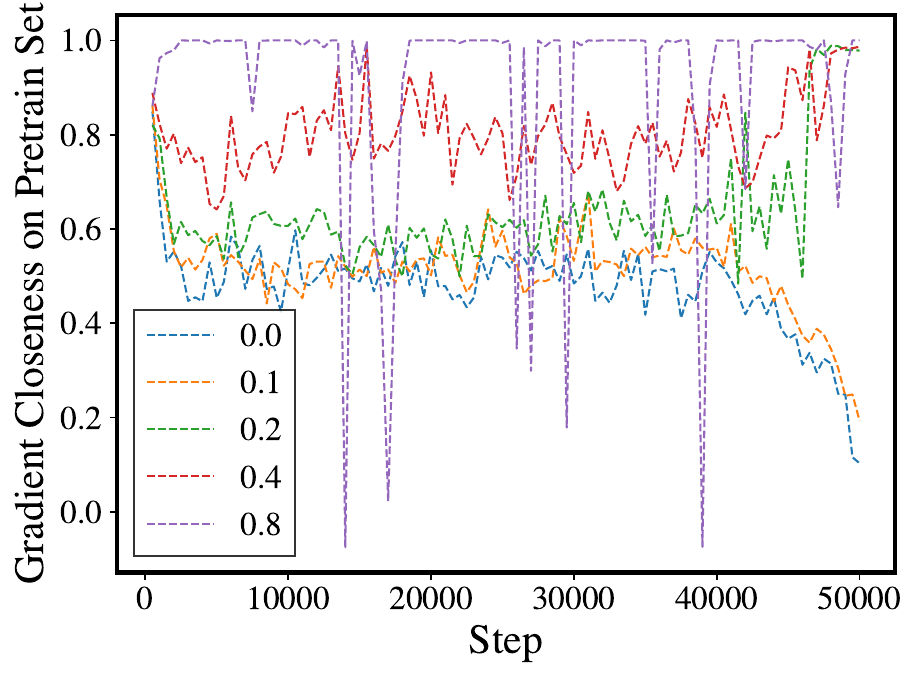}
        \caption{Gradient Similarity (Pretrain Set)}
    \end{subfigure}\hfill
    \begin{subfigure}[b]{0.32\linewidth}
        \centering
        \includegraphics[width=\linewidth]{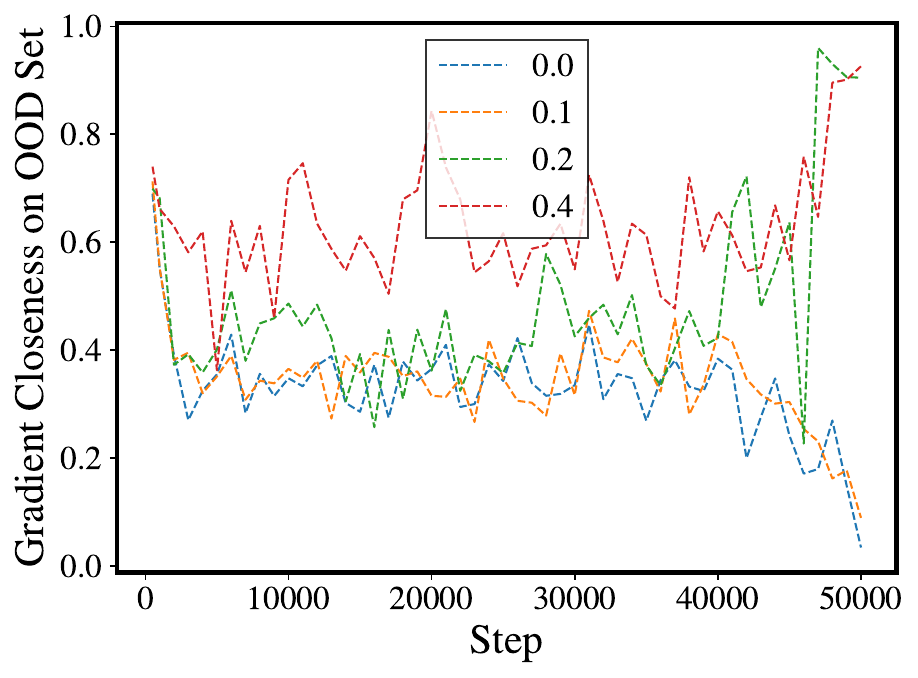}
        \caption{Gradient Similarity (OOD Set)}
    \end{subfigure}
    
    \caption{\textbf{Ablation Results on Regularization Strength $\lambda$.} As the coefficient $\lambda$ increases, there is a monotonic increase in gradient closeness across both training and OOD sets.}
    \label{appendix:fig:ablation_lambda}
\end{figure}

\textbf{Results.} As illustrated in \cref{appendix:fig:ablation_lambda}, we observe a clear monotonic relationship: the higher the $\lambda$, the higher the gradient closeness on both pretraining and OOD sets. This empirical evidence confirms that the Nexus mechanism indeed explicitly maximizes gradient closeness by inner-loop outer-loop mechanism.

\FloatBarrier

\newpage
\begingroup
\setlength{\intextsep}{7pt}
\subsection{Multi-Seed Evaluation}
\label{appendix:multi_seed}

We evaluate the stability of Nexus across multiple runs, including three runs for the 1B model and two runs for the 3B model, with the same data order across runs. Results are reported in \cref{tab:nexus_multi_seed}.

Across both model sizes, Nexus achieves nearly the same pretraining loss as AdamW while consistently improving downstream performance. The improvements in downstream loss are substantially larger than the observed standard deviations.

\begin{table}[htbp]
    \centering
    \caption{Multi-seed evaluation on 1B and 3B models. We report mean $\pm$ standard deviation across three and two runs, respectively.}
    \label{tab:nexus_multi_seed}
    \small
    \begin{tabular}{llccc}
        \toprule
        Model & Optim. & Pretrain. Loss ($\downarrow$) &
        Avg. Downstream Loss ($\downarrow$) &
        Avg. Acc. (\%, $\uparrow$) \\
        \midrule
        1B & AdamW
        & $1.818 \pm 0.0055$
        & $1.786 \pm 0.0012$
        & $22.43 \pm 0.41$ \\
        1B & Nexus
        & $1.818 \pm 0.0055$
        & $\mathbf{1.756 \pm 0.0017}$
        & $24.07 \pm 0.59$ \\
        \midrule
        3B & AdamW
        & $1.607 \pm 0.0005$
        & $1.631 \pm 0.0020$
        & $37.40 \pm 0.30$ \\
        3B & Nexus
        & $1.602 \pm 0.0000$
        & $\mathbf{1.609 \pm 0.0000}$
        & $\mathbf{40.33 \pm 0.12}$ \\
        \bottomrule
    \end{tabular}
\end{table}
\FloatBarrier

\subsection{Closeness and Flatness}
\label{appendix:closeness_flatness}

Does Nexus improve flatness in addition to closeness? As discussed in \cref{sec:flatness_closeness}, downstream loss depends on both the distance to task-specific minimizers (closeness) and the directional curvature of the loss landscape (flatness).

To investigate this, we measure the leading Hessian eigenvalues throughout 1B pretraining. As shown in \cref{fig:nexus_hessian_eigenvalues}, Nexus consistently exhibits smaller top-1, average top-4, and average top-16 Hessian eigenvalues than AdamW. For example, at step 50k, the largest eigenvalue decreases from 8.330 to 1.609.

These results suggest that Nexus affects both closeness and flatness. While our theoretical analysis focuses on closeness under controlled curvature assumptions, the two properties may be coupled in practice. Their individual contributions to downstream generalization remain to be understood.

\begin{figure}[htbp]
    \centering
    \includegraphics[width=0.83\linewidth]{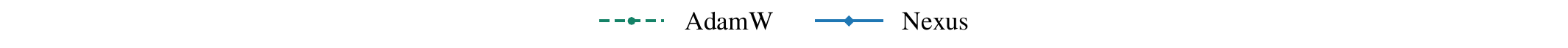}\par\smallskip
    \begin{subfigure}[b]{0.32\linewidth}
        \centering
        \includegraphics[width=\linewidth]{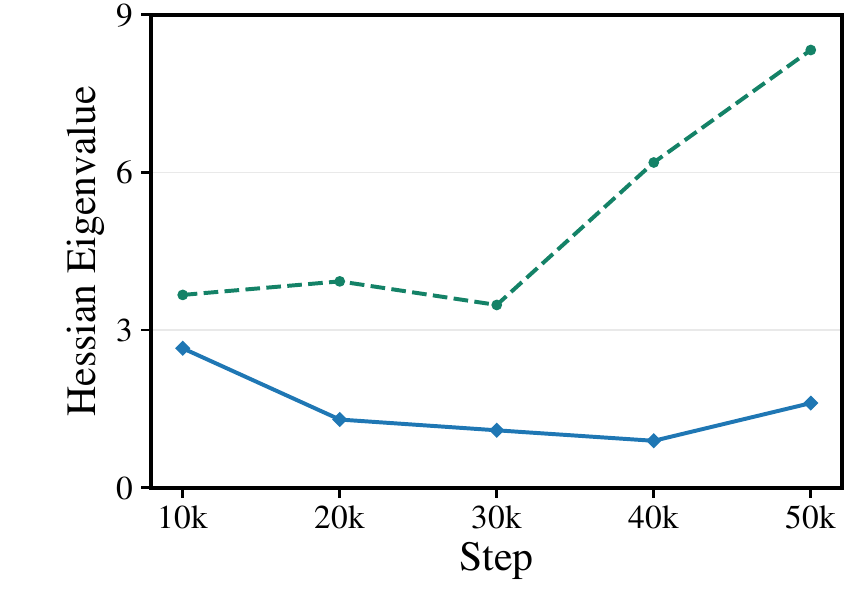}
        \caption{Top-1 Hessian Eigenvalue}
        \label{fig:nexus_hessian_eigenvalues:top1}
    \end{subfigure}\hfill
    \begin{subfigure}[b]{0.32\linewidth}
        \centering
        \includegraphics[width=\linewidth]{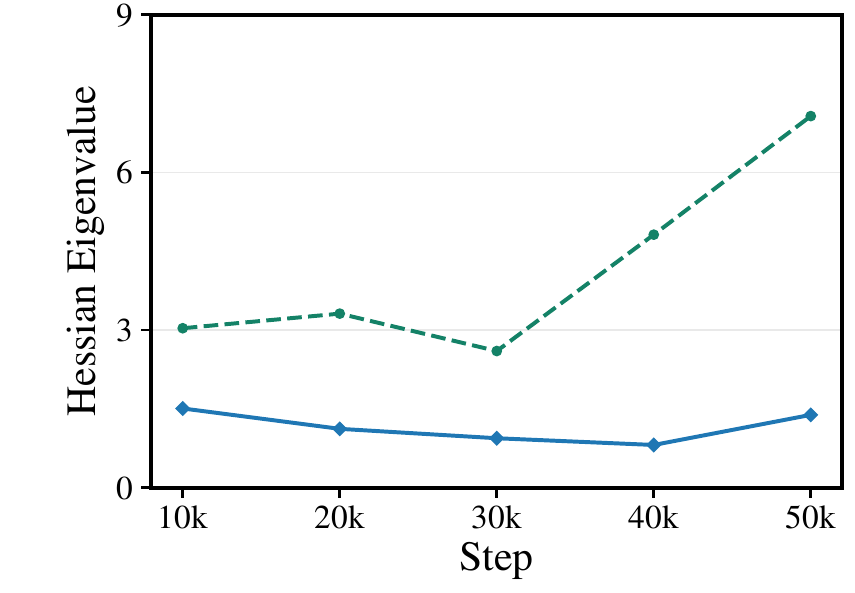}
        \caption{Mean Top-4 Hessian Eigenvalues}
        \label{fig:nexus_hessian_eigenvalues:top4}
    \end{subfigure}\hfill
    \begin{subfigure}[b]{0.32\linewidth}
        \centering
        \includegraphics[width=\linewidth]{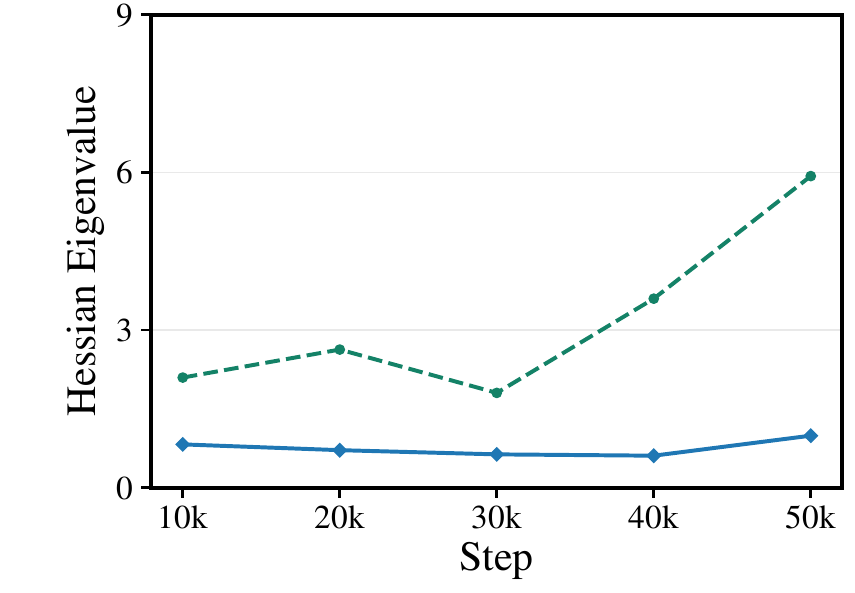}
        \caption{Mean Top-16 Hessian Eigenvalues}
        \label{fig:nexus_hessian_eigenvalues:top16}
    \end{subfigure}
    \caption{Leading Hessian eigenvalues throughout 1B pretraining. We report the top-1 eigenvalue and the averages of the top-4 and top-16 eigenvalues. Nexus consistently exhibits lower leading Hessian eigenvalues than AdamW.}
    \label{fig:nexus_hessian_eigenvalues}
\end{figure}
\FloatBarrier

\subsection{Computational Overhead}
\label{appendix:computational_overhead}

Nexus introduces no additional forward or backward passes compared to standard pretraining. Its extra computation comes from updating and synchronizing the auxiliary inner model. Under 8$\times$A100-equivalent accounting, the measured wall-clock overhead over AdamW is only 0.55\% for the 1B model and 2.52\% for the 3B model.
\par
\endgroup

\end{document}